\documentclass{article}
\usepackage{fullpage}
\usepackage{amsmath,amsfonts,amsthm,commath,amssymb,dsfont,mathtools}
\usepackage[normalem]{ulem}
\usepackage{graphicx}
\usepackage{grffile}
\usepackage{xcolor}
\usepackage{amssymb}
\usepackage{hyperref}
\usepackage[shortlabels]{enumitem}
\usepackage{tikz}
\usetikzlibrary{arrows}
\usepackage{makecell}
\usepackage{minitoc}        
\usepackage{bbm}
\usepackage{subcaption}

\usepackage[style=numeric,citestyle=numeric,maxbibnames=99,backend=biber,sorting=nyt,natbib=true,backref=true]{biblatex}
\usetikzlibrary{shapes,arrows,positioning}
\tikzstyle{ellipsoid} = [draw, ellipse, minimum height=3em, minimum width=3em]
\tikzstyle{block} = [draw, rectangle, minimum height=3em, minimum width=3em]
\tikzstyle{round} = [draw, circle, minimum height=3em, minimum width=2em]
\tikzstyle{virtual} = [coordinate]

\newcommand{\A}{\boldsymbol{A}}
\newcommand{\bc}{\boldsymbol{c}}
\newcommand{\C}{\boldsymbol{C}}

\newcommand{\x}{\boldsymbol{x}}
\newcommand{\y}{\boldsymbol{y}}
\newcommand{\bv}{\boldsymbol{v}}
\newcommand{\w}{\boldsymbol{w}}
\newcommand{\W}{\boldsymbol{W}}
\newcommand{\m}{\boldsymbol{m}}

\newcommand{\g}{\boldsymbol{g}}
\newcommand{\I}{\boldsymbol{I}}

\newcommand{\D}{\boldsymbol{D}}
\newcommand{\dbold}{\boldsymbol{d}}
\newcommand{\bmu}{\boldsymbol{\mu}}
\newcommand{\X}{\boldsymbol{X}}
\newcommand{\Z}{\boldsymbol{Z}}

\newcommand{\z}{\boldsymbol{z}}
\newcommand{\p}{\boldsymbol{p}}
\newcommand{\bP}{\boldsymbol{P}}

\newcommand{\q}{\boldsymbol{q}}
\newcommand{\Q}{\boldsymbol{Q}}
\newcommand{\Qs}[2][]{\boldsymbol{Q}_{#2}^{#1}}
\newcommand{\bu}{\boldsymbol{u}}
\newcommand{\blambda}{\boldsymbol{\lambda}}
\newcommand{\bdelta}{\boldsymbol{\delta}}

\newcommand{\bbeta}{\boldsymbol{\beta}}
\newcommand{\zero}{\boldsymbol{0}}
\newcommand{\one}{\boldsymbol{1}}
\newcommand{\bxi}{\boldsymbol{\xi}}
\newcommand{\bXi}{\boldsymbol{\Xi}}

\newcommand{\bDelta}{\boldsymbol{\Delta}}
\newcommand{\nov}{\frac{\one}{\sqrt{n}}}
\newcommand{\heta}{\hat{\eta}}
\newcommand{\hetat}{\hat{\eta}_t}
\newcommand{\hetaj}{\hat{\eta}_j}

\newcommand{\barbq}{\bar{\q}}
\newcommand{\barq}{\bar{q}}
\newcommand{\barqki}{\bar{q}_{k,i}}

\newcommand{\barbu}{\bar{\bu}}
\newcommand{\barQ}{\bar{\Q}}

\newcommand{\barw}{\bar{\w}}
\newcommand{\barW}{\bar{\W}}

\newcommand{\tZ}{\tilde{\Z}}
\newcommand{\tX}{\tilde{\X}}
\newcommand{\tQ}{\tilde{\Q}}
\newcommand{\tbq}{\tilde{\q}}
\newcommand{\tq}{\tilde{q}}
\newcommand{\tqki}{\tilde{q}_{k,i}}
\newcommand{\tqyi}{\tilde{q}_{{y_i},i}}

\newcommand{\ick}{\boldsymbol{c}_{k}^{-1}}
\newcommand{\ico}{\boldsymbol{c}_{1}^{-1}}

\newcommand{\icK}{\boldsymbol{c}_{K}^{-1}}

\newcommand{\icki}{c_{k,i}^{-1}}
\newcommand{\icko}{c_{k,1}^{-1}}
\newcommand{\ickn}{c_{k,n}^{-1}}
\newcommand{\icyi}{c_{y_i,i}^{-1}}
\newcommand{\ck}{\boldsymbol{c}_{k}}
\newcommand{\cki}{c_{k,i}}
\newcommand{\chhi}{c_{h,i}}
\newcommand{\cyi}{c_{y_i,i}}
\newcommand{\cy}{\boldsymbol{c}_{y_i}}

\newcommand{\wmni}{\w_{\MNI}}
\newcommand{\Wmni}{\W_{\MNI}}
\newcommand{\Wova}{\W_{\OVA}}
\newcommand{\Wsim}{\W_{\simplex}}

\newcommand{\wovak}{\w_{\OVA,k}}
\newcommand{\wovao}{\w_{\OVA,1}}
\newcommand{\wovaK}{\w_{\OVA,K}}
\newcommand{\wsimk}{\w_{\simplex,k}}
\newcommand{\wsimo}{\w_{\simplex,1}}
\newcommand{\wsimK}{\w_{\simplex,K}}

\newcommand{\bumni}{\bu_{\MNI}}
\newcommand{\qy}{\barbq_{\y}}
\newcommand{\ells}[1]{\ell_{#1}}

\newcommand{\xik}{\boldsymbol{\xi}_k}
\newcommand{\xiK}{\boldsymbol{\xi}_K}
\newcommand{\xio}{\boldsymbol{\xi}_1}
\newcommand{\xiki}{\xi_{k,i}}
\newcommand{\xikj}{\xi_{k,j}}
\newcommand{\xiko}{\xi_{k,1}}
\newcommand{\xikn}{\xi_{k,n}}

\newcommand{\xihi}{\xi_{h,i}}
\newcommand{\xihj}{\xi_{h,j}}
\newcommand{\xiyi}{\xi_{y_i,i}}
\newcommand{\xikis}{\seqr{\xiki}{k=1}{K}}
\newcommand{\vki}{v_{k,i}}
\newcommand{\vi}{v_i}
\newcommand{\vhi}{v_{h,i}}

\newcommand{\ptk}{\p_{t,k}}
\newcommand{\pto}{\p_{t,1}}
\newcommand{\ptK}{\p_{t,K}}
\newcommand{\ptki}{p_{t,k,i}}
\newcommand{\ptyi}{p_{t,y_i,i}}

\newcommand{\pki}{p_{k,i}}

\newcommand{\pyi}{p_{{y_i},i}}
\newcommand{\qtk}{\q_{t,k}}
\newcommand{\qto}{\q_{t,1}}
\newcommand{\qtK}{\q_{t,K}}
\newcommand{\qtki}{q_{t,k,i}}
\newcommand{\qtyi}{q_{t,y_i,i}}

\newcommand{\qki}{q_{k,i}}

\newcommand{\qyi}{q_{{y_i},i}}
\newcommand{\wk}{\w_k}
\newcommand{\wyi}{\w_{y_i}}
\newcommand{\wtk}{\w_{t,k}}

\newcommand{\Pt}{\bP_t}
\newcommand{\Pto}{\bP_{t+1}}
\newcommand{\Po}{\bP_{0}}
\newcommand{\Qt}{\Q_t}
\newcommand{\Qto}{\Q_{t+1}}
\newcommand{\Qo}{\Q_{0}}
\newcommand{\Wt}{\W_t}
\newcommand{\Wto}{\W_{t+1}}
\newcommand{\deltai}{\delta_i}
\newcommand{\deltaki}{\delta_{k,i}}
\newcommand{\deltahi}{\delta_{h,i}}
\newcommand{\R}{\mathbb{R}}
\newcommand{\Ind}{\mathbb{I}}
\newcommand{\E}{\mathbb{E}}
\newcommand{\Prob}{\mathbb{P}}

\newcommand{\Risk}{\mathcal{R}}

\newcommand{\multiloss}{\mathfrak{L}}
\newcommand{\sumn}{\sum_{i=1}^{n}}

\newcommand{\sumk}{\sum_{k=1}^{K}}
\newcommand{\sumh}{\sum_{h=1}^{K}}
\newcommand{\sumjt}{\sum_{j<t}}
\newcommand{\sumt}{\sum_{t=0}^{\infty}}
\newcommand{\sumkny}{\sum_{k\neq y_i}^{K}}
\newcommand{\sumknynk}{\sum_{k\neq y_i, k\neq k}^{K}}
\newcommand{\tinf}[1]{#1\rightarrow \infty}
\newcommand{\tninf}[1]{#1\rightarrow -\infty}
\newcommand{\tzero}[1]{#1\rightarrow 0}
\newcommand{\limt}{\underset{t\rightarrow \infty}{\lim}}
\newcommand{\lima}{\underset{a\rightarrow 0}{\lim}}
\newcommand{\umin}[1]{\underset{#1}{\min}}
\newcommand{\smin}[1]{\min_{#1}}
\newcommand{\umax}[1]{\underset{#1}{\max}}
\newcommand{\ulim}[1]{\underset{#1}{\lim}}
\newcommand{\usup}[1]{\underset{#1}{\sup}}
\newcommand{\uargmin}[1]{\underset{#1}{\arg\min\,}}
\newcommand{\sargmin}[1]{\arg\min_{#1}}
\newcommand{\usub}[2]{\underset{#1}{\underbrace{#2}}}
\newcommand{\diag}[1]{\text{diag}\left(#1\right)}
\newcommand{\diagt}[1]{\text{diag}\bigl(#1\bigl)}
\newcommand{\dom}[1]{\text{dom\,}#1}
\newcommand{\MNI}{\mathsf{MNI}}
\newcommand{\OVA}{\mathsf{OvA}}
\newcommand{\simplex}{\mathsf{simplex}}
\newcommand{\vct}[2]{\left(#1,\cdots,#2\right)^\top}
\newcommand{\vctt}[2]{\bigl(#1,\cdots,#2\bigl)^\top}
\newcommand{\cond}[4]{\left\{ \begin{matrix}
      #1 &\; #2 \\
      #3 &\; #4
    \end{matrix}\right.\,}
\newcommand{\pts}[1]{\left(#1\right)}
\newcommand{\mts}[1]{\left[#1\right]}
\newcommand{\Riskf}[1]{\Risk\left(#1\right)}
\newcommand{\gRiskf}[2][]{\nabla_{#1}\Risk\left(#2\right)}
\newcommand{\ellf}[2][]{\ell^{#1}\left(#2\right)}
\newcommand{\ellsf}[2][]{\ell_{#1}\left(#2\right)}
\newcommand{\gsf}[2][]{g_{#1}\left(#2\right)}
\newcommand{\ellpf}[2][]{\ell_{#1}'\left(#2\right)}
\newcommand{\ellppf}[2][]{\ell_{#1}''\left(#2\right)}

\newcommand{\ellpellnofz}[1]{\ell'\bigl(\ell^{-1}\bigl(#1\bigl)\bigl)}
\newcommand{\ellpellnofzo}[1]{\ell'\Bigl(\ell^{-1}\Bigl(#1\Bigl)\Bigl)}
\newcommand{\ellpellnof}[1]{\ell'\biggl(\ell^{-1}\biggl(#1\biggl)\biggl)}
\newcommand{\ellpellnoft}[1]{\ell'\Biggl(\ell^{-1}\Biggl(#1\Biggl)\Biggl)}
\newcommand{\ellppellnof}[1]{\ell''\biggl(\ell^{-1}\left(#1\right)\biggl)}

\newcommand{\ellpinvf}[1]{\left[\ell'\right]^{-1}\left(#1\right)}
\newcommand{\ip}[2]{\left\langle #1, #2\right\rangle}
\newcommand{\Lossf}[1]{\mathcal{L}\left(#1\right)}
\newcommand{\psif}[2][]{\psi^{#1}\left(#2\right)}
\newcommand{\psicf}[1]{\psi^{*}\left(#1\right)}
\newcommand{\gpsif}[2][]{\nabla_{#1}\psi\left(#2\right)}

\newcommand{\nm}[1]{\ell_#1}
\newcommand{\gf}[2][]{g^{#1}\left(#2\right)}

\newcommand{\gpf}[2][]{\left[g^{#1}\right]'\left(#2\right)}
\newcommand{\hf}[2][]{h^{#1}\left(#2\right)}
\newcommand{\bhf}[2][]{\bar{h}^{#1}\left(#2\right)}
\newcommand{\ff}[2][]{f^{#1}\left(#2\right)}

\newcommand{\nnorm}[2][]{\left\|#2\right\|_{#1}}
\newcommand{\normo}[1]{\left|#1\right|}
\newcommand{\normalize}[2][]{\frac{#2}{\left\|#2\right\|_{#1}}}
\newcommand{\seq}[1]{\left\{#1\right\}}
\newcommand{\seqr}[3]{\left\{#1\right\}_{#2}^{#3}}
\newcommand{\expf}[2][]{\exp^{#1}\left(#2\right)}
\newcommand{\lnf}[2][]{\ln^{#1}\left(#2\right)}
\newcommand{\multilossf}[1]{\multiloss\left(#1\right)}

\newcommand{\breg}[2]{D_{\psi^*}\left(#1, #2\right)}
\newcommand{\deltaf}[2]{\delta\left(#1, #2\right)}
\newcommand{\Ff}[2][]{F^{#1}\left(#2\right)}

\newcommand{\sigmaf}[2][]{\sigma^{#1}\left(#2\right)}
\newcommand{\Hess}[2][]{\boldsymbol{H}_{#2}^{#1}}
\newcommand{\hesscomp}[2]{\frac{\partial^2\psif{\bXi}}{\partial\xi_{k,{#1}}\partial\xi_{h,{#2}}}}
\newcommand{\XX}{\X\X^\top}
\newcommand{\XXi}{\bigl(\X\X^\top\bigl)^{-1}}

\newcommand{\yXXy}{\diag{\y}\X\X^\top\diag{\y}}
\newcommand{\XXiy}{\bigl(\X\X^\top\bigl)^{-1}\y}
\newcommand{\XXXy}{\X^\top\bigl(\X\X^\top\bigl)^{-1}\y}

\newcommand{\XXaI}{\XX-\alpha\I}
\newcommand{\tXtX}{\tX\tX^\top}

\newcommand{\CXXC}{\C\tX\tX^\top\C}
\newcommand{\cXXc}{\diagt{\ick}\X\X^\top\diagt{\ick}}
\newcommand{\cXXic}{\diag{\ck}\bigl(\X\X^\top\bigl)^{-1}\ck}
\newcommand{\cXXicv}{\ck^\top\bigl(\X\X^\top\bigl)^{-1}\ck}
\newcommand{\XXXc}{\X^\top\bigl(\X\X^\top\bigl)^{-1}\ck}
\newcommand{\myquad}[1][1]{\hspace*{#1em}\ignorespaces}
\newcommand{\expt}{\text{exp}}
\newcommand{\logt}{\text{log}}

\newcommand{\vst}{\text{vs}}
\newcommand{\polyt}{\text{poly}}
\newtheorem{assumption}{Assumption}
\newtheorem{theorem}{Theorem}
\newtheorem{lemma}{Lemma}
\newtheorem{proposition}{Proposition}
\newtheorem{corollary}{Corollary}
\newcommand{\kw}[1]{\textcolor{violet}{\it {\bf Kuo-Wei:} {#1}}}

\begin{document}

\doparttoc 
\faketableofcontents 

\begin{center}

{\bf{\LARGE{General Loss Functions Lead to (Approximate) Interpolation in High Dimensions}}}

\vspace*{.2in}

{\large{
\begin{tabular}{ccc}
Kuo-Wei Lai$^\dagger$ & Vidya Muthukumar$^{\dagger,\ddagger}$
\end{tabular}}}

\vspace*{.2in}

\begin{tabular}{c}
School of Electrical \& Computer Engineering, Georgia Institute of Technology$^\dagger$\\
H. Milton School of Industrial \& Systems Engineering, Georgia Institute of Technology$^\ddagger$
\end{tabular}

\vspace*{.2in}
\date{}

\end{center}

\begin{abstract}%
We provide a unified framework that applies to a general family of convex losses across binary and multiclass settings in the overparameterized regime to approximately characterize the implicit bias of gradient descent in closed form. Specifically, we show that the implicit bias is approximated (but not exactly equal to) the minimum-norm interpolation in high dimensions, which arises from training on the squared loss.
In contrast to prior work, which was tailored to exponentially-tailed losses and used the intermediate support-vector-machine formulation, our framework directly builds on the primal-dual analysis of~\cite{ji2021characterizing}, allowing us to provide new approximate equivalences for general convex losses through a novel sensitivity analysis.
Our framework also recovers existing exact equivalence results for exponentially-tailed losses across binary and multiclass settings.
Finally, we provide evidence for the tightness of our techniques and use our results to demonstrate the effect of certain loss functions designed for \emph{out-of-distribution} problems on the closed-form solution. 
\end{abstract}


\section{Introduction}\label{sec:intro}

The choice of loss function to optimize a model over training examples is an important cornerstone of the machine learning (ML) pipeline.
This choice is particularly nuanced for the task of classification, which is evaluated by the 0-1 risk on test data.
An elegant classical viewpoint is that training loss functions should be designed as continuous and optimizable surrogates~\citep{bartlett2006convexity,zhang2004statistical,lugosi2004bayes,steinwart2005consistency} to the 0-1 risk, as the training surrogate loss can often be related to the test surrogate risk, and the test surrogate risk can in turn be related to the test 0-1 risk.
However, the first part of this reasoning breaks down in the modern high-dimensional regime, where infinitely many solutions can achieve zero training loss, but the test risk widely varies across these solutions~\citep{zhang2021understanding,neyshabur2014search}.

The goal of this work is to provide a more transparent understanding of the impact of the training loss function on the eventual solution (and, thereby, its generalization) in this high-dimensional regime.
Recent empirical and theoretical work provides a mixed and incomplete picture of the impact of loss.
On one hand, large-scale empirical studies~\citep{hui2020evaluation,kline2005revisiting,golik2013cross,janocha2017loss} have shown that the less popular squared loss generates surprisingly competitive performance to the popular cross-entropy loss (the multiclass extension of the binary logistic loss).
On the other hand, the cross-entropy loss (and, more generally, the family of \emph{exponentially-tailed losses}~\citep{soudry2018implicit,ji2019implicit}) is the only one that admits a direct relationship with maximization of the worst-case training data margin, which often correlates with good generalization~\citep{bartlett1998boosting,bartlett2002rademacher}.
The empirically more challenging task of \emph{out-of-distribution (OOD) generalization}~\citep{mansour2008domain} yields further subtleties, with a diversity of loss functions that deviate significantly from this standard family of exponentially-tailed losses being recently designed and evaluated~\citep{sagawa2019distributionally,cao2019learning,menon2020long,kini2021label,wang2021importance}.
Even for high-dimensional linear models, a comprehensive theory for the impact of a general loss function on the ensuing solution (and, thereby, its generalization) is currently missing.
While promising frameworks have been recently provided for the implicit bias of general losses through convex programming~\citep{ji2020gradient,ji2021characterizing}, the properties of the implicit bias itself remain opaque. 
A separate recent line of work~\citep{muthukumar2021classification,hsu2021proliferation,wang2022binary,wang2021benign,cao2021risk} shows that the squared loss and cross-entropy loss can yield identical solutions with high probability in high dimensions, complementing their aforementioned noticed similarities in empirical performance.
In particular, both solutions are shown to exactly coincide with \emph{minimum-norm interpolation} (MNI), which enjoys a closed-form expression and often generalizes well in high dimensions~\citep{bartlett2020benign,belkin2020two,hastie2022surprises,kobak2020optimal,muthukumar2020harmless,muthukumar2021classification}.
However, these proof techniques are highly tailored to exponentially-tailed losses and in particular the intermediate support-vector-machine (SVM) formulation~\citep{soudry2018implicit}, leaving open whether such equivalences can be proved for more general losses.

\paragraph{Our contributions:}
In this paper we characterize the closed-form properties of the implicit bias of general convex losses arising from gradient descent in high dimensional linear models, by building on the primal-dual characterization of the implicit bias provided in~\cite{ji2021characterizing}.
\textbf{In Section~\ref{sec:binarymainresults}} we show (Proposition~\ref{thm:b_eqv_thm} and Theorem~\ref{thm:b_sen_upperbound}) that general convex losses in conjunction with gradient descent yield solutions that are \emph{approximately directionally close} to minimum-norm interpolation (MNI) on binary labels in a sufficiently high-dimensional regime with high probability.
 Our approximation error term is a decreasing function of an ``effective dimension" which also appears in sufficient and necessary conditions for exact equivalence between the SVM and MNI~\citep{hsu2021proliferation,ardeshir2021support}.
In contrast to all prior literature that works with the SVM, our analysis directly leverages the primal-dual framework of~\cite{ji2021characterizing}, allowing us to recover the exact equivalence to MNI for exponentially-tailed losses~\citep{hsu2021proliferation} through an alternative proof technique. Our upper bounds on the approximation error utilize a novel sensitivity analysis of the dual implicit bias in high dimensions and are applicable to general convex losses.

\textbf{In Section~\ref{sec:multiclassmainresults}} we extend our framework and analysis in binary classification to the multiclass classification where the primal-dual analysis in~\cite{ji2021characterizing} can be naturally extended.
We also treat the cross-entropy loss separately and provide an alternative proof of exact equivalence to MNI that is conceptually simpler than the one provided in~\cite{wang2021benign}, in particular, not requiring any reparameterization of the dual.


Finally, \textbf{in Section~\ref{sec:converse}} we provide partial evidence for the tightness of our arguments.
First, in Proposition~\ref{prop:converseexact} we show that the conditions for exact equivalence in Theorem~\ref{thm:b_eqv_thm} are not only sufficient but necessary.
We leverage this converse result to make an interpretable link between the popular techniques of \emph{importance-weighting} on heavy-tailed losses~\citep{wang2021importance} and vector-scaling of exponentially-tailed losses~\citep{kini2021label} and a type of \emph{cost-sensitive interpolation}, thereby providing a possible explanation for their success in addressing OOD generalization. 
Finally, under further assumptions on the data covariance, we provide a lower bound in Proposition~\ref{prop:b_sen_lowerbound} that in some sense ``matches" the upper bound of Theorem~\ref{thm:b_sen_upperbound}.


\subsection{Related work}\label{sec:relatedwork}



We organize our discussion of related work under three verticals.

\paragraph{Classical perspectives on loss function design:}

There are two classical perspectives on loss function design for classification. The first, supported by decades of research in the statistics community, advocates for choosing the loss function to match the negative logarithm of the maximum likelihood function and requires knowledge of the family of conditional distributions of the label.
For binary (multiclass) labels, a popular family of conditional distributions is given by the \emph{logistic} (multinomial) model, which yields the empirically popular choice of the \emph{logistic (cross-entropy) loss}.
The second and relatively more recent perspective, pioneered by the papers~\citep{bartlett2006convexity,zhang2004statistical,lugosi2004bayes,steinwart2005consistency}, advocates for designing \emph{continuous surrogates} to the discontinuous 0-1 test risk such that a bound on the 0-1 test risk can be easily obtained by inverting a bound on the surrogate test risk.
In an indirect sense, this perspective suggests a type of equivalence in surrogate loss functions in terms of ensuing generalization bounds.
However, principally because of the reliance on empirical-process-theory (to relate in turn the surrogate test risk to the surrogate training loss), this reasoning can frequently break down in high-dimensional settings, particularly when \emph{perfectly fitting}, or \emph{interpolating} models are considered.
This is because infinitely many models interpolate the training data, but each of them suffers a different test risk that is fundamentally unrelated to the training loss.
On the other hand, while the relations between test risks (e.g.~\cite[Theorems 1 and 3]{bartlett2006convexity}) remain universally applicable, they also suffer from some shortcomings in high-dimensional settings --- in particular, they are only powerful enough to provide faster \emph{statistical rates} for classification tasks as compared to parameter recovery~\citep{audibert2007fast}, rather than full separations in asymptotic consistency (many classic examples of such separations are considered in~\cite{devroye2013probabilistic}, but such separations were also shown more recently in the overparameterized regime in~\cite{muthukumar2021classification}). 
The first statistical perspective is similarly not prescriptive in the high-dimensional regime where the maximum-likelihood estimator is no longer unique, and training loss, again, cannot be related to test risk.


\paragraph{Implicit bias characterization of optimization algorithms:}
In the modern high-dimensional regime, infinitely many solutions achieve zero training loss for most canonical choices of training loss functions. Therefore, it is not only the loss function but also the choice of optimization algorithm that determines the eventual solution, commonly called the \emph{implicit bias}.
An extensive body of work implicitly characterizes this implicit bias of optimization algorithms as solutions to various convex programs~\citep{telgarsky2013margins,soudry2018implicit,ji2019implicit,ji2020gradient,ji2021characterizing,gunasekar2018characterizing,gunasekar2018implicit,woodworth2020kernel,nacson2019convergence}.
The convex program formulation typically does not admit a closed-form solution, except for gradient descent and the squared loss (which yields the MNI for linear models~\citep{engl1996regularization}).
Early work here was tailored to exponentially-tailed losses~\citep{soudry2018implicit,ji2019implicit}, and their established equivalence to the MNI and thereby the squared loss~\citep{muthukumar2021classification,hsu2021proliferation,wang2022binary,wang2021benign,cao2021risk} in turn heavily rely on the intermediate SVM formulation.
The more recent works~\citep{nacson2019convergence,ji2020gradient,ji2021characterizing} study some non-exponential losses, but leave the exact nature of the implicit bias somewhat mysterious, other than that the ensuing convex program no longer corresponds to the max-margin SVM.
For example,~\cite[Figure 1]{ji2020gradient} provides a simulated example for which exponential and polynomial losses induce very different directions, and~\cite[Proposition 12]{ji2020gradient} provides an example under which the training data margin can be arbitrarily worse for polynomial losses. These are specialized examples of $2$-dimensional data that is linearly separable; therefore, do not apply to the high-dimensional regime of interest.
Whether such heavy-tailed losses are actually provably worse than exponentially-tailed losses is left open.
Our results in this work imply intriguing similarities, but also differences, between heavy-tailed losses and exponential losses in the high-dimensional regime.

The recent papers~\citep{ji2020gradient,ji2021characterizing} provide promising avenues to understanding the nature of the implicit bias by formulating convex programs for general losses.
\cite{ji2020gradient} make minimal assumptions on the loss function beyond convexity and differentiability, and characterize the implicit bias as the limit of a set of solutions to convex programs that minimize the training loss subject to an $\ell_2$-norm constraint of increasing radius (i.e. a \emph{regularization path}). 
\cite[Appendix A]{wang2021importance} show for polynomially-tailed losses that this limit can itself be written as the solution to an explicit convex program, but their proof is tailored to polynomially-tailed losses and in particular their property of positive homogeneity --- moreover, no closed-form characterization is provided.
On the other hand,~\cite{ji2021characterizing} make slightly stronger assumptions on the loss function, but provide a clearer path to characterizing a closed-form solution for the implicit bias by understanding its \emph{mirror-descent dual} as a solution to an explicit convex program (i.e.~not a limit of solutions to convex programs on the regularization path).
It is thus natural to attempt to obtain closed-form expressions for the ``primal" implicit bias by understanding its ``dual" for general losses\footnote{This is especially true given that the mirror-descent dual for the case of exponentially-tailed losses turns out to exactly correspond to a scalar multiple of the SVM dual. Indeed, the proofs of SVM equivalence all construct a dual witness.}. 
A second advantage with analyzing the mirror-descent dual is that we show it automatically yields the non-trivial variable substitution of the multiclass SVM dual that was made in~\cite{wang2021benign}, resulting in a conceptually simpler proof of SVM equivalence to MNI for the cross-entropy loss.
We also show that the primal-dual analysis is applicable to more general formulations of multiclass losses~\citep{zhang2004statistical,tewari2007consistency,ji2021fast}.

\paragraph{Generalization analysis of interpolating predictors in high dimensions:}

A comprehensive theory for overparameterized models arising from training with the squared loss (i.e.~the MNI) was provided in work beginning with the papers that analyzed the test regression risk~\citep{bartlett2020benign,belkin2020two,hastie2022surprises,kobak2020optimal,muthukumar2020harmless}.
This theory critically utilizes the closed-form expression for the MNI.
Sharply analyzing the classification risk poses distinct challenges, the most daunting of which is the lack of a closed-form expression for the solution arising from any other convex loss function used for classification.
To tackle this challenge for the special case of exponential losses,~\cite{muthukumar2021classification} introduced a two-step recipe.
First, they related the implicit bias of exponential losses (i.e. the SVM) to the MNI --- in fact, by showing an \emph{exact equivalence} result (which was since improved on by~\cite{hsu2021proliferation}).
Second, they sharply analyzed the classification test risk of the MNI and showed that it can achieve classification-consistency even when a corresponding regression task would not be consistent.
It is worth noting that this type of consistency result cannot be easily recovered through any generalization bound that relies on empirical-process-theory, including margin-based data-dependent generalization bounds (as described in~\cite[Section 6]{muthukumar2021classification}).
This recipe was since applied to binary and multiclass Gaussian and sub-Gaussian mixture models to identify new high-dimensional regimes in which classification-consistency is possible~\citep{wang2022binary,cao2021risk,wang2021benign,subramanian2022generalization}.
To be able to apply this recipe to more general losses, corresponding equivalences would need to be established between general losses and the MNI, which is the focus of this paper.

Other than the approach described above, two other families of techniques are prevalent in the recent literature.
The first applies to proportionally high-dimensional regimes (where $d \propto n$) and directly characterizes the limiting test risk as $(d,n) \to \infty$ as the solution to a system of nonlinear equations, beginning with the efforts tailored to exponential or exponentially tailed losses~\citep{huang2017asymptotic,sur2019modern,mai2019large,salehi2019impact,deng2022model,montanari2019generalization}.
More recently,~\cite{loureiro2021learning} provide precise asymptotic analysis for general losses and multiclass classification for Gaussian mixture models for \emph{regularized empirical risk minimization} with general losses and regularizers\footnote{Note that this covers the implicit bias of gradient descent when the regularization proportion $\lambda \to 0$ due to the results of~\cite{ji2020gradient}.}.
However, they do not examine in detail the impact of loss functions on performance.
In general, none of our results for general losses have direct implications for this proportional regime.
However, we believe the auxiliary convex program proposed in Lemma~\ref{lem:b_imp_bias_relax} might be of independent interest, particularly for the subset of approaches above that utilize Gordon's comparison theorems and the convex Gaussian min-max theorem.

The second technique was proposed by~\cite{chatterji2021finite} for directly analyzing the generalization error of the implicit bias of exponential losses on a sub-Gaussian mixture model.
The key technical innovation is to prove a ``loss ratio" bound: under sufficiently overparameterized settings,~\cite{chatterji2021finite} show that the training losses of any two examples are within a constant factor of each other throughout the optimization path of gradient descent.
This proof technique is quite generally applicable and was since used for polynomially-tailed losses~\citep{wang2021importance}, deep linear networks~\citep{chatterji2022interplay} and certain 2-layer neural networks on high-dimensional data~\citep{frei2022benign}. 
However, the loss-ratio bound often requires a much larger data dimension $d \gg n^2$ to hold as compared to the MNI-equivalence approach to analyzing the SVM, as shown explicitly in~\citep{wang2022binary}.
It is not clear whether this dimension requirement is tight even in the worst case.
A natural question of interest is whether a loss-ratio bound implies exact or approximate equivalence of solutions, or vice versa.
\cite{behnia2022avoid} showed recently that a loss-ratio bound can imply exact equivalence to the MNI in the case of exponential losses, but this is a research direction that is otherwise largely unexplored.

\paragraph{Comparison to related work:} In Table~\ref{tab:comparison}, we succinctly situate our work in the literature on the implicit bias of classification-oriented loss functions.
In sum, we go \emph{beyond worst-case characterizations} (by investigating an approximate equivalence to the MNI under sufficiently high-dimensional random data) of the implicit bias of gradient descent on \emph{general convex loss functions} (going beyond previous work that only established an approximate equivalence for the class of exponentially-tailed losses). 
While the beyond-worst-case aspect had been previously explored on exponentially-tailed losses~\citep{muthukumar2021classification,hsu2021proliferation,wang2021benign}, and a worst-case characterization of general losses was provided~\citep{ji2020gradient,ji2021characterizing}, prior to our work these had not been studied together.
Our starting point for analyzing the implicit bias of general losses is the insightful \emph{dual} convex program characterization provided by~\cite{ji2021characterizing}.
We introduce several novel ideas over and above their work; prominent among them a new, and simpler to analyze, auxiliary convex program for the dual (Lemma~\ref{lem:b_imp_bias_relax}), as well as a new sensitivity analysis of this auxiliary program that is ``fixed-design" in nature (Theorem~\ref{thm:b_sen_upperbound}).
Our main sensitivity theorem can easily be applied in conjunction with standard results on high-dimensional probability, e.g. random matrix concentration, to establish approximate equivalence to the MNI for general losses and a variety of random data models (Corollary~\ref{cor:convergencehighdims}).

\begin{table}[htbp]
    \renewcommand{\arraystretch}{1}
    \caption{Our result, contextualized in related implicit bias literature. 
    }
    \label{tab:comparison}
    \centering
    \resizebox{1\textwidth}{!}{\begin{tabular}{|c | c | c |} 
    \hline
 & \begin{tabular}{c}Implicit bias\\for worst-case data\end{tabular} & \begin{tabular}{c}Implicit bias vs. MNI\\under high-dimensional random data\end{tabular} \\
    \hline
    \begin{tabular}{c}Exponentially-tailed losses\\(e.g. exponential loss, logistic loss)\end{tabular} & \begin{tabular}{c}Soudry et al.~\cite{soudry2018implicit}\\ Ji and Telgarsky~\cite{ji2019implicit}\\ Ravi et al.~\cite{ravi2024implicit}\end{tabular} & \begin{tabular}{c}Muthukumar et al.~\cite{muthukumar2021classification} \\Hsu et al.~\cite{hsu2021proliferation} \\Wang et al.~\cite{wang2021benign}\end{tabular} \\
    \hline\begin{tabular}{c}General convex losses\\(e.g. polynomially-tailed loss)\end{tabular} & \begin{tabular}{c}Ji et al.~\cite{ji2020gradient}\\Ji and Telgarsky~\cite{ji2021characterizing}\end{tabular} & This work \\
    \hline
    \end{tabular}}
    
\end{table}



\paragraph{Notation:} We use lower-case boldface (e.g. $\x$) to denote vector notation and upper-case boldface (e.g. $\X$) to denote matrix notation.
We use $\nnorm[p]{\cdot}$ to denote the $\ell_p$-norm of a vector for $p \in [1,\infty)$ and $\nnorm[2]{\cdot}$ to additionally denote the operator norm of a matrix.
$\diag{\x}$ denotes the diagonal matrix whose entries are given by the vector $\x$.
For a 1-dimensional function $h(\cdot): \R \to \R$, we frequently overload notation and denote its element-wise operation on a vector by 
 $h(\x) :=\vct{h(x_1)}{h(x_n)}$.
All other appearances of the notation $h(\x): \R^n \to \R^k$ instead denote a function that takes a vector-valued argument.
We denote first and second derivatives by $'$ and $''$ respectively, and use $\partial$ to denote a partial derivative.
We use the shorthand notation $[n]$ to denote the set of natural numbers $\{1,\ldots,n\}$.

\section{Approximate Equivalences for Binary Classification}
Since our results build on the primal-dual analysis presented in \cite{ji2021characterizing}, we reproduce their assumptions on the data and loss function below.

\paragraph{Problem setup.} We consider a labeled dataset $\seqr{\x_i, y_i}{i=1}{n}$, where $\x_i \in \R^d$ satisfies the normalization $\nnorm[2]{\x_i} \leq 1$ (which can be done without loss of generality) and the labels $y_i \in \seq{-1, 1}$ are binary.
We denote $\X = \vct{\x_1}{\x_n}\in\R^{n \times d}$ and $\y := \vct{y_1}{y_n}\in\R^n$. We focus on an unbounded, unregularized empirical risk minimization (ERM) problem with a margin-based loss function and a linear classifier:
\begin{align}\label{eq:erm}
    \umin{\w\in\R^d} \Riskf{\w} := \frac{1}{n}\sumn \ellf{-y_i\ip{\w}{\x_i}}=\frac{1}{n}\sumn \ellf{y_i \ip{\w}{\z_i}},
\end{align}
where we denote $\z_i := -\x_i$, $\Z:=-\X$, and $\w\in\R^d$ is the set of parameters of the linear classifier. 

\begin{assumption}[\cite{ji2021characterizing}]\label{asm:b_loss_assump}
    The loss function $\ell(\cdot)$ is twice differentiable, and satisfies:
    \begin{enumerate}
        \item $\ell$, $\ell'$, $\ell''>0$, and $\ulim{\tninf{z}}\ellf{z}=0$.
        \item $z\ellpf{z}/\ellf{z}$ is increasing on $(-\infty, 0)$, and $\ulim{\tninf{z}}z\ellpf{z}=0$.
        \item For all $b\geq1$, there exists $c>0$ (which may depend on $b$), such that for all $a>0$, we have $\ellpellnofz{a}/\ellpellnofz{ab}\geq c$. 
        \item Given $\bxi \in \R^n$, we define
            \begin{align*}
                \Lossf{\bxi} := \sumn\ellf{\xi_i}, \text{ and } \psif{\bxi} := \ellf[-1]{\Lossf{\bxi}},
            \end{align*}
        and the ``generalized sum" $\psi$ is convex and $\beta$-smooth with respect to $\nm{\infty}$ norm.
\end{enumerate}
\end{assumption}

Next, we show that for any loss function $\ellf{\cdot}$ that satisfies Assumption~\ref{asm:b_loss_assump}, there exists an explicit analytical function $\gf{\cdot}$, derived as the limit of a certain ratio of derivatives of inverses of the loss function $\ellf{\cdot}$, that will be instrumental in our analysis of the implicit bias. This lemma is a direct implication of Assumption~\ref{asm:b_loss_assump}, without any additional assumptions.

\begin{lemma}\label{lem:g_func}
    Under Assumption~\ref{asm:b_loss_assump}, the limit $\ulim{a \to 0} \frac{\ellpellnofz{a\cdot z}}{\ellpellnofz{a}}$ exists for every $0 < z \leq 1$. Moreover, there exists a function $\gf{\cdot}$ such that $\gf{z} \coloneqq \ulim{a \to 0} \frac{\ellpellnofz{a\cdot z}}{\ellpellnofz{a}}$ for $0 < z \leq 1$, where $g: (0,1] \rightarrow (0,1]$ is a non-negative, strictly increasing, convex function satisfying $\gf{1}=1$.
\end{lemma}

The proof of Lemma~\ref{lem:g_func} can be found in Appendix~\ref{sec:g_func_proof}. Lemma~\ref{lem:g_func} is central to all of our results, since different loss functions $\ellf{\cdot}$ may result in different functions $\gf{\cdot}$. In particular, we critically use the convexity of the function $\gf{\cdot}$ to obtain a simplified auxiliary convex program, that is equivalent in optimal solution, underlying the dual of the implicit bias. Figure~\ref{fig:g_func} displays various examples of the form the function $\gf{\cdot}$ takes for specific, commonly used loss functions.


\paragraph{The implicit bias formulation.} We use the gradient descent algorithm to solve this unregularized empirical risk minimization problem with initial weights $\w_0$ and the update rule: $\w_{t+1} := \w_t - \eta_t \gRiskf{\w_t}$ for $t\geq0$. We also denote, in the context of mirror-descent analysis, the ``primal" $\p_t:=\diag{\y}\Z\w_t\in\R^n$ and its corresponding ``dual" $\q_t:=\gpsif{\p_t} \in \R^n$, where
\begin{align}\label{eq:dual-primal-t}
    q_{t,i}
    =\frac{\ellpf{p_{t,i}}}{\ellpellnofz{\sumn\ellf{p_{t,i}}}}
    =\frac{\ellpf{p_{t,i}}}{\ellpf{\psif{\p_t}}}.
\end{align}

These mirror-descent primal and dual terms were defined in~\cite{ji2021characterizing}. 
Next, we assume that the data can be interpolated or perfectly fitted, which corresponds to a full-rank assumption on the Gram matrix $\XX$.
Note that this in turn implies that the dataset is linearly separable.
This full-rank assumption is satisfied with high probability in the overparameterized regime $d \gg n$ for most canonical data distributions; see, e.g.~\cite{hsu2021proliferation}.
\begin{assumption} \label{asm:b_lin_sep}
    We assume that $d \geq n$ and the data Gram matrix satisfies $\XX \succ \zero$. This in turn implies that there exists a linear separator $\bu \in \R^d$ that $y_i\ip{\bu}{\x_i}>0$ for all $i\in[n]$.
\end{assumption}

We restate the primal-dual implicit bias formulation of~\cite[Theorem 5]{ji2021characterizing} below.

\begin{lemma}[\cite{ji2021characterizing}] \label{lem:b_imp_bias}
    Under Assumptions~\ref{asm:b_loss_assump} and \ref{asm:b_lin_sep}, and provided that $\hetat \coloneqq \eta_t \ellpf{\psif{\p_t}}/n \leq 1/\beta$ is nonincreasing and $\sumt\hetat=\infty$, the primal and dual implicit bias $(\barw,\barbq)$ are given by:
    \begin{gather} \label{eq:b_imp_bias_eq}
        \barw:=\ulim{\tinf{t}}\normalize[2]{\w_t}=\normalize[2]{-\Z^\top\diag{\y}\barbq}=\normalize[2]{\X^\top\diag{\y}\barbq},
    \end{gather}    
    and 
    \begin{gather}\label{eq:b_barq_def}
        \barbq \in \uargmin{\psif[*]{\q}\leq 0}\ff{\q},
    \end{gather}
    where $\psi^*$ denotes the convex conjugate of $\psi$, and we define $\ff{\q}:=\frac{1}{2}\nnorm[2]{\X^\top\diag{\y}\q}^2$.
\end{lemma}

Consequently, a characterization of any 
solution $\barbq$ to the convex program~\eqref{eq:b_barq_def} defined in Lemma~\ref{lem:b_imp_bias} would directly characterize the desired primal implicit bias $\barw$.
Accordingly, our techniques and results largely focus on characterizing a suitable solution $\barbq$ to~\eqref{eq:b_barq_def}.


\paragraph{Minimum-norm interpolation:} We are especially interested in relating the primal implicit bias $\barw$ to the minimum-norm interpolation (MNI) $\wmni:=\XXXy$.
The MNI arises as the implicit bias of gradient descent applied to the square loss under a sufficiently small step size and initialization $\w_0 = \zero$~\citep{engl1996regularization}.
For example, it is easy to see that the candidate dual solution $\q := \diag{\y}\XXiy$ would correspond to a primal solution proportional to $\wmni$; we will utilize this candidate solution in our equivalence results.


\subsection{Main Results}\label{sec:binarymainresults}

The convex program defined in~\eqref{eq:b_barq_def} is challenging to directly work with and analyze.
This is primarily because the convex conjugate constraint $\psif[*]{\q}$ is in general an implicitly defined function on $\q$ (except for the exact exponential loss as shown in~\cite{ji2021characterizing}), and therefore its non-positivity can be difficult to verify.
To make progress, we present a simple but critical auxiliary convex program that recovers the same dual implicit bias solution in Lemma~\ref{lem:b_imp_bias_relax} that critically utilizes the convex function $\gf{\cdot}$ that we defined in Lemma~\ref{lem:g_func}.


\begin{lemma} \label{lem:b_imp_bias_relax}
    Under Assumptions~\ref{asm:b_loss_assump} and~\ref{asm:b_lin_sep}, any solution to the auxiliary convex program
    \begin{align} \label{eq:b_imp_bias_relax_eq}
        \barbq \in &\,\uargmin{\q\in\R^n}\usub{\ff{\q}}{\frac{1}{2}\q^\top\yXXy\q} \\
            \text{subject to} \myquad[1] &-q_i < 0 \myquad[1] \text{for all } i\in[n],
            \text{ and } \myquad[1] 1-\sumn \gf[-1]{q_i}  \leq 0\nonumber,
    \end{align}
    is also an optimal solution to the original convex program~\eqref{eq:b_barq_def}.
\end{lemma}
The full proof for Lemma~\ref{lem:b_imp_bias_relax} is contained in Appendix~\ref{sec:b_relaxproofs}.
%
The proof of Lemma~\ref{lem:b_imp_bias_relax} follows via a two-part argument. We first show that the convex conjugate constraint in the convex program~\eqref{eq:b_barq_def} must be active at optimality, which implies that $\psif[*]{\barbq}=0$.
We then demonstrate that the condition $\sumn \gf[-1]{q_i}=1$, derived from the Karush-Kuhn-Tucker (KKT)~\citep{karush1939minima} conditions for the auxiliary convex program~\eqref{eq:b_imp_bias_relax_eq}, is sufficient to ensure that $\psif[*]{\q}=0$. Next, we show that any solution to the original convex program~\eqref{eq:b_barq_def} also satisfies $\sumn \gf[-1]{q_i}=1$. Therefore, every solution to the auxiliary convex program~\eqref{eq:b_imp_bias_relax_eq} is also a solution to the original program~\eqref{eq:b_barq_def}. The idea is illustrated in Figure~\ref{fig:convex_prog}.

\begin{figure}
\centering 
\begin{minipage}[t]{0.45\textwidth}
 \includegraphics[width=70mm]{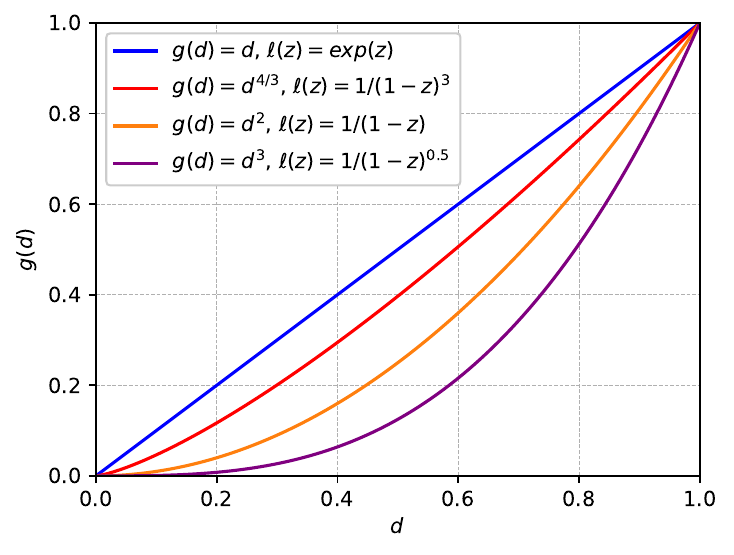}
 \caption{Plots of $g(\cdot)$ for different losses.} \label{fig:g_func}
\end{minipage}
\hfill
\begin{minipage}[t]{.45\textwidth}
 \includegraphics[width=65mm]{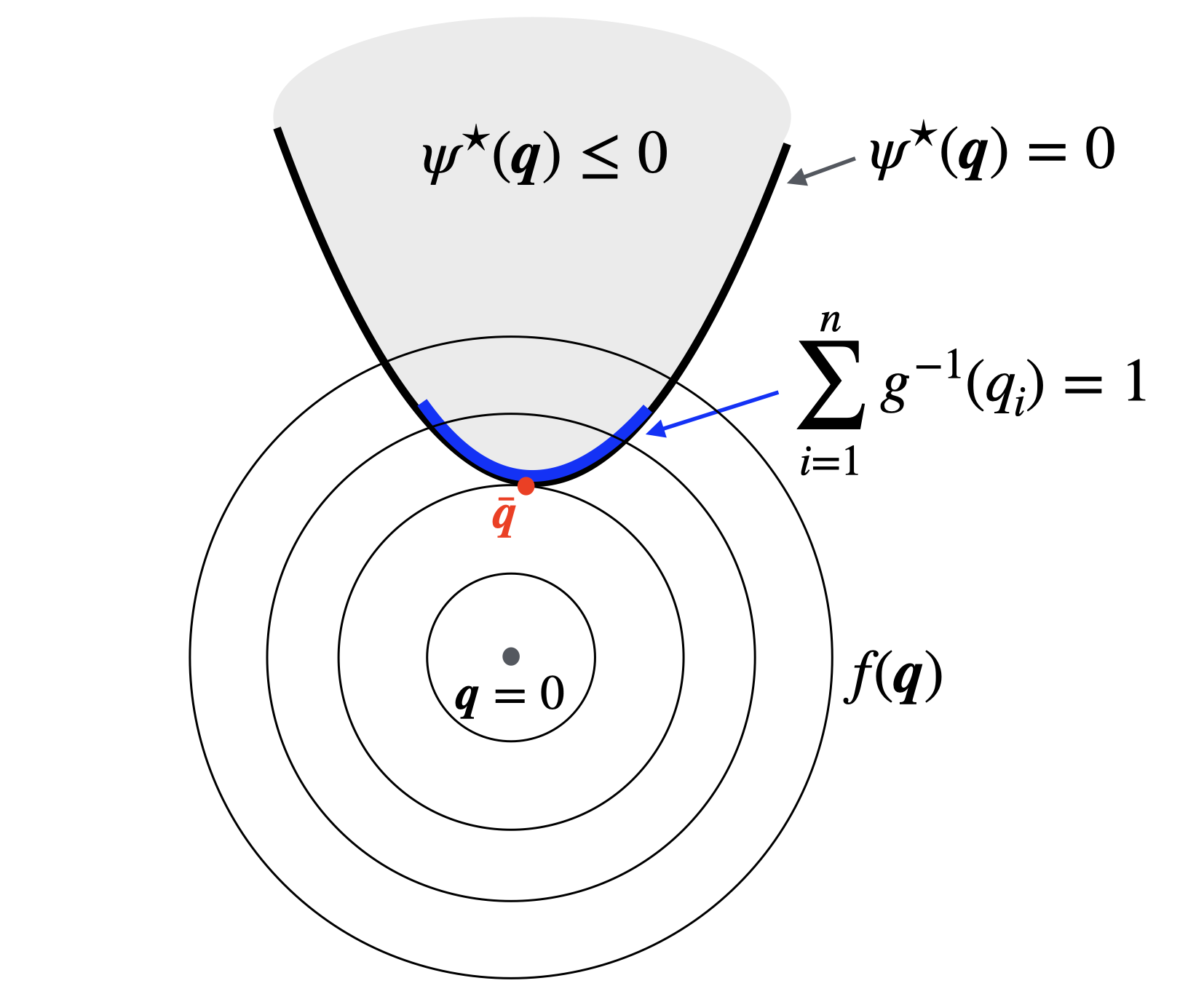}
 \caption{Illustration of the original convex program~\eqref{eq:b_barq_def} and how it relates to the auxiliary convex program~\eqref{eq:b_imp_bias_relax_eq}.} \label{fig:convex_prog}
\end{minipage}
\end{figure}

\subsubsection{Warm-up: Conditions for exact equivalence to MNI}\label{sec:binaryexact}

Although the auxiliary convex program in~\eqref{eq:b_imp_bias_relax_eq} is simpler to analyze, it still does not admit a closed-form solution in general.
We begin by providing a warm-up result characterizing settings under which~\eqref{eq:b_imp_bias_relax_eq} does admit a closed-form solution, which turns out to yield the MNI primal $\wmni$.


\begin{proposition} \label{thm:b_eqv_thm}
    Under Assumptions~\ref{asm:b_loss_assump} and \ref{asm:b_lin_sep}, the following statements hold:
    \begin{enumerate}
        \item If $\y$ is an exact eigenvector of $\XX$, the implicit bias $\barw$ is parallel to the MNI $\wmni$, i.e. $\normalize[2]{\barw} = \normalize[2]{\wmni}$.
        \item For any loss function that admits the identity function $g(d) = d$, the implicit bias $\barw$ is parallel to the MNI $\wmni$, i.e. $\normalize[2]{\barw} = \normalize[2]{\wmni}$ iff
        \begin{align}\label{eq:svmequivalence}
            \XX \succ \zero \text{ and } \bbeta:=\XXiy \text{ satisfies } y_i\beta_i>0 \text{ for all } i \in [n].
        \end{align} \label{thm:b_exp_eqv}
    \end{enumerate}
\end{proposition}

The full proof of Proposition~\ref{thm:b_eqv_thm} is provided in Appendix~\ref{sec:beqvthmproof} and works directly with the KKT conditions of the auxiliary convex program~\eqref{eq:b_imp_bias_relax_eq}.
We make a few remarks here about this proposition.
First, note that Part 2 of Proposition~\ref{thm:b_eqv_thm} recovers the sufficient and necessary condition for the equivalence between the SVM and the MNI, i.e. \emph{support-vector-proliferation} (SVP) originally studied in~\cite{muthukumar2021classification,hsu2021proliferation}.
This makes sense, as the class of loss functions that admits the identity function $g(d) = d$ corresponds to the class of \emph{exponentially-tailed losses}, which are well-known to generate implicit bias that is parallel to the SVM~\citep{soudry2018implicit}.
Next, note that the condition for general losses in Part 1 (that $\y$ is an exact eigenvector of $\XX$) is significantly stronger than the condition in Part 2 --- while $\y$ being an exact eigenvector of $\XX$ implies Eq.~\eqref{eq:svmequivalence}, the reverse implication does not hold.
We show in Proposition~\ref{prop:converseexact} in Section~\ref{sec:converse} that the exact-eigenvector condition is in fact necessary for any loss function that does not admit the identity function $g(d) = d$.
Finally, we informally remark on some sufficient conditions under which the exact-eigenvector condition would hold.
One easily verifiable case is when the Gram matrix is an exact multiple of the identity, as stated below.
\begin{corollary}\label{cor:identityexact}
If $\XX = \alpha \I$ for some $\alpha > 0$, then we have $\barbq\propto \one$ and $\barw \propto \wmni$ for any loss satisfying Assumption~\ref{asm:b_loss_assump}.
\end{corollary}

Corollary~\ref{cor:identityexact} describes a scenario that will not arise in practice, as in general the Gram matrix $\X$ will be random. \cite{muthukumar2020harmless} showed that the scenario $\XX = \alpha \I$ can, however, arise with data that is \emph{uniformly spaced} in conjunction with certain feature families.
Uniformly-spaced data models also appear in some pedagogical analyses of nonparametric statistics, as they often provide a simpler analysis as compared to random data~\citep{nemirovski2000topics,tsybakov2009nonparametric}.

\subsubsection{Main result: Approximate equivalence to MNI in high dimensions}\label{sec:binaryapprox}

We now turn to more realistic scenarios to handle random data.
In general, we only expect the Gram matrix to be \emph{close} to a multiple of the identity (in the sense that the operator norm of the difference $\nnorm[2]{\XX - \alpha \I}$ is typically controlled in high dimensions). This leads to whether the solution $\barw$ is now close in its direction to $\wmni$. 
Theorem~\ref{thm:b_sen_upperbound} below addresses this question.
\begin{theorem} \label{thm:b_sen_upperbound}
    Under Assumptions~\ref{asm:b_loss_assump} and \ref{asm:b_lin_sep},
    consider any value of $\alpha > 0$ satisfying\\ $\frac{\nnorm[2]{\XX-\alpha\I}}{\alpha} \leq \frac{1}{3}$.
    Then, the implicit bias $\barw$ converges in direction to $\wmni$ at the rate
    \begin{align}\label{eq:b_sen_upperbound}
    \nnorm[2]{\normalize[2]{\barw} - \normalize[2]{\wmni}} &\leq \frac{C \nnorm[2]{\XX \y - \alpha \y}}{\alpha \nnorm[2]{\y}},
    \end{align}
    where $C$ is a universal constant that does not depend on $\alpha,\X$ or $\y$.
\end{theorem}

Theorem~\ref{thm:b_sen_upperbound} shows that every loss function satisfying Assumption~\ref{asm:b_loss_assump} yields an approximately equivalent implicit bias in high dimensions. 
It also recovers Corollary~\ref{cor:identityexact} as a special case (as in this case the RHS of Eq.~\eqref{eq:b_sen_upperbound} becomes equal to $0$).

Before discussing how to prove Theorem~\ref{thm:b_sen_upperbound}, we describe a canonical high-dimensional statistical ensemble under which it implies directional convergence of the implicit bias $\barw$ to the MNI $\wmni$.
\begin{corollary}\label{cor:convergencehighdims}
    Assume independent and identically distributed data $\{\x_i,y_i\}_{i=1}^n$ such that each covariate satisfies one of the following: a) $\x_i \sim \mathcal{N}(\mathbf{0}, \boldsymbol{\Sigma})$, and we denote the spectrum of $\boldsymbol{\Sigma}$ by $\blambda$; \emph{or} b) $\x_i = \diag{\blambda}^{1/2} \z_i$, where $\z_i$ has independent entries such that each $z_{ij}$ is mean-zero, unit-variance, and sub-Gaussian with parameter $v > 0$ (i.e.~$\E[z_{ij}] = 0,\E[z_{ij}^2] = 1$, and $\E[e^{tz_{ij}}] \leq e^{vt^2/2}$ for all $t \in \R$).
    In both cases, define the effective dimensions $d_2 := \frac{\nnorm[1]{\blambda}^2}{\nnorm[2]{\blambda}^2}$ and $d_{\infty} := \frac{\nnorm[1]{\blambda}}{\nnorm[\infty]{\blambda}}$ and assume that $d_2 \gg v^2 n$ and $d_{\infty} \gg vn$.
    Then, Theorem~\ref{thm:b_sen_upperbound} implies that
    \begin{align*}
            \nnorm[2]{\normalize[2]{\barw} - \normalize[2]{\wmni}} &\leq C \cdot v \cdot \max\left\{ \sqrt{\frac{n}{d_2}}, \frac{n}{d_{\infty}} \right\},
    \end{align*}
    with probability at least $1 - 4e^{-cn}$, where $C, c> 0$ are appropriately chosen universal constants.
    This implies that $\nnorm[2]{\normalize[2]{\barw} - \normalize[2]{\wmni}}$ is vanishingly small for any high-dimensional ensemble $\{(n,d,\blambda)\}_{n \geq 1}$ satisfying $d_2 \gg v^2 n$ and $d_{\infty} \gg v n$.
\end{corollary}

The proof of Corollary~\ref{cor:convergencehighdims} is in Appendix~\ref{sec:corhighdimsproof} and applies the operator norm concentration inequality of~\cite[Lemma 8]{hsu2021proliferation} (which in turn uses a volume argument from~\cite{pisier1999volume}).
The corollary demonstrates the role of a sufficiently high-dimensional ensemble in ensuring that the implicit bias from a general convex loss eventually converges, in a directional sense, to the MNI.
As a special case, consider the \emph{isotropic} high-dimensional ensemble for which $\blambda = \one$ and $v = 1$.
Here, we have $d_2 = d_{\infty} = d$, and the required effective dimension conditions reduce to $d \gg n$.
\cite{hsu2021proliferation} shows that when $d_2 \gg v^2 n$ and $d_{\infty} \gg v n \log n$, the stronger phenomenon of SVP would occur\footnote{The careful reader might notice that the SVP result has an extra $\log n$ factor in the required condition on the effective dimension $d_{\infty}$, that in fact turns out to be necessary~\citep{ardeshir2021support}.
There is no contradiction with our results, because SVP describes a stronger phenomenon of \emph{exact} equivalence that holds even when $n$ and $d$ are finite, as opposed to our directional convergence result, which only gives exact asymptotic equivalence as $(n,d) \to \infty$.
}, working from the condition in Proposition~\ref{thm:b_eqv_thm} Part 2. The anisotropic Gaussian or independent sub-Gaussian model for covariates considered in Corollary~\ref{cor:convergencehighdims} does not directly cover certain high-dimensional ensembles for which conditions for SVP have been characterized; in particular, mixture models~\citep{wang2022binary,wang2021benign,cao2021risk}.
We believe that results similar to Corollary~\ref{cor:convergencehighdims} can also be established for these cases.

\begin{figure}[ht]
\centering 
\noindent\begin{subfigure}[b]{.45\textwidth}
 \includegraphics[width=65mm]{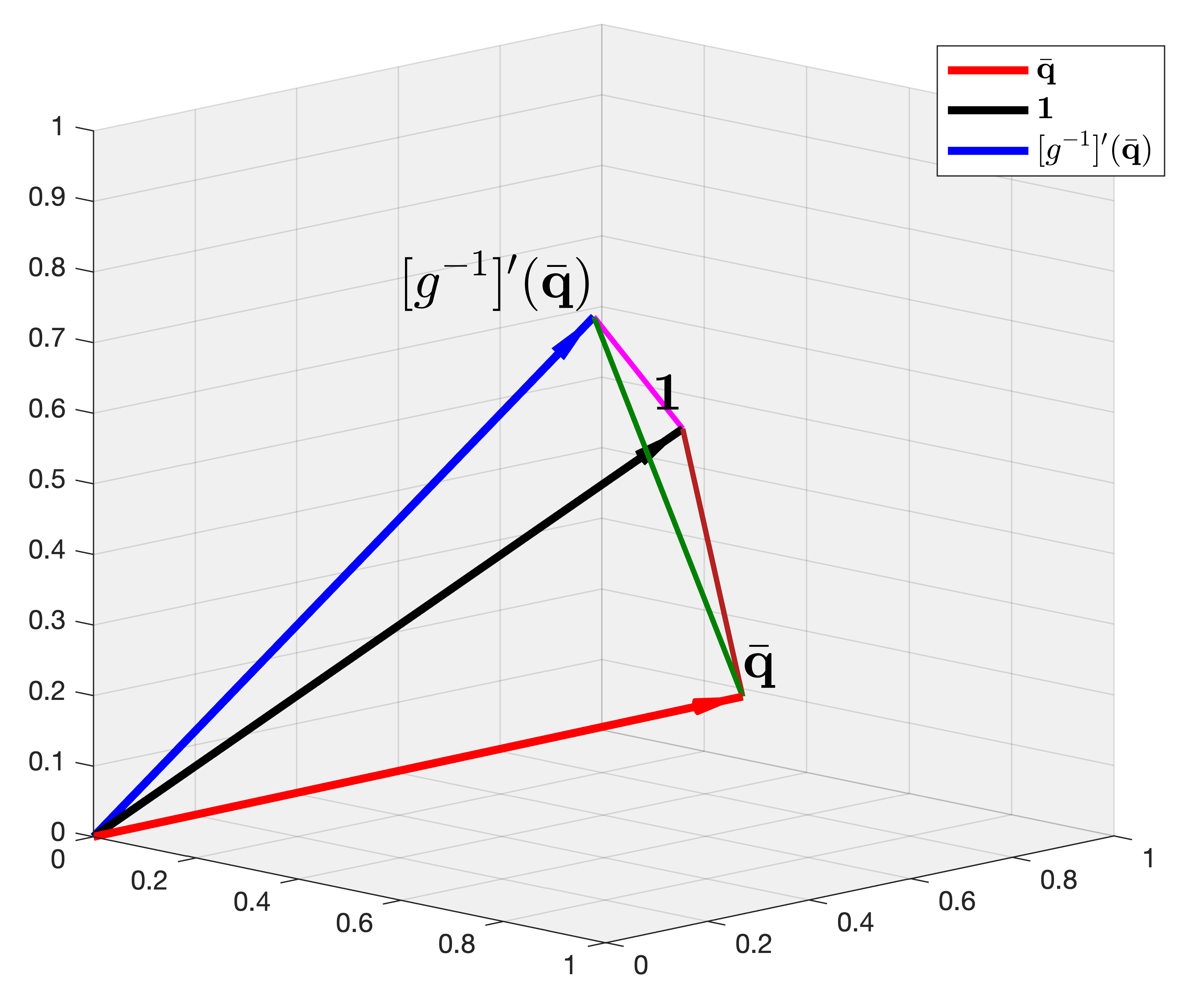}
 \caption{Proof idea of Theorem~\ref{thm:b_sen_upperbound}}\label{fig:q_hq}
\end{subfigure}
\noindent\begin{subfigure}[b]{.45\textwidth}
 \includegraphics[width=75mm]{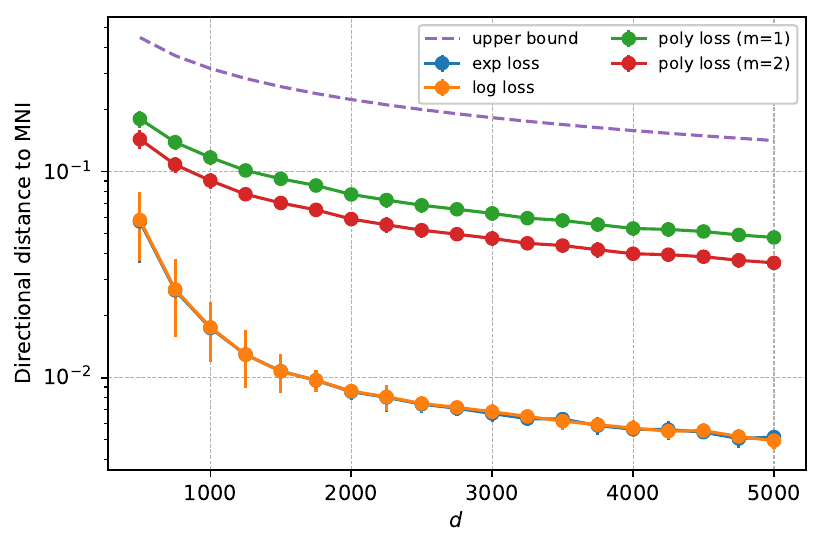}
 \caption{Simulation for binary classification}\label{fig:binary_exp}
\end{subfigure}
\caption{Panel (a) illustrates the relationship between the vectors $\barbq$, $\gpf[-1]{\barbq}$ and $\one$ for the loss function $\ellf{z}=1/\pts{1-z}$. Panel (b) is a simulation that compares the implicit bias of gradient descent to the MNI. The covariate-response pairs $\seqr{\x_i, y_i}{i=1}{n}$ are independently and identically distributed (IID) with a fixed sample size $n = 100$ and varying data dimension $d$, where $\x_i$ is isotropic Gaussian and $y_i$ is uniformly distributed in $\{\pm 1\}$. Gradient descent is run for the minimum of $10^3$ iterations or when the empirical risk falls below $10^{-12}$. The results demonstrate that the directional distance to the MNI is upper bounded by the theoretical guarantee in Theorem~\ref{thm:b_sen_upperbound}. For exponentially-tailed loss functions, exact convergence to the MNI is not observed, as it only occurs when the number of iterations of gradient descent is infinite. Each experiment is repeated over 100 independent trials.}
\end{figure}

\paragraph{Proof sketch for Theorem~\ref{thm:b_sen_upperbound}:} The full proof of Theorem~\ref{thm:b_sen_upperbound} is in Appendix~\ref{sec:proofofsensitivity}.
We divide the proof in four steps.
\textbf{In Step 1,} we begin with the auxiliary convex program~\eqref{eq:b_imp_bias_relax_eq}, and determine necessary characteristic equations for the solution $\barbq$; in particular, we show that it is \emph{necessary} for $\barbq$ to solve the system of nonlinear equations $\XX \diag{\y} \barbq = \mu \diag{\y} \gpf[-1]{\barbq}$ for some $\mu > 0$.
\textbf{In Step 2,} we use the relative closeness (in an operator-norm sense) of $\XX$ to a multiple of $\I$ to show that the nonlinear equation above implies that the vectors $\barbq$ and $\gpf[-1]{\barbq}$ are close in a directional sense in Eq.~\eqref{b_sen_inter_res3}.

Next, \textbf{Step 3} proves a simple but non-trivial observation which, as pictured in Figure~\ref{fig:q_hq}, states that the vector $\one$ is \emph{in between} the vectors $\barbq$ and $\gpf[-1]{\barbq}$ in Eq.~\eqref{eq:dualconvergenceqy} (implying that its angle with either of the vectors is smaller than the angle between $\barbq$ and $\gpf[-1]{\barbq}$).
The proof of this observation critically uses the convexity of $g(\cdot)$ which turns out to lead to an application of Chebyshev's sum inequality~\citep{hardy1952inequalities} to complete the desired argument.
Steps 1, 2 and 3 together give a rate on the directional convergence of the \emph{dual optimal solution} $\barbq$ to $\one$ in Eq.~\eqref{eq:dualconvergencefinal}.

The final \textbf{Step 4} uses the primal-dual relationship in Eq.~\eqref{eq:dual-primal-t} to show that the primal convergence rate is identical to the dual convergence rate up to universal constant factors and is proved through a series of algebraic manipulations which repeatedly utilize the operator-norm concentration of $\XX$ around $\alpha \I$.


\paragraph{Loss functions satisfying Assumption~\ref{asm:b_loss_assump}:} \label{sec:b_losses} We conclude this section with a brief discussion of popular loss functions that satisfy Assumption~\ref{asm:b_loss_assump}, and to which Proposition~\ref{thm:b_eqv_thm} and Theorem~\ref{thm:b_sen_upperbound} are therefore applicable.
These loss functions are also discussed in~\cite[Sec. 5]{ji2021characterizing}.
\begin{proposition} \label{prop:b_loss_func}
    Assumption~\ref{asm:b_loss_assump} is satisfied by the following losses with the corresponding values of the function $\gf{\cdot}$ provided:
    \begin{align*}
        \text{Exponential loss: } \ells{\expt} &:= \expf{z}\text{, } \gsf[\expt]{d} = d \\
        \text{Logistic loss: } \ells{\logt}&:=\lnf{1+\expf{z}}\text{, } \gsf[\logt]{d} = d \\
        \text{Polynomial loss (degree $m > 0$): }  \ellsf[\polyt]{z} &:=\cond{\frac{1}{\pts{1-z}^m}}{z\leq0}{\frac{1}{\pts{1+z}^m} + 2mz}{z>0}{, \;} \gsf[\polyt]{d} = d^{\frac{m+1}{m}}.
    \end{align*}
\end{proposition}

The proof of Proposition~\ref{prop:b_loss_func} is provided in Appendix~\ref{sec:b_loss_proof}.
A plot of the function $g(\cdot)$ that underlies each loss is given in Figure~\ref{fig:g_func}.
Note that $g(\cdot)$ that deviate more from $g(d) = d$ are heavier-tailed, that the furthest pictured such function corresponds to the purple line $g(d) = d^3$ for the polynomial loss with degree $m = 0.5$.





\section{Approximate equivalences for multiclass classification}\label{sec:m_approx_eq}
We now turn to the multiclass setting and consider a labeled dataset $\seqr{\x_i, y_i}{i=1}{n}$, where $\x_i \in \R^d$ and $y_i \in [K]$. We assume there is at least one example in each class. For each class $k$, we assign a weight vector $\wk\in\R^d$.  We denote as shorthand $\X=\vct{\x_1}{\x_n} \in \R^{n \times d}$, and an $n$-dimensional encoding of the multiclass labels $\bc_k(\alpha,\beta)=\vct{c_{k,1}(\alpha,\beta)}{c_{k,n}(\alpha,\beta)}\in\R^n$, where $\cki(\alpha,\beta)=\cond{\alpha}{k=y_i}{-\beta}{k\neq y_i}, \text{ for all } \alpha,\beta>0$ for all $i\in[n]$ and $k\in[K]$. 
(We frequently omit the arguments $(\alpha,\beta)$ and simply write $\bc_k$ when the values of $\alpha$ and $\beta$ are clear from context.)

We assume w.l.o.g. that $\nnorm[2]{\x_i}\leq\umax{k\in[K]}|\cki|$.
We concatenate the weight vector $\W\in\R^{Kd}$, data matrix $\tX\in\R^{Kn\times Kd}$ and label matrix across classes $\C\in\R^{Kn\times Kn}$ as below:
\begin{align*}
    \W=
    \left[
    \begin{matrix}
        \w_1    \\
        \vdots  \\
        \wk    \\
   \end{matrix}
   \right],\;
    \tX=
    \left[
    \begin{matrix}
        \X      &\cdots &\zero      \\
        \vdots  &\ddots &\vdots      \\
        \zero       &\cdots &\X     \\
   \end{matrix}
   \right],\;
   \C=
    \left[
    \begin{matrix}
        \diagt{\ico}     &\cdots &\zero      \\
        \vdots  &\ddots &\vdots      \\
        \zero       &\cdots &\diagt{\icK}     \\
   \end{matrix}
   \right].
\end{align*}
We focus on an unbounded, unregularized ERM problem with a linear classifier:
\begin{align}
    \umin{\W\in\R^{Kd}} \Riskf{\W}:=\frac{1}{n}\sumn \multilossf{-\seqr{\icki\ip{\wk}{\x_i}}{k=1}{K}}=\frac{1}{n}\sumn \multilossf{\seqr{\icki\ip{\wk}{\z_i}}{k=1}{K}},
\end{align}
where we denote $\z_i:=-\x_i$, and therefore $\Z:=-\X\in\R^{n\times d}$ and $\tZ=-\tX\in\R^{Kn\times Kd}$. Next, we introduce different variants of the multiclass loss function, which we denote by $\multiloss$.

\begin{assumption}[One-vs-all multiclass loss] \label{asm:m_ova_loss_assump} 
    The multiclass loss function satisfies
    \begin{align*}
        \multilossf{\seqr{\icki\ip{\wk}{\z_i}}{k=1}{K}} &= \sumk \ellf{\icki\ip{\wk}{\z_i}}
    \end{align*}
    where $\ell$ follows Assumption~\ref{asm:b_loss_assump} Parts 1, 2 and 3. Additionally, given $\xik \in \R^n$ and $\bXi=\vct{\xio^\top}{\xiK^\top}\in\R^{Kn}$, we define $\Lossf{\bXi}:=\sumn\multilossf{\xikis}$
    and $\psif{\bXi}:=\ellf[-1]{\Lossf{\bXi}}$, where $\psi$ is jointly convex and $\beta$-smooth with respect to the $\nm{\infty}$ norm.
\end{assumption}

Our framework is able to handle general losses satisfying Assumption~\ref{asm:b_loss_assump} under the popular one-vs-all framework.
Finally, we treat the popular cross-entropy loss, which is a generalization of the binary logistic loss, separately.
\begin{assumption}[Cross-entropy loss] \label{asm:m_ce_loss_assump}
    The loss function $\multiloss$ satisfies
    \begin{align*}
        \multilossf{\seqr{\icki\ip{\wk}{\z_i}}{k=1}{K}} &= -\ln\Biggl(\frac{\exp\bigl(\ip{\wyi}{\x_i}\bigl)}{\sumk \exp\bigl(\ip{\wk}{\x_i}\bigl)}\Biggl)\\
        &=\ln\Biggl(1+\sumkny\exp\bigl(\cyi\bigl(\icyi\ip{\wyi}{\z_i}\bigl)-\cki\bigl(\icki\ip{\w_{k}}{\z_i}\bigl)\bigl)\Biggl).
    \end{align*}
    Given $\xik \in \R^n$, $\bXi=\vctt{\xio^\top}{\xiK^\top}\in\R^{Kn}$, and $\ellf{z}=\lnf{1+\expf{z}}$, we define
    \begin{align*}
        \Lossf{\bXi}:&=\sumn \multilossf{\xikis} = \sumn\ln\Biggl(1+\sumkny \expf{\cyi\xiyi-\cki\xiki}\Biggl)\\
        \text{and} \;\; \psif{\bXi}:&=\ellf[-1]{\Lossf{\bXi}},
    \end{align*}
    where $\psi$ is individually convex with respect to each $\xik$, and $\beta$-smooth with respect to $\nm{\infty}$ norm.
\end{assumption}

For the loss functions that satisfy Assumption~\ref{asm:m_ova_loss_assump}, we use
the “equal assignment” encoding of the labels, $\alpha=\beta=1$; for cross-entropy loss under Assumption~\ref{asm:m_ce_loss_assump}, we use the ``simplex representation" encoding of the labels~\citep{lee2004multicategory,wang2021benign} with $\alpha=\frac{K-1}{K}$ and $\beta=\frac{1}{K}$.
In Appendix~\ref{app_m_conv_beta} we show that the properties of convexity and $\beta$-smoothness of $\psi$ carry over to the multiclass case; interestingly, we can only prove \emph{individual convexity} for cross-entropy loss under Assumption~\ref{asm:m_ce_loss_assump}.

\paragraph{Multiclass minimum-norm interpolation:} Analogous to the case of binary labels, we define the minimum-norm interpolator (MNI) of multiclass labels as $\Wmni \coloneqq \XXXc$ where $\bc_k$ is a specific encoding of the multiclass labels as defined at the beginning of this section.
Specifically, gradient descent run with the square loss on labels encoded with the ``equal assignment" choice $\alpha = \beta = 1$ would result in what we call the \emph{one-vs-all MNI}, given by $\Wova \coloneqq \vctt{\wovao^\top}{\wovaK^\top} \in \R^{Kd}$ where $\wovak \coloneqq \X^\top \XXi \ck(1,1)$.
Similarly, gradient descent run with the square loss on labels encoded with the ``simplex representation" $\alpha=\frac{K-1}{K}$ and $\beta=\frac{1}{K}$ would result in what we call the \emph{simplex MNI}, given by $\Wsim \coloneqq \vctt{\wsimo^\top}{\wsimK^\top} \in \R^{Kd}$ where $\wsimk \coloneqq \X^\top (\XX)^{-1} \ck\left(\frac{K-1}{K},\frac{1}{K}\right)$.

\subsection{Main results}\label{sec:multiclassmainresults}

First, we extend the primal-dual framework from~\cite{ji2021characterizing} to the multiclass case. We again use gradient descent to solve this unregularized ERM problem with initialization $\W_0$ and update rule: $\Wto:=\Wt-\eta_t \gRiskf{\Wt}$ for $t\geq0$. We denote, in the context of mirror-descent analysis, the ``primal'' $\Pt:=\C\tZ\Wt=\vctt{\pto^\top}{\ptK^\top}\in\R^{Kn}$, and its corresponding ``dual'' $\Qt:=\gpsif{\Pt}=\vctt{\qto^\top}{\qtK^\top}\in\R^{Kn}$, where $\ptk=\diagt{\ick}\Z\wtk \in \R^n$ and $\qtk=\gpsif[\ptk]{\Pt} \in \R^n$ for all $k\in[K]$ and $t\geq0$. 
%
This concatenated representation together with Assumption~\ref{asm:m_ova_loss_assump}, (or~\ref{asm:m_ce_loss_assump}) and Assumption~\ref{asm:b_lin_sep} ensure that the setup is identical to that of~\cite{ji2021characterizing}.
Therefore, we can directly apply their primal-dual result, which we restate below in our notation specific to the multiclass setting.

\begin{lemma} \label{lem:m_imp_bias_res}
    Under Assumption~\ref{asm:m_ova_loss_assump}, (or~\ref{asm:m_ce_loss_assump}) and \ref{asm:b_lin_sep}, when all $t$ with $\psi\bigl(\C\tZ\Wt\bigl)\leq0$, and iteration of gradient descent goes to infinity with $\heta_t=\eta_t\ell'\bigl(\psi\bigl(\C\tZ\Wt\bigl)\bigl)/n\leq1/\beta$ is nonincreasing and $\sumt\hetat=\infty$, we have the implicit bias $\barW:=\ulim{\tinf{t}}\normalize[2]{\Wt}=\normalize[2]{\tX^\top\C\barQ}$, where
    \begin{align} \label{eq:m_barQ_def}
        \barQ\in\uargmin{\psicf{\Q}\leq0}\Ff{\Q}, \text{ and } \Ff{\Q}:=\frac{1}{2}\nnorm[2]{\tX^\top\C\Q}^2
    \end{align}
\end{lemma}

We provide the details of this proof, which is mostly an extension of~\cite{ji2021characterizing}, in Appendix~\ref{app_m_imp_bias}.
One subtlety is that we were only able to establish individual convexity in $\psi$ for cross-entropy loss in Assumption~\ref{asm:m_ce_loss_assump}.
Lemma~\ref{lem:m_psi_ind_conv} shows that this is sufficient to recover Lemma~\ref{lem:m_imp_bias_res}, and joint convexity is only required to prove the tightness of the convergence rates in~\cite{ji2021characterizing}.

We now present the main results of this section. We first show that for any multiclass loss satisfying Assumption~\ref{asm:m_ova_loss_assump}, the implicit bias solution $\barW$ is approximately close to the one-vs-all MNI $\Wova$. This result is analogous to Theorem~\ref{thm:b_sen_upperbound} which we proved for the binary case.
\begin{theorem} \label{thm:m_sen_upperbound}
    Under Assumptions~\ref{asm:m_ova_loss_assump} and \ref{asm:b_lin_sep},
    consider any value of $\alpha > 0$ satisfying\\ $\frac{\nnorm[2]{\XX-\alpha\I}}{\alpha} \leq \frac{1}{3}$.
    Then, for every class $k \in [K]$, the implicit bias $\barw_k$ converges in direction to $\wovak$ at the rate:
    \begin{align}\label{eq:m_sen_upperbound}
    \nnorm[2]{\normalize[2]{\barw_k} - \normalize[2]{\wovak}} &\leq \frac{C\nnorm[2]{\XX \ck - \alpha \ck}}{\alpha\nnorm[2]{\ck}}, \qquad
    \end{align}
    where $C$ is a universal constant that does not depend on $\alpha$, $\X$ or $\ck$.
\end{theorem}
The proof of Theorem~\ref{thm:m_sen_upperbound} is provided in Appendix~\ref{sec:m_sen_upperbound_proof} and is a simple extension of the proof of Theorem~\ref{thm:b_sen_upperbound}.
We now state a corollary (analogous to Corollary~\ref{cor:convergencehighdims}) showing that the canonical high-dimensional ensembles that admit directional convergence in probability of the implicit bias to the MNI on binary labels also do so for the one-vs-all MNI on one-hot-encoded labels.
\begin{corollary}\label{cor:m_convergencehighdims}
    Assume independent and identically distributed data $\{\x_i,y_i\}_{i=1}^n$ such that each covariate satisfies one of the following: a) $\x_i \sim \mathcal{N}(\mathbf{0}, \boldsymbol{\Sigma})$, and we denote the spectrum of $\boldsymbol{\Sigma}$ by $\blambda$; \emph{or} b) $\x_i = \diag{\blambda}^{1/2} \z_i$, where $\z_i$ has independent entries such that each $z_{ij}$ is mean-zero, unit-variance, and sub-Gaussian with parameter $v > 0$ (i.e.~$\E[z_{ij}] = 0,\E[z_{ij}^2] = 1$, and $\E[e^{tz_{ij}}] \leq e^{vt^2/2}$ for all $t \in \R$).
    In both cases, define the effective dimensions $d_2 := \frac{\nnorm[1]{\blambda}^2}{\nnorm[2]{\blambda}^2}$ and $d_{\infty} := \frac{\nnorm[1]{\blambda}}{\nnorm[\infty]{\blambda}}$ and assume that $d_2 \gg v^2 n$ and $d_{\infty} \gg vn$.
    Then, Theorem~\ref{thm:m_sen_upperbound} implies that for each class $k \in [K]$, we have
    \begin{align*}
            \nnorm[2]{\normalize[2]{\barw_k} - \normalize[2]{\wovak}} &\leq C \cdot v \cdot \max\left\{ \sqrt{\frac{n}{d_2}}, \frac{n}{d_{\infty}} \right\},
    \end{align*}
    with probability at least $1 - 4e^{-cn}$, where $C, c> 0$ are appropriately chosen universal constants.
    This implies that $\nnorm[2]{\normalize[2]{\barw_k} - \normalize[2]{\wovak}}$ is vanishingly small for any high-dimensional ensemble $\{(n,d,\blambda)\}_{n \geq 1}$ satisfying $d_2 \gg v^2 n$ and $d_{\infty} \gg v n$.
\end{corollary}
The proof of Corollary~\ref{cor:m_convergencehighdims} is identical to the proof of Corollary~\ref{cor:convergencehighdims}, only with $\y$ replaced by $\ck$; therefore, we omit the details.

The next theorem shows an \emph{exact} equivalence to the simplex MNI for cross-entropy loss under Assumption~\ref{asm:m_ce_loss_assump}. This result is the multiclass analog of Proposition~\ref{thm:b_eqv_thm} Part 2.

%

\begin{theorem} \label{thm:m_exp_ce_imp_bias} Under Assumption~\ref{asm:m_ce_loss_assump}, the implicit bias is parallel to the simplex MNI $\Wsim$ iff $\XX \succ \zero$ and  $\bbeta_k:=\XXi\ck$ satisfies $\cki\beta_{k,i}>0$ for all $i\in[n]$ and $k\in[K]$. 
\end{theorem}

The proof of Theorem~\ref{thm:m_exp_ce_imp_bias} is in provided in Appendix~\ref{sec:m_ce_imp_proof}.
Note that Theorem~\ref{thm:m_exp_ce_imp_bias} recovers the exact equivalence condition of~\cite{wang2021benign} without using the intermediate multiclass SVM formulation of the implicit bias primal.
Interestingly, the convex programs on $\barbq_k$ for all $k\in[K]$ that are formulated in the proof of Theorem~\ref{thm:m_exp_ce_imp_bias} already contains the novel equality constraints that~\cite{wang2021benign} were only able to obtain after applying a non-trivial transformation to the multiclass SVM dual variables.
This suggests that the mirror-descent dual is the more natural dual to analyze in the multiclass case.

\begin{figure}[ht]
\centering 
\noindent\begin{subfigure}[b]{.45\textwidth}
 \includegraphics[width=68mm]{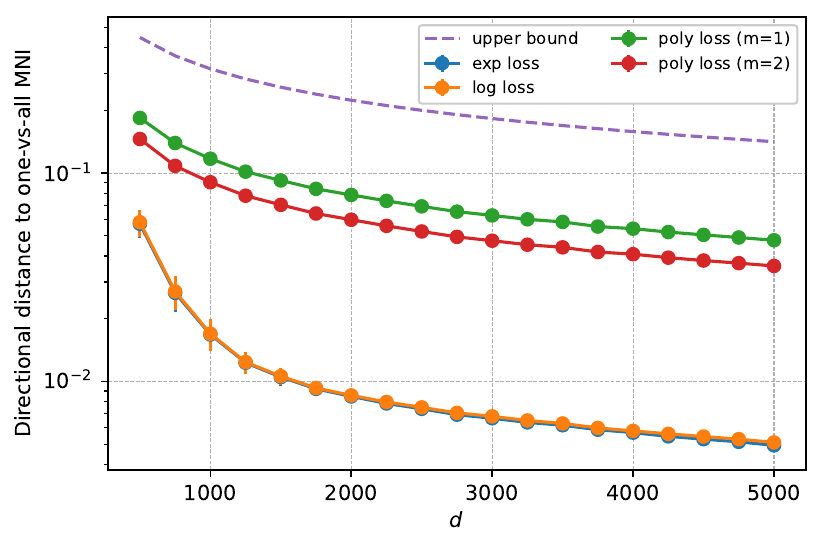}
 \caption{Simulation for multiclass classification using one-vs-all losses (Assumption~\ref{asm:m_ova_loss_assump}).}\label{fig:ova}
\end{subfigure}
\hfill
\noindent\begin{subfigure}[b]{.45\textwidth}
 \includegraphics[width=69mm]{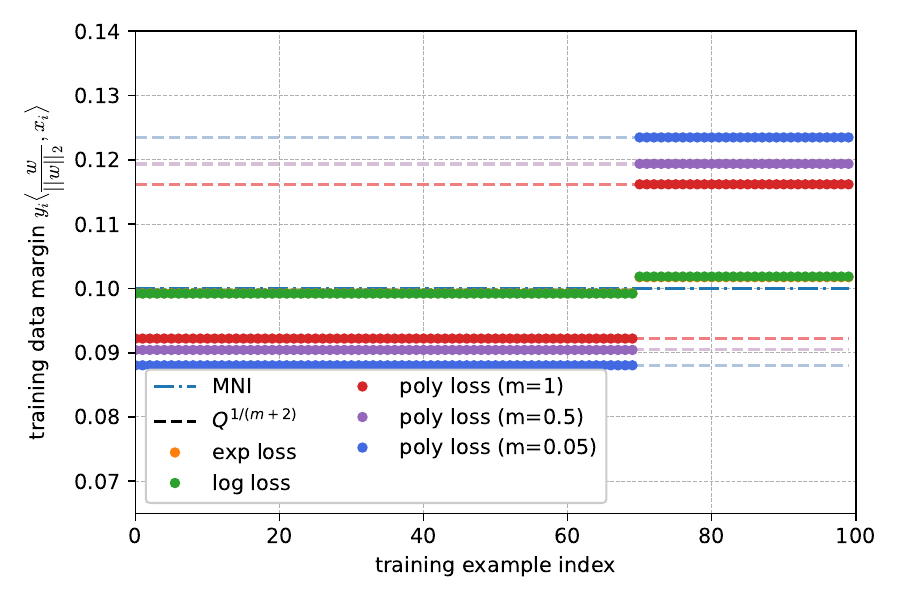}
 \caption{Simulation for importance weighting on different loss functions (Corollary~\ref{cor:ood_interpolation}).
 }\label{fig:margin}
\end{subfigure}
\caption{Panel (a) compares the implicit bias of gradient descent to the one-vs-all MNI. The results demonstrate that the directional distance to the MNI is upper bounded by the theoretical guarantee in Theorem~\ref{thm:m_sen_upperbound}, analogous to the binary case. The simulation setup is the same as Figure~\ref{fig:binary_exp} with $K=5$ classes, and labels drawn uniformly at random in $[K]$. Panel (b) visualizes the \emph{normalized} training data margins induced by importance weighting on different loss functions in Corollary~\ref{cor:ood_interpolation}. We consider the idealized assumption $\XX = \I$ with $n=100$ and $d=5000$. The first $70$ examples are \emph{majority examples} and labeled as $y_i=+1$, and the rest of the $30$ examples are \emph{minority examples} labeled as $y_i=-1$. Note that we apply the importance weighting factor $Q=2.0$ only to the minority examples. We run gradient descent on different loss functions for a minimum of $10^4$ iterations, or when the empirical risk falls below $10^{-12}$. As predicted by Corollary~\ref{cor:ood_interpolation}, the margins of exponentially-tailed losses are not impacted by importance weighting and are almost intact to those of the MNI, but polynomially-tailed losses interpolate adjusted labels to different extents depending on the value of $m$.  In Appendix~\ref{app:simulation}, we provide corresponding simulations on random data.}

\end{figure}

\section{A converse result}\label{sec:converse}

We now show that the condition for exact equivalence in Proposition~\ref{thm:b_eqv_thm} is necessary.
%
For conciseness, we consider binary labels, but these proofs can easily be extended to the multiclass case.
\begin{proposition}\label{prop:converseexact}
Consider any loss function that satisfies Assumption~\ref{asm:b_loss_assump} with a strictly convex function $g(d) \neq d$.
Define $\hf{d}:=\gpf[-1]{d}$ and $f(d) := \frac{\hf{d}}{d}$.
Then, the following statements are true about the optimal solution $\barbq$ to the dual convex program~\eqref{eq:b_barq_def}:
\begin{enumerate}
    \item If $\y$ is not an exact eigenvector of $\XX$, then at least two of the entries in $\barbq$ need to be distinct, i.e. $\barbq$ cannot be parallel to $\one$; therefore, $\barw$ is not parallel to $\wmni$.
    \item If $\XX = \D = \diag{\dbold}$, the primal solution $\barw$ interpolates the adjusted labels $\tilde{y}_i = y_i d_i \cdot f^{-1} \left(\frac{d_i}{\mu}\right)$, where $\mu > 0$ is any solution to the equation $\sum_{i=1}^n g^{-1}\Bigl(f^{-1}\left(\frac{d_i}{\mu}\right)\Bigl) = 1$.
\end{enumerate}
\end{proposition}

Proposition~\ref{prop:converseexact} is proved in Appendix~\ref{sec:converseexactproof} and also utilizes the relaxed convex program of Lemma~\ref{lem:b_imp_bias_relax}.
The proposition shows that the condition for exact equivalence in Eq.~\eqref{eq:svmequivalence} only applies to the implicit bias of exponentially-tailed losses, which satisfy Assumption~\ref{asm:b_loss_assump} with the identity mapping $g(d) = d$.
Moreover, Part 1 of Proposition~\ref{thm:b_eqv_thm} is a sufficient \emph{and} necessary condition for exact equivalence for any non-exponential loss with a non-identity mapping $g(d) \neq d$.
Part 2 of Proposition~\ref{prop:converseexact} provides explicit counterexamples in the form of Gram matrices $\XX = \D$ that can easily be verified to satisfy the SVM equivalence condition  $y_i (\XX)^{-1} \y \succ 0$, yet, induce a very different solution from the MNI that interpolates labels adjusted differently \emph{per training example}.
To drive home this point, we use Proposition~\ref{prop:converseexact} to characterize the impact of the \textbf{importance weighting procedure} with polynomial losses.
This procedure, parameterized by a subset of underrepresented examples $S \subset [n]$ and weight $Q > 1$ and applied with a loss function $\ell(\cdot)$, minimizes the weighted risk $\Riskf{\w;(Q,S)} := \frac{1}{n} \sumn Q^{\Ind[i \in S]} \cdot \ellf{-y_i\ip{\w}{\x_i}}$. 
Recently,~\cite{wang2021importance} proposed applying this procedure with polynomial losses to address OOD generalization.

\begin{corollary}\label{cor:ood_interpolation}
Consider the idealized data matrix $\XX = \alpha \I$ for some $\alpha > 0$, as in Corollary~\ref{cor:identityexact}.
Then, importance weighting with a polynomial loss of degree $m$ leads to implicit bias $\barw$ that interpolates per-example-adjusted labels $\tilde{y_i} \propto Q^{\frac{1}{m+2} \cdot \Ind[i \in S]} y_i$.
We call the implicit bias $\barw$ the cost-sensitive MNI.
\end{corollary}

Corollary~\ref{cor:ood_interpolation} is proved in Appendix~\ref{sec:oodinterpolationproof} and implies that importance weighting with polynomial losses will interpolate labels that are larger in magnitude on minority points.
As shown in~\citep{kini2021label,behnia2022avoid}, this type of \emph{cost-sensitive interpolation} is provably beneficial for OOD generalization.
Since $Q > 1$, heavier-tailed polynomial losses (corresponding to smaller values of $m$) lead to a stronger importance-weighting effect. In Figure~\ref{fig:margin}, we illustrate how different loss functions influence the training data margins (and also the interpolated adjusted labels) with an identical choice of importance weighting $Q$. This visualization clearly demonstrates that heavier-tailed losses (e.g. smaller values of $m$ in the polynomially-tailed loss) increase the margin on minority examples. 
Interestingly, we also observe a corresponding slight \emph{decrease} in the margin on majority examples.
This is because we normalized the training data margins (i.e. use the normalized weights of the linear model $\normalize[2]{\w}$) in order to provide a fair comparison of the directional differences between solutions.
Appendix~\ref{app:simulation} shows that similar patterns manifest on randomly generated data, for which Corollary~\ref{cor:ood_interpolation} does not apply.
One can compare this interpolation to that induced by the \textbf{vector-scaling} (VS-loss)~\citep{ye2020identifying,kini2021label}, defined as a per-example loss function $\ellsf[\vst]{z_i;(Q,S)}:= \ln\Bigl(1+\expf{Q^{\Ind[i \in S]} z_i}\Bigl)$.
\cite{behnia2022avoid} shows\footnote{This result is also recoverable in our framework, although we omit the details for brevity.} that in our high-dimensional regime, this will lead to cost-sensitive interpolation of the adjusted labels $\tilde{y}_i \propto Q^{\Ind[i \in S]} y_i$, which is in fact a stronger interpolation effect. Finally, we present converse results on multiclass data that are analogous to Proposition~\ref{prop:converseexact} and Corollary~\ref{cor:ood_interpolation} respectively.
\begin{proposition}\label{prop:m_converseexact}
Consider any loss function that satisfies Assumption~\ref{asm:m_ova_loss_assump} with a strictly convex function $g(d) \neq d$.
Define $\hf{d}:=\gpf[-1]{d}$ and $f(d) := \frac{\hf{d}}{d}$.
Then, the following statements are true about the optimal solution $\barbq_k$ for $k\in[K]$ to the dual convex program~\eqref{eq:m_barQ_def}:
\begin{enumerate}
    \item If $\ck$ is not an exact eigenvector of $\XX$, then at least two of the entries in $\barbq_k$ need to be distinct, i.e. $\barbq_k$ cannot be parallel to $\one$; therefore, $\barw_k$ is not parallel to $\wovak$.
    \item If $\XX = \D = \diag{\dbold}$, the primal solution $\barw_k$ interpolates the adjusted labels $\tilde{\bc}_{k,i} = \cki d_i \cdot f^{-1} \left(\frac{d_i}{\mu}\right)$ for each $k\in[K]$, where $\mu > 0$ is any solution to the equation $\sum_{i=1}^n\sumk g^{-1}\Bigl(f^{-1}\left(\frac{d_i}{\mu}\right)\Bigl) = 1$.
\end{enumerate}
\end{proposition}
\begin{corollary}\label{cor:m_ood_interpolation}
Consider the idealized data matrix $\XX = \alpha \I$ for some $\alpha > 0$.
Then, importance weighting with a polynomial loss of degree $m$ leads to implicit bias $\barw_k$ that interpolates per-example-adjusted labels $\tilde{\bc}_{k,i} \propto Q^{\frac{1}{m+2} \cdot \Ind[\{k,i\} \in S]} \cki$ for each $k\in[K]$.

\end{corollary}
The proof of Proposition~\ref{prop:m_converseexact} is identical to the proof of Proposition~\ref{prop:converseexact}, and the proof of Corollary~\ref{cor:m_ood_interpolation} is identical to the proof of Corollary~\ref{cor:ood_interpolation}, since they analyze the same characteristic equation -- Eq.\eqref{b_con_sen_charc_1} in the binary case and Eq.\eqref{m_sen_charc_1} in multiclass case with $\y$ replaced by $\ck$ for each class $k \in [K]$. Therefore, we omit the details.



\subsection{Lower bound on directional convergence between \texorpdfstring{$\barbq$}{q} and  \texorpdfstring{$\one$}{1}}\label{sec:converseapprox}

The preceding Proposition~\ref{prop:converseexact} addressed the question of tightness of our \emph{exact} equivalence theorem (Theorem~\ref{thm:b_eqv_thm}).
This section addresses whether we can obtain a lower bound on the approximation error that matches Theorem~\ref{thm:b_sen_upperbound}.
We show that we can obtain a lower bound on the approximation error for loss functions with homogeneous function $\hf{z}$ such that, in some sense, ``matches" our upper bound.

\begin{proposition}\label{prop:b_sen_lowerbound}
Consider any loss function $\ell(z)$ satisfying Assumption~\ref{asm:b_loss_assump}, and additionally assume that its corresponding function $\hf{q}:=\gpf[-1]{q}$ is a homogeneous function, i.e. $\hf{ab}=a^{\gamma}\hf{b}$ for $a,b\geq0$ and $\gamma\in\R$. 
Further, assume that $\nnorm[2]{\normalize[2]{\barbq}-\nov} \leq \frac{\delta}{\sqrt{n}}$ for some $\delta \in (0,1)$.
Then, the dual implicit bias is lower bounded (in its directional distance from the dual-MNI) as:
\begin{align}\label{eq:b_sen_lowerbound}
        \nnorm[2]{\normalize[2]{\barbq}-\nov} &\geq \frac{1}{2\sqrt{n}}\min_{\alpha > 0} \min\left\{\frac{\nnorm[2]{\XX\y-\alpha\y}}{k\alpha}, \frac{\nnorm[2]{\XX\y-\alpha\y}}{\nnorm[2]{\XX}} \right\},
\end{align}
where $k=\max\pts{{\frac{\bhf{1-\delta}-1}{\delta}, \frac{1-\bhf{1+\delta}}{\delta}}}$, and $\bhf{z}=a\hf{z}$ for some $a>0$ such that $\bhf{1}=1$.
\end{proposition}

Proposition~\ref{prop:b_sen_lowerbound} is proved in Appendix~\ref{sec:converseapproxproof}. We first remark on the sense in which Eq.~\eqref{eq:b_sen_lowerbound} is tight with respect to the upper bound in Theorem~\ref{thm:b_sen_upperbound}.
If the best value of $\alpha$ is one for which $\nnorm[2]{\XX} \leq \frac{4\alpha}{3}$ (which is the assumption made in Theorem~\ref{thm:b_sen_upperbound}), then the lower bound becomes $\min\left\{\frac{\nnorm[2]{\XX\y-\alpha\y}}{2k\alpha \nnorm[2]{\y}}, \frac{3\nnorm[2]{\XX\y-\alpha\y}}{8\alpha \nnorm[2]{\y}}\right\}$, which matches the upper bound (Eq.~\eqref{eq:b_sen_upperbound}) up to the constant factor $k$.
Next, we briefly comment on the extra assumptions appearing in the proposition, starting with the assumption of homogeneity on $\hf{q}$.
In particular, the special case of polynomial loss has $\hf{q}=\frac{m}{m+1}q^{\frac{-1}{m+1}}$ which is a homogeneous function; therefore, Proposition~\ref{prop:b_sen_lowerbound} applies. 
We also comment on the requirement that  $\nnorm[2]{\normalize[2]{\barbq}-\nov} \leq \frac{\delta}{\sqrt{n}}$ for some $\delta \in (0,1)$.
Note that Corollary~\ref{cor:convergencehighdims} directly implies that this condition would be satisfied w.h.p. if $d_2\gg v^2n^2$ and $d_\infty \gg vn^{\frac{3}{2}}$; i.e.~under a very high-dimensional regime.
We believe that the extra $\frac{1}{\sqrt{n}}$ factor in the upper bound above is not required, and could be removed if one were able to show that all entries of the directional error vector $\normalize[2]{\barbq}-\nov$ were within constant factors of one another.
Showing this (and, relatedly, providing tight upper and lower bounds on the $\ell_{\infty}$-directional error) is an important direction for future work.
Finally, we present a corollary (analogous to Proposition~\ref{prop:b_sen_lowerbound}) that lower bounds the approximation error for multiclass losses under Assumption~\ref{asm:m_ova_loss_assump}.
\begin{corollary}\label{cor:m_ova_sen_lowerbound}
Consider any multiclass loss function satisfying Assumption~\ref{asm:m_ova_loss_assump}, and additionally assume that its corresponding function $\hf{q}:=\gpf[-1]{q}$ is a homogeneous function, i.e. $\hf{ab}=a^{\gamma}\hf{b}$ for $a,b\geq0$ and $\gamma\in\R$. 
Further, assume that $\nnorm[2]{\normalize[2]{\barbq_k}-\nov} \leq \frac{\delta}{\sqrt{n}}$ for some $\delta \in (0,1)$ for all $k\in[K]$.
Then, the dual implicit bias for each class $k$ is lower bounded (in its directional distance from the dual-MNI) as:
\begin{align}\label{eq:m_ova_sen_lowerbound}
        \nnorm[2]{\normalize[2]{\barbq_k}-\nov} &\geq \frac{1}{2\sqrt{n}}\min_{\alpha > 0} \min\left\{\frac{\nnorm[2]{\XX\ck-\alpha\ck}}{t\alpha}, \frac{\nnorm[2]{\XX\ck-\alpha\ck}}{\nnorm[2]{\XX}} \right\},
\end{align}
where $t=\max\pts{{\frac{\bhf{1-\delta}-1}{\delta}, \frac{1-\bhf{1+\delta}}{\delta}}}$, and $\bhf{z}=a\hf{z}$ for some $a>0$ such that $\bhf{1}=1$.
\end{corollary}

The proof of Corollary~\ref{cor:m_ova_sen_lowerbound} is identical to the proof of Proposition~\ref{prop:b_sen_lowerbound}, as it analyzes the same characteristic equation---Eq.\eqref{b_con_sen_charc_1} in the binary case and Eq.\eqref{m_sen_charc_1} in multiclass case with $\y$ replaced by $\ck$ for each class $k \in [K]$. Therefore, we omit the details.

\section{Discussion}
Our results show that once we move away from the exponentially-tailed family of losses, general losses exhibit a variety of influence on the eventual solution, with similarities for ``in-distribution"-oriented loss functions but differences for ``out-of-distribution"-oriented loss functions.
We believe that these results show the potential of the primal-dual framework to study closed-form properties of the implicit bias.
It would be interesting to provide similar closed-form characterizations for the implicit bias of other optimization algorithms and/or for nonlinear models.
Specific to linear models and gradient descent, there are still many open questions.
Based on converse results in~\cite{hsu2021proliferation,ardeshir2021support} for exponential losses, the effective overparameterization conditions in Corollary~\ref{cor:convergencehighdims} appear necessary for asymptotic directional convergence of the implicit bias to MNI. However, whether Theorem~\ref{thm:b_sen_upperbound} provides the optimal rate of convergence (beyond the partial converse result in Proposition~\ref{prop:b_sen_lowerbound}) is unclear.
Also, of interest is whether it is possible to obtain results similar to Propositions~\ref{thm:b_eqv_thm} and Theorem~\ref{thm:b_sen_upperbound} under even fewer assumptions on losses, such as in~\citep{ji2020gradient,bartlett2006convexity}.
Finally, we are interested in using these closed-form characterizations to obtain tight non-asymptotic bounds on the test risk.


\section*{Acknowledgements}

We gratefully acknowledge the support of the NSF (through CAREER award CCF-2239151 and award IIS-2212182), an Adobe Data Science Research Award, an Amazon Research Award and a Google Research Colabs award.

\newpage
\printbibliography

@article{bartlett2006convexity,
  title={Convexity, classification, and risk bounds},
  author={Bartlett, Peter L and Jordan, Michael I and McAuliffe, Jon D},
  journal={Journal of the American Statistical Association},
  volume={101},
  number={473},
  pages={138--156},
  year={2006},
  publisher={Taylor \& Francis}
}

@article{zhang2004statistical,
  title={Statistical behavior and consistency of classification methods based on convex risk minimization},
  author={Zhang, Tong},
  journal={The Annals of Statistics},
  volume={32},
  number={1},
  pages={56--85},
  year={2004},
  publisher={Institute of Mathematical Statistics}
}

@article{lugosi2004bayes,
  title={On the Bayes-risk consistency of regularized boosting methods},
  author={Lugosi, G{\'a}bor and Vayatis, Nicolas},
  journal={The Annals of statistics},
  volume={32},
  number={1},
  pages={30--55},
  year={2004},
  publisher={Institute of Mathematical Statistics}
}

@article{steinwart2005consistency,
  title={Consistency of support vector machines and other regularized kernel classifiers},
  author={Steinwart, Ingo},
  journal={IEEE transactions on information theory},
  volume={51},
  number={1},
  pages={128--142},
  year={2005},
  publisher={IEEE}
}

@article{bartlett1998boosting,
  title={Boosting the margin: A new explanation for the effectiveness of voting methods},
  author={Bartlett, Peter and Freund, Yoav and Lee, Wee Sun and Schapire, Robert E},
  journal={The Annals of Statistics},
  volume={26},
  number={5},
  pages={1651--1686},
  year={1998},
  publisher={Institute of Mathematical Statistics}
}

@article{bartlett2002rademacher,
  title={Rademacher and Gaussian complexities: Risk bounds and structural results},
  author={Bartlett, Peter L and Mendelson, Shahar},
  journal={The Journal of Machine Learning Research},
  volume={3},
  number={Nov},
  pages={463--482},
  year={2002}
}

@article{zhang2021understanding,
  title={Understanding deep learning (still) requires rethinking generalization},
  author={Zhang, Chiyuan and Bengio, Samy and Hardt, Moritz and Recht, Benjamin and Vinyals, Oriol},
  journal={Communications of the ACM},
  volume={64},
  number={3},
  pages={107--115},
  year={2021},
  publisher={ACM New York, NY, USA}
}

@article{neyshabur2014search,
  title={In search of the real inductive bias: On the role of implicit regularization in deep learning},
  author={Neyshabur, Behnam and Tomioka, Ryota and Srebro, Nathan},
  journal={arXiv preprint arXiv:1412.6614},
  year={2014}
}

@article{hui2020evaluation,
  title={Evaluation of neural architectures trained with square loss vs cross-entropy in classification tasks},
  author={Hui, Like and Belkin, Mikhail},
  journal={arXiv preprint arXiv:2006.07322},
  year={2020}
}

@article{kline2005revisiting,
  title={Revisiting squared-error and cross-entropy functions for training neural network classifiers},
  author={Kline, Douglas M and Berardi, Victor L},
  journal={Neural Computing \& Applications},
  volume={14},
  pages={310--318},
  year={2005},
  publisher={Springer}
}

@inproceedings{golik2013cross,
  title={Cross-entropy vs. squared error training: a theoretical and experimental comparison.},
  author={Golik, Pavel and Doetsch, Patrick and Ney, Hermann},
  booktitle={Interspeech},
  volume={13},
  pages={1756--1760},
  year={2013}
}

@article{janocha2017loss,
  title={On loss functions for deep neural networks in classification},
  author={Janocha, Katarzyna and Czarnecki, Wojciech Marian},
  journal={arXiv preprint arXiv:1702.05659},
  year={2017}
}

@article{sagawa2019distributionally,
  title={Distributionally robust neural networks for group shifts: On the importance of regularization for worst-case generalization},
  author={Sagawa, Shiori and Koh, Pang Wei and Hashimoto, Tatsunori B and Liang, Percy},
  journal={arXiv preprint arXiv:1911.08731},
  year={2019}
}

@article{mansour2008domain,
  title={Domain adaptation with multiple sources},
  author={Mansour, Yishay and Mohri, Mehryar and Rostamizadeh, Afshin},
  journal={Advances in Neural Information Processing Systems},
  volume={21},
  year={2008}
}

@article{cao2019learning,
  title={Learning imbalanced datasets with label-distribution-aware margin loss},
  author={Cao, Kaidi and Wei, Colin and Gaidon, Adrien and Arechiga, Nikos and Ma, Tengyu},
  journal={Advances in Neural Information Processing Systems},
  volume={32},
  year={2019}
}

@article{menon2020long,
  title={Long-tail learning via logit adjustment},
  author={Menon, Aditya Krishna and Jayasumana, Sadeep and Rawat, Ankit Singh and Jain, Himanshu and Veit, Andreas and Kumar, Sanjiv},
  journal={arXiv preprint arXiv:2007.07314},
  year={2020}
}

@article{kini2021label,
  title={Label-imbalanced and group-sensitive classification under overparameterization},
  author={Kini, Ganesh Ramachandra and Paraskevas, Orestis and Oymak, Samet and Thrampoulidis, Christos},
  journal={Advances in Neural Information Processing Systems},
  volume={34},
  pages={18970--18983},
  year={2021}
}

@article{wang2021importance,
  title={Is importance weighting incompatible with interpolating classifiers?},
  author={Wang, Ke Alexander and Chatterji, Niladri S and Haque, Saminul and Hashimoto, Tatsunori},
  journal={arXiv preprint arXiv:2112.12986},
  year={2021}
}

@article{ye2020identifying,
  title={Identifying and compensating for feature deviation in imbalanced deep learning},
  author={Ye, Han-Jia and Chen, Hong-You and Zhan, De-Chuan and Chao, Wei-Lun},
  journal={arXiv preprint arXiv:2001.01385},
  year={2020}
}

@article{bartlett2020benign,
  title={Benign overfitting in linear regression},
  author={Bartlett, Peter L and Long, Philip M and Lugosi, G{\'a}bor and Tsigler, Alexander},
  journal={Proceedings of the National Academy of Sciences},
  volume={117},
  number={48},
  pages={30063--30070},
  year={2020},
  publisher={National Acad Sciences}
}

@article{hastie2022surprises,
  title={Surprises in high-dimensional ridgeless least squares interpolation},
  author={Hastie, Trevor and Montanari, Andrea and Rosset, Saharon and Tibshirani, Ryan J},
  journal={The Annals of Statistics},
  volume={50},
  number={2},
  pages={949--986},
  year={2022},
  publisher={Institute of Mathematical Statistics}
}

@article{muthukumar2020harmless,
  title={Harmless interpolation of noisy data in regression},
  author={Muthukumar, Vidya and Vodrahalli, Kailas and Subramanian, Vignesh and Sahai, Anant},
  journal={IEEE Journal on Selected Areas in Information Theory},
  volume={1},
  number={1},
  pages={67--83},
  year={2020},
  publisher={IEEE}
}

@article{belkin2020two,
  title={Two models of double descent for weak features},
  author={Belkin, Mikhail and Hsu, Daniel and Xu, Ji},
  journal={SIAM Journal on Mathematics of Data Science},
  volume={2},
  number={4},
  pages={1167--1180},
  year={2020},
  publisher={SIAM}
}

@article{kobak2020optimal,
  title={The optimal ridge penalty for real-world high-dimensional data can be zero or negative due to the implicit ridge regularization},
  author={Kobak, Dmitry and Lomond, Jonathan and Sanchez, Benoit},
  journal={The Journal of Machine Learning Research},
  volume={21},
  number={1},
  pages={6863--6878},
  year={2020},
  publisher={JMLRORG}
}

@article{muthukumar2021classification,
  title={Classification vs regression in overparameterized regimes: Does the loss function matter?},
  author={Muthukumar, Vidya and Narang, Adhyyan and Subramanian, Vignesh and Belkin, Mikhail and Hsu, Daniel and Sahai, Anant},
  journal={The Journal of Machine Learning Research},
  volume={22},
  number={1},
  pages={10104--10172},
  year={2021},
  publisher={JMLRORG}
}

@inproceedings{hsu2021proliferation,
  title={On the proliferation of support vectors in high dimensions},
  author={Hsu, Daniel and Muthukumar, Vidya and Xu, Ji},
  booktitle={International Conference on Artificial Intelligence and Statistics},
  pages={91--99},
  year={2021},
  organization={PMLR}
}

@article{wang2021benign,
  title={Benign overfitting in multiclass classification: All roads lead to interpolation},
  author={Wang, Ke and Muthukumar, Vidya and Thrampoulidis, Christos},
  journal={Advances in Neural Information Processing Systems},
  volume={34},
  pages={24164--24179},
  year={2021}
}

@article{wang2022binary,
  title={Binary classification of gaussian mixtures: Abundance of support vectors, benign overfitting, and regularization},
  author={Wang, Ke and Thrampoulidis, Christos},
  journal={SIAM Journal on Mathematics of Data Science},
  volume={4},
  number={1},
  pages={260--284},
  year={2022},
  publisher={SIAM}
}

@article{cao2021risk,
  title={Risk bounds for over-parameterized maximum margin classification on sub-gaussian mixtures},
  author={Cao, Yuan and Gu, Quanquan and Belkin, Mikhail},
  journal={Advances in Neural Information Processing Systems},
  volume={34},
  pages={8407--8418},
  year={2021}
}

@article{ardeshir2021support,
  title={Support vector machines and linear regression coincide with very high-dimensional features},
  author={Ardeshir, Navid and Sanford, Clayton and Hsu, Daniel J},
  journal={Advances in Neural Information Processing Systems},
  volume={34},
  pages={4907--4918},
  year={2021}
}

@book{devroye2013probabilistic,
  title={A probabilistic theory of pattern recognition},
  author={Devroye, Luc and Gy{\"o}rfi, L{\'a}szl{\'o} and Lugosi, G{\'a}bor},
  volume={31},
  year={2013},
  publisher={Springer Science \& Business Media}
}

@article{audibert2007fast,
  title={Fast learning rates for plug-in classifiers},
  author={Audibert, Jean-Yves and Tsybakov, Alexandre B},
  journal={The Annals of Statistics},
  volume={35},
  number={2},
  pages={608--633},
  year={2007}
}

@article{chatterji2021finite,
  title={Finite-sample analysis of interpolating linear classifiers in the overparameterized regime},
  author={Chatterji, Niladri S and Long, Philip M},
  journal={The Journal of Machine Learning Research},
  volume={22},
  number={1},
  pages={5721--5750},
  year={2021},
  publisher={JMLRORG}
}

@inproceedings{frei2022benign,
  title={Benign overfitting without linearity: Neural network classifiers trained by gradient descent for noisy linear data},
  author={Frei, Spencer and Chatterji, Niladri S and Bartlett, Peter},
  booktitle={Conference on Learning Theory},
  pages={2668--2703},
  year={2022},
  organization={PMLR}
}

@article{chatterji2022interplay,
  title={The interplay between implicit bias and benign overfitting in two-layer linear networks},
  author={Chatterji, Niladri S and Long, Philip M and Bartlett, Peter L},
  journal={The Journal of Machine Learning Research},
  volume={23},
  number={263},
  pages={1--48},
  year={2022}
}

@inproceedings{behnia2022avoid,
  title={On how to avoid exacerbating spurious correlations when models are overparameterized},
  author={Behnia, Tina and Wang, Ke and Thrampoulidis, Christos},
  booktitle={2022 IEEE International Symposium on Information Theory (ISIT)},
  pages={121--126},
  year={2022},
  organization={IEEE}
}

@article{subramanian2022generalization,
  title={Generalization for multiclass classification with overparameterized linear models},
  author={Subramanian, Vignesh and Arya, Rahul and Sahai, Anant},
  journal={arXiv preprint arXiv:2206.01399},
  year={2022}
}

@inproceedings{telgarsky2013margins,
  title={Margins, shrinkage, and boosting},
  author={Telgarsky, Matus},
  booktitle={International Conference on Machine Learning},
  pages={307--315},
  year={2013},
  organization={PMLR}
}

@article{soudry2018implicit,
  title={The implicit bias of gradient descent on separable data},
  author={Soudry, Daniel and Hoffer, Elad and Nacson, Mor Shpigel and Gunasekar, Suriya and Srebro, Nathan},
  journal={The Journal of Machine Learning Research},
  volume={19},
  number={1},
  pages={2822--2878},
  year={2018},
  publisher={JMLR. org}
}

@inproceedings{ji2019implicit,
  title={The implicit bias of gradient descent on nonseparable data},
  author={Ji, Ziwei and Telgarsky, Matus},
  booktitle={Conference on Learning Theory},
  pages={1772--1798},
  year={2019},
  organization={PMLR}
}

@inproceedings{ji2020gradient,
  title={Gradient descent follows the regularization path for general losses},
  author={Ji, Ziwei and Dud{\'{i}}k, Miroslav and Schapire, Robert E and Telgarsky, Matus},
  booktitle={Conference on Learning Theory},
  pages={2109--2136},
  year={2020},
  organization={PMLR}
}

@inproceedings{ji2021characterizing,
  title={Characterizing the implicit bias via a primal-dual analysis},
  author={Ji, Ziwei and Telgarsky, Matus},
  booktitle={Algorithmic Learning Theory},
  pages={772--804},
  year={2021},
  organization={PMLR}
}

@article{dudik2022convex,
  title={Convex Analysis at Infinity: An Introduction to Astral Space},
  author={Dud{\'\i}k, Miroslav and Ji, Ziwei and Schapire, Robert E and Telgarsky, Matus},
  journal={arXiv preprint arXiv:2205.03260},
  year={2022}
}

@inproceedings{gunasekar2018characterizing,
  title={Characterizing implicit bias in terms of optimization geometry},
  author={Gunasekar, Suriya and Lee, Jason and Soudry, Daniel and Srebro, Nathan},
  booktitle={International Conference on Machine Learning},
  pages={1832--1841},
  year={2018},
  organization={PMLR}
}

@article{gunasekar2018implicit,
  title={Implicit bias of gradient descent on linear convolutional networks},
  author={Gunasekar, Suriya and Lee, Jason D and Soudry, Daniel and Srebro, Nati},
  journal={Advances in Neural Information processing Systems},
  volume={31},
  year={2018}
}

@inproceedings{woodworth2020kernel,
  title={Kernel and rich regimes in overparametrized models},
  author={Woodworth, Blake and Gunasekar, Suriya and Lee, Jason D and Moroshko, Edward and Savarese, Pedro and Golan, Itay and Soudry, Daniel and Srebro, Nathan},
  booktitle={Conference on Learning Theory},
  pages={3635--3673},
  year={2020},
  organization={PMLR}
}

@inproceedings{nacson2019convergence,
  title={Convergence of gradient descent on separable data},
  author={Nacson, Mor Shpigel and Lee, Jason and Gunasekar, Suriya and Savarese, Pedro Henrique Pamplona and Srebro, Nathan and Soudry, Daniel},
  booktitle={The 22nd International Conference on Artificial Intelligence and Statistics},
  pages={3420--3428},
  year={2019},
  organization={PMLR}
}

@article{tewari2007consistency,
  title={On the Consistency of Multiclass Classification Methods.},
  author={Tewari, Ambuj and Bartlett, Peter L},
  journal={The Journal of Machine Learning Research},
  volume={8},
  number={5},
  year={2007}
}

@inproceedings{ji2021fast,
  title={Fast margin maximization via dual acceleration},
  author={Ji, Ziwei and Srebro, Nathan and Telgarsky, Matus},
  booktitle={International Conference on Machine Learning},
  pages={4860--4869},
  year={2021},
  organization={PMLR}
}

@book{hardy1952inequalities,
  title={Inequalities},
  author={Hardy, Godfrey Harold and Littlewood, John Edensor and P{\'o}lya, George},
  year={1952},
  publisher={Cambridge university press}
}

@book{rockafellar1970convex,
  title={Convex analysis},
  author={Rockafellar, R Tyrrell},
  volume={18},
  year={1970},
  publisher={Princeton university press}
}

@article{karush1939minima,
  title={Minima of functions of several variables with inequalities as side constraints},
  author={Karush, William},
  journal={M. Sc. Dissertation. Dept. of Mathematics, Univ. of Chicago},
  year={1939}
}

@book{engl1996regularization,
  title={Regularization of inverse problems},
  author={Engl, Heinz Werner and Hanke, Martin and Neubauer, Andreas},
  volume={375},
  year={1996},
  publisher={Springer Science \& Business Media}
}

@article{tsybakov2009nonparametric,
  title={Nonparametric estimators},
  author={Tsybakov, Alexandre B},
  journal={Introduction to Nonparametric Estimation},
  pages={1--76},
  year={2009},
  publisher={Springer}
}

@article{nemirovski2000topics,
  title={Topics in non-parametric statistics},
  author={Nemirovski, Arkadi},
  journal={Ecole d’Et{\'e} de Probabilit{\'e}s de Saint-Flour},
  volume={28},
  pages={85},
  year={2000}
}

@book{pisier1999volume,
  title={The volume of convex bodies and Banach space geometry},
  author={Pisier, Gilles},
  volume={94},
  year={1999},
  publisher={Cambridge University Press}
}

@article{lee2004multicategory,
  title={Multicategory support vector machines: Theory and application to the classification of microarray data and satellite radiance data},
  author={Lee, Yoonkyung and Lin, Yi and Wahba, Grace},
  journal={Journal of the American Statistical Association},
  volume={99},
  number={465},
  pages={67--81},
  year={2004},
  publisher={Taylor \& Francis}
}

@article{huang2017asymptotic,
  title={Asymptotic behavior of support vector machine for spiked population model},
  author={Huang, Hanwen},
  journal={The Journal of Machine Learning Research},
  volume={18},
  number={1},
  pages={1472--1492},
  year={2017},
  publisher={JMLR. org}
}

@article{sur2019modern,
  title={A modern maximum-likelihood theory for high-dimensional logistic regression},
  author={Sur, Pragya and Cand{\`e}s, Emmanuel J},
  journal={Proceedings of the National Academy of Sciences},
  volume={116},
  number={29},
  pages={14516--14525},
  year={2019},
  publisher={National Acad Sciences}
}

@inproceedings{mai2019large,
  title={A large scale analysis of logistic regression: Asymptotic performance and new insights},
  author={Mai, Xiaoyi and Liao, Zhenyu and Couillet, Romain},
  booktitle={ICASSP 2019-2019 IEEE International Conference on Acoustics, Speech and Signal Processing (ICASSP)},
  pages={3357--3361},
  year={2019},
  organization={IEEE}
}

@article{salehi2019impact,
  title={The impact of regularization on high-dimensional logistic regression},
  author={Salehi, Fariborz and Abbasi, Ehsan and Hassibi, Babak},
  journal={Advances in Neural Information Processing Systems},
  volume={32},
  year={2019}
}

@article{deng2022model,
  title={A model of double descent for high-dimensional binary linear classification},
  author={Deng, Zeyu and Kammoun, Abla and Thrampoulidis, Christos},
  journal={Information and Inference: A Journal of the IMA},
  volume={11},
  number={2},
  pages={435--495},
  year={2022},
  publisher={Oxford University Press}
}

@article{montanari2019generalization,
  title={The generalization error of max-margin linear classifiers: High-dimensional asymptotics in the overparametrized regime},
  author={Montanari, Andrea and Ruan, Feng and Sohn, Youngtak and Yan, Jun},
  journal={arXiv preprint arXiv:1911.01544},
  year={2019}
}

@book{shalev2007online,
  title={Online learning: Theory, algorithms, and applications},
  author={Shalev-Shwartz, Shai},
  year={2007},
  publisher={Hebrew University}
}

@article{loureiro2021learning,
  title={Learning gaussian mixtures with generalized linear models: Precise asymptotics in high-dimensions},
  author={Loureiro, Bruno and Sicuro, Gabriele and Gerbelot, C{\'e}dric and Pacco, Alessandro and Krzakala, Florent and Zdeborov{\'a}, Lenka},
  journal={Advances in Neural Information Processing Systems},
  volume={34},
  pages={10144--10157},
  year={2021}
}

@article{ravi2024implicit,
  title={The implicit bias of gradient descent on separable multiclass data},
  author={Ravi, Hrithik and Scott, Clay and Soudry, Daniel and Wang, Yutong},
  journal={Advances in Neural Information Processing Systems},
  volume={37},
  pages={81324--81359},
  year={2024}
}
\newpage
\appendix
\addcontentsline{toc}{section}{Appendix} 
\part{Appendix} 
\parttoc 

\newpage

\section{Proofs of all binary results} \label{app_b_proof}

In this section, we include all the detailed proofs for our analysis for the binary case. 

\subsection{Proof of Lemma~\ref{lem:g_func} (existence of \texorpdfstring{$g$}{g} function)}\label{sec:g_func_proof}
In this section, we prove Lemma~\ref{lem:g_func}, which establishes the existence of a function $g(\cdot)$ that is strictly increasing, convex, and a function of the original loss function $\ellf{\cdot}$. To do so, we first introduce a different co-convergent sequence $\{\bar{g}_a(\cdot)\}_{a > 0}$. We start with a lemma proving the existence of its limit, which implies the existence of the limit of the original sequence of functions of interest, $g_a\pts{z} \coloneqq \frac{\ellpellnofz{a\cdot z}}{\ellpellnofz{a}}$.

\begin{lemma}\label{lem:g_convergence}
    Under Assumption~\ref{asm:b_loss_assump}, define two sequences $g_a\pts{z} \coloneqq \frac{\ellpellnofz{a\cdot z}}{\ellpellnofz{a}}$ and $\bar{g}_a\pts{z} \coloneqq \frac{\ell^{-1}\pts{a}\cdot z}{\ell^{-1}\pts{a\cdot z}}$, for $0<a < \ellf{0}$ and $z\in(0, 1]$. Then, these two sequences converge to the same limit, i.e., $\ulim{a \to 0}\; g_a\pts{z} = \ulim{a \to 0}\; \bar{g}_a\pts{z}$ for all $z\in(0, 1]$.
\end{lemma}
\begin{proof}
    We first show that both sequences are equivalent in the limit through the following chain of equalities:
    \begin{align*}
        \ulim{a \to 0} \;\bar{g}_a\pts{z} = \ulim{a \to 0} \;\frac{\ell^{-1}\pts{a}\cdot z}{\ell^{-1}\pts{a\cdot z}} = \ulim{a \to 0} \;\frac{\ellpellnofz{a\cdot z}}{\ellpellnofz{a}} = \ulim{a \to 0}\; g_a\pts{z}.
    \end{align*}
    Above, the second equality follows from l'Hospital's rule (because $\ulim{a \to 0} \;\frac{\ell^{-1}\pts{a}\cdot z}{\ell^{-1}\pts{a\cdot z}}$ is in indeterminate form). 
    Therefore, it suffices to show the existence of the limit of $\bar{g}_a\pts{z}$ as $a \to 0$.
    To do this, we will show that $\bar{g}_a\pts{z}$ is decreasing in $a$ as well as bounded above for all $a > 0$.
    Following~\cite[Lemma 6]{ji2021characterizing}, we define the function $\sigma\pts{s} \coloneqq \ellpellnofz{s} \cdot \ell^{-1}\pts{s}$, where Parts 1 and 2 of Assumption~\ref{asm:b_loss_assump} together imply that $\ulim{s\to 0}\; \sigma\pts{s} = 0$ and the function $\sigma\pts{s}/s$ is increasing in $s \in \pts{0, \ellf{0}}$. Additionally, since $\ell^{-1}\pts{s} < 0$ for $s \in \pts{0, \ellf{0}}$, $\sigma\pts{s}/s$ is non-positive in $s \in \pts{0, \ellf{0}}$.

    Using these properties, we will show that $\bar{g}_a\pts{z}$ is decreasing in $a\in\pts{0, \ellf{0}}$ by showing that $\frac{\partial \bar{g}_a\pts{z}}{\partial a} \leq 0$.
    In particular, we have
    \begin{align*}
        \frac{\partial \bar{g}_a\pts{z}}{\partial a} &= \frac{\frac{\ell^{-1}\pts{a\cdot z}\cdot z}{\ellpellnofz{a}} - \frac{\ell^{-1}\pts{a}\cdot z^2}{\ellpellnofz{a\cdot z}}}{\mts{\ell^{-1}\pts{a\cdot z}}^2}\\
        &= \frac{\frac{az^2}{\ellpellnofz{a} \cdot \ellpellnofz{a\cdot z}}}{\mts{\ell^{-1}\pts{a\cdot z}}^2} \cdot \pts{\frac{\ell^{-1}\pts{a\cdot z} \cdot \ellpellnofz{a\cdot z}}{a\cdot z} - \frac{\ell^{-1}\pts{a} \cdot \ellpellnofz{a}}{a}}\\
        &= \frac{\frac{az^2}{\ellpellnofz{a}\ellpellnofz{a\cdot z}}}{\mts{\ell^{-1}\pts{a\cdot z}}^2} \cdot \pts{\frac{\sigma\pts{a\cdot z}}{a\cdot z} - \frac{\sigma\pts{a}}{a}} \leq 0,
    \end{align*}
    where the last step follows for $z\in(0, 1]$ as the function $\sigma(s)/s$ is increasing in $s$. Finally, since $\ellpellnofz{\cdot}$ is an increasing function, we have $\frac{\ellpellnofz{a\cdot z}}{\ellpellnofz{a}} \leq 1$ for all $a\in\pts{0, \ellf{0}}$ and $z\in(0, 1]$, meaning that $\bar{g}_a\pts{z} \leq 1$. 
    
    Thus, we have shown that $\bar{g}_a\pts{z}$ is decreasing in $a$ (therefore, increasing as $a \downarrow 0$) and is bounded above by $1$. By the monotone convergence theorem, its limit exists as $a \to 0$.
    This completes the proof of the lemma.
\end{proof}

Armed with Lemma~\ref{lem:g_convergence}, we now provide the proof of Lemma~\ref{lem:g_func}.
\begin{proof} (of Lemma~\ref{lem:g_func})\\
    We showed in Lemma~\ref{lem:g_convergence} that the limit of the sequence of functions $\{g_a(z)\}_{a > 0}$ exists as $a \to 0$.
    Accordingly, we define $\gf{z} \coloneqq \ulim{a \to 0} \frac{\ellpellnofz{a\cdot z}}{\ellpellnofz{a}}$ for $z\in(0, 1]$. It remains for us to show that $\gf{z}$ is strictly increasing and convex.
    We first show that the derivative of $\gf{z}$ exists. For this, we reuse the co-convergent sequence defined in Lemma~\ref{lem:g_convergence}, i.e. $\bar{g}_a\pts{z} \coloneqq \frac{\ell^{-1}\pts{a}\cdot z}{\ell^{-1}\pts{a\cdot z}}$. Specifically, we want to show that
    \begin{align}\label{eq:conv_derivative}
       g'\pts{z} = \frac{\partial}{ \partial z} \ulim{a \to 0} \;\frac{\ell^{-1}\pts{a}\cdot z}{\ell^{-1}\pts{a\cdot z}} = \ulim{a \to 0} \frac{\partial}{\partial z}  \;\frac{\ell^{-1}\pts{a}\cdot z}{\ell^{-1}\pts{a\cdot z}} = \ulim{a \to 0}\; \bar{g}'_a\pts{z},
    \end{align}
    where in the above, the second equality will hold if $\ulim{a \to 0}\;\bar{g}'_a\pts{z}$ converges uniformly. To show this, we provide a direct calculation as below:
    \begin{align*}
        \bar{g}'_a\pts{z} &= \frac{\ell^{-1}\pts{a}\cdot \ell^{-1}\pts{a\cdot z} - \ell^{-1}\pts{a}\cdot z \cdot \frac{a}{\ell'\pts{\ell^{-1}\pts{a\cdot z}}} }{\mts{\ell^{-1}\pts{a\cdot z}}^2}\\
        &= \frac{\bar{g}_a\pts{z}}{z}\pts{1 - \frac{a\cdot z}{\ell'\pts{\ell^{-1}\pts{a\cdot z}} \cdot \ell^{-1}\pts{a\cdot z}}}\\
        &= \frac{\bar{g}_a\pts{z}}{z}\pts{1 - \frac{a\cdot z}{\sigma\pts{a\cdot z}}},
    \end{align*}
    where we defined the function $\sigma\pts{s}$ in the proof of Lemma~\ref{lem:g_convergence}.
    Since the function $\sigma\pts{s} / s$ is increasing on $s \in \pts{0, \ellf{0}}$ and non-positive, the function $s/ \sigma\pts{s}$ is decreasing in $s$ (therefore, increasing as $s \downarrow 0$), non-positive and bounded above by $0$. 
    By the monotone convergence theorem, we then have $\ulim{a \to 0}\;\frac{a\cdot z}{\sigma\pts{a\cdot z}} = \ulim{s \to 0}\; \frac{s}{\sigma\pts{s}} = c \leq 0$. As a result, we have
    \begin{align}
        \ulim{a \to 0}\; \bar{g}'_a\pts{z} &= \ulim{a \to 0}\; \frac{\bar{g}_a\pts{z}}{z}\pts{1 - \frac{a\cdot z}{\sigma\pts{a\cdot z}}}\nonumber \\
        &= \ulim{a \to 0}\; \frac{\bar{g}_a\pts{z}}{z}\cdot \ulim{a \to 0}\;\pts{1 -\frac{a\cdot z}{\sigma\pts{a\cdot z}}}\nonumber\\
        &= \frac{\gf{z}}{z} \cdot \pts{1 - c} \label{eq:1st_g_derivative}.
    \end{align}
    Therefore, Equation~\eqref{eq:conv_derivative} holds with $g'\pts{z} = \frac{\pts{1 - c}\gf{z}}{z}$. 
    Because $c \leq 0$, $g(z) > 0$ and $z > 0$, we can conclude that $g'\pts{z} > 0$ and so $g(\cdot)$ is a strictly increasing function.
    It remains to show convexity.
    Using a similar procedure to the above, we can show that
    \begin{align*}
        g''\pts{z} = \frac{\partial}{\partial z} \ulim{a \to 0}\; \bar{g}'_a\pts{z} = \frac{\partial}{\partial z}\frac{\pts{1 - c}\gf{z}}{z} = \frac{\pts{1 - c}\pts{g'\pts{z}\cdot z - \gf{z}}}{z^2} = \gf{z}\pts{\frac{\pts{1 - c}\pts{-c}}{z^2}}.
    \end{align*}
    Finally, since $c \leq 0$ and $\gf{z}$ is non-negative, we can conclude that $g''\pts{z}\geq 0$ and therefore the function $g(\cdot)$ is convex.
    This completes the proof of the lemma.
\end{proof}

\subsection{Proof of Lemma~\ref{lem:b_imp_bias_relax} (auxiliary convex program)}\label{sec:b_relaxproofs}
In this section, we prove Lemma~\ref{lem:b_imp_bias_relax}. The proof follows a two-step procedure. First, we show that the new equality constraint in~\eqref{eq:b_imp_bias_relax_eq} is sufficient to imply the original equality constraint in~\eqref{eq:b_barq_def}. Second, we show that any solution to the original program also satisfies the new equality constraint; therefore, the solution sets of both programs coincide. The following lemma demonstrates the first part of this reasoning.

\begin{lemma}\label{lem:b_constraint_equal}
    Under Assumption~\ref{asm:b_loss_assump}, for any $\q\in\R^n$ such that $q_i = \gf{z_i} = \ulim{a \to 0} \frac{\ellpellnofz{a\cdot z_i}}{\ellpellnofz{a}} > 0$, for all $i \in [n]$ with $z_i \in (0, 1]$ and $\sum_{i=1}^n z_i = 1$, it implies $\psif[*]{\q} = 0$.
\end{lemma}
\begin{proof}
    In this proof, we substantially apply the convex analysis in the astral space introduced in~\cite{dudik2022convex}. 
    Informally, astral space $\overline{\R^n}$ consists of the union of the set $\R^n$ and all ``astral points", i.e. $n$-dimensional points at infinity.
    Accordingly, we define the astral extension function to take into account of astral points naturally~\citep[Chapter 7]{dudik2022convex}; e.g. $\overline{\ell}: \overline{\R}\to\overline{\R}$.

    Using this framework, we show that $\q$ is a subdifferential of the astral extension of $\psi$ for some $\overline{\p}\in\overline{\R^n}$, and then the astral convex conjugate (which is equivalent to the original convex conjugate) is equal to zero. We start with writing $\q$ in the astral format. Since $q_i$ is finite and continuous in $a$, we can take the limit inside the function, obtaining
    \begin{align}
        q_i = \ulim{a \to 0} \frac{\ellpellnofz{a\cdot z_i}}{\ellpellnofz{a}} =  \frac{\ellpellnofzo{\ulim{a \to 0}a\cdot z_i}}{\ellpellnofzo{\ulim{a \to 0} a}} = \frac{\overline{\ell}'\Bigl(\overline{\ell}^{-1}\Bigl(\ulim{a \to 0}a\cdot z_i\Bigl)\Bigl)}{\overline{\ell}'\Bigl(\overline{\ell}^{-1}\Bigl(\ulim{a \to 0} a\Bigl)\Bigl)},\label{eq:b_q_astral}
    \end{align}
    where in the last equality, we replace the original functions with their astral extensions. Next, we define an astral point $\overline{\p}\in\overline{\R^n}$ such that $\overline{p}_i = \ulim{a\to 0}\; \overline{\ell}^{-1}\pts{a\cdot z_i} =  \overline{\ell}^{-1}\Bigl(\ulim{a\to 0}\;a\cdot z_i\Bigl)$ for all $i\in[n]$. This also implies $\sumn \overline{\ell}\pts{\overline{p}_i} = \ulim{a\to 0} a$. Substituting these values in Eq.~\eqref{eq:b_q_astral}, we can write $q_i$ as
    \begin{align*}
        q_i = \frac{\overline{\ell}'\pts{\overline{p}_i}}{\overline{\ell}'\Bigl(\overline{\ell}^{-1}\Bigl(\sumn \overline{\ell}\pts{\overline{p}_i}\Bigl)\Bigl)}
    \end{align*}
    for all $i\in[n]$.
    On the other hand, according to Lemma~\ref{lem:g_func}, we can also define $q_i$ in the limit of a different co-convergent sequence as
    \begin{align}\label{eq:b_q_astral2}
        q_i = \ulim{a \to 0} \;\frac{\ell^{-1}\pts{a}\cdot z_i}{\ell^{-1}\pts{a\cdot z_i}} = \frac{\ell^{-1}\Bigl( \ulim{a \to 0} a\Bigl)\cdot z_i}{\ell^{-1}\Bigl( \ulim{a \to 0} a \cdot z_i\Bigl)} = \frac{\overline{\ell}^{-1}\Bigl( \ulim{a \to 0} a\Bigl)\cdot z_i}{\overline{\ell}^{-1}\Bigl(\ulim{a \to 0} a \cdot z_i\Bigl)} = \frac{\overline{\ell}^{-1}\pts{ \sumn\overline{\ell}(\overline{p}_i)}\cdot z_i}{\overline{p}_i}
    \end{align}
    for all $i \in[n]$.
    Next, we show that $\q$ is a subdifferential of $\overline{\psi}\pts{\overline{\p}}$. By the definition of $\overline{\psi}:\overline{\R^n}\to\overline{\R}$ and for any $\p\in\overline{\R^n}$, we have
    \begin{align*}
        \overline{\psi}\pts{\p} = \overline{\ell}^{-1}\biggl(\sumn\overline{\ell}\pts{p_i}\biggl), \text{ and } \frac{\partial}{\partial p_i} \overline{\psi}\pts{\p} = \frac{\overline{\ell}'\pts{p_i}}{\overline{\ell}'\Bigl(\overline{\ell}^{-1}\Bigl(\sumn\overline{\ell}\pts{p_i}\Bigl)\Bigl)},
    \end{align*}
    for all $i\in[n]$. Therefore, according to Eq.~\eqref{eq:b_q_astral}, it implies that $\q$ is in the subdifferential of $\overline{\psi}\pts{\overline{\p}}$ such that $\q \in \partial \overline{\psi}\pts{\overline{\p}}$. 
    Hence, we can further apply the property of Fenchel–Young inequality in the convex conjugate, obtaining
    \begin{align*}
        \overline{\psi}^*\pts{\q} = \usup{\p\in\overline{\R^n}}\ip{\p}{\q}-\overline{\psi}\pts{\p} = \ip{\overline{\p}}{\q}-\overline{\psi}\pts{\overline{\p}}.
    \end{align*}
    As a result, by Eq.~\eqref{eq:b_q_astral2} and the definition of $\overline{\psi}\pts{\cdot}$, we can write
    \begin{align*}
        \overline{\psi}^*\pts{\q} = \ip{\overline{\p}}{\q}-\overline{\psi}\pts{\overline{\p}} = \sumn \overline{p}_i \cdot \frac{\overline{\ell}^{-1}\pts{ \sumn\overline{\ell}(\overline{p}_i)}\cdot z_i}{\overline{p}_i} - \overline{\ell}^{-1}\biggl( \sumn\overline{\ell}(\overline{p}_i)\biggl) = 0.
    \end{align*}
    Finally, by~\cite[Proposition 8.5]{dudik2022convex}, we have $\psi^*\pts{\q} = \overline{\psi}^*\pts{\q} = 0$. This completes the proof of the lemma.
    
\end{proof}

Next, we show in the following lemma that $\barbq=\lim_{t \to \infty} \q_t$ under Assumption~\ref{asm:b_lin_sep}.

\begin{lemma} \label{lem:b_qt_eq_bq}
    Under Assumption~\ref{asm:b_loss_assump} and~\ref{asm:b_lin_sep}, the gradient descent dual variable $\q_t$ converges to $\barbq$ when $t\rightarrow\infty$. It gives $\barq_i = \lim_{t \to \infty} q_{t,i} = g(z_i) > 0$ for all $i \in [n]$ with some $z_i \in (0, 1]$ and $\sum_{i=1}^n z_i = 1$.
\end{lemma}
\begin{proof}
    According to~\cite[Theorem 5]{ji2021characterizing}, we have $\limt \Z^\top\diag{\y}\q_t=\Z^\top\diag{\y}\barbq$, and $\Z^\top\diag{\y}\barbq$ is the same for all $\barbq\in\uargmin{\psif[*]{\q}\leq 0}\ff{\q}$. Based on the definition of $q_{t,i}$ in Eq.~\eqref{eq:dual-primal-t} and considering $\ellpf{\cdot}$ is an increasing function with $p_{t,i} \leq \psif{\p_t}$ for all $i\in[n]$ and $t\geq0$, it follows that $0<q_{t,i}\leq1$; hence, $\limt \Z^\top\diag{\y}\q_t=\Z^\top\diag{\y}\limt \q_t=\Z^\top\diag{\y}\barbq$. Next, by Assumption~\ref{asm:b_lin_sep}, $\Z^\top\diag{\y}$ has full column rank, and $\diag{\y}\Z\Z^\top\diag{\y}\succ \zero$. Therefore, we can multiply the pseudo-inverse of $\Z^\top\diag{\y}$ on both sides of $\Z^\top\diag{\y}\limt \q_t=\Z^\top\diag{\y}\barbq$, which implies $\limt \q_t=\barbq$. Next, by the  definition of $q_{t,i}$ in Eq.~\eqref{eq:dual-primal-t} and the primal convergence in~\cite[Theorem 1]{ji2021characterizing} such that $\limt \sum_{i=1}^n \ellf{p_{t,i}} = 0$, we have
    \begin{align*}
        \limt q_{t,i} = \limt \frac{\ellpellnofz{\ellf{p_{t,i}}}}{\ellpellnofz{\sum_{i=1}^n \ellf{p_{t,i}}}} = \ulim{a \to 0} \frac{\ellpellnofz{a\cdot z_i}}{\ellpellnofz{a}} = g(z_i),
    \end{align*}
    where we let $a = \sum_{i=1}^n \ellf{p_{t,i}}$, $\ellf{p_{t,i}} = a \cdot z_i$, and $z_i = \limt \frac{\ellf{p_{t,i}}}{\sum_{i=1}^n \ellf{p_{t,i}}}$. Finally, Assumption~\ref{asm:b_loss_assump} guarantees that $\barq_i = g(z_i) > 0$. This completes the proof of the lemma.
\end{proof}

Armed with Lemmas~\ref{lem:b_constraint_equal} and~\ref{lem:b_qt_eq_bq}, we can prove Lemma~\ref{lem:b_imp_bias_relax}.

\begin{proof} (of Lemma~\ref{lem:b_imp_bias_relax})
    We start with the original convex program in~\eqref{eq:b_barq_def}:
    \begin{align*}
        \barbq \in \uargmin{\psif[*]{\q}\leq 0}\ff{\q}.
    \end{align*}
    By complementary slackness in the KKT conditions, if the constraint is inactive such that $\psif[*]{\q}< 0$, we get an invalid solution $\barbq = \zero$. Therefore, the constraint is active and $\barbq$ satisfies $\psif[*]{\q} = 0$. 
    In other words, we can write
    \begin{align}\label{eq:b_barq_def_tightened_1}
        \barbq \in \uargmin{\psif[*]{\q} = 0}\ff{\q}.
    \end{align}
    Now, Lemma~\ref{lem:b_qt_eq_bq} directly implies that $\barbq$ must satisfy $\sum_{i=1}^n g^{-1}(\barq_i) = 1$ and $\barq_i > 0$ for all $i\in[n]$. This means that we can further tighten~\eqref{eq:b_barq_def_tightened_1} to obtain
        \begin{gather} \label{eq:b_barq_def_tightened_2}
        \barbq \in \uargmin{\q\in\R^n} \ff{\q} \\
        \begin{aligned}
        \myquad[1]\text{subject to} \myquad[5] \psif[*]{\q} &= 0,\nonumber \\
        -q_i &< 0 \myquad[2]\text{for all } i\in [n],\nonumber\\ 
        \text{and} \myquad[1] 1-\sumn \gf[-1]{q_i}  &= 0\nonumber.
        \end{aligned}
    \end{gather}
    Next, Lemma~\ref{lem:b_constraint_equal} tells us that $1-\sumn \gf[-1]{q_i}  = 0 \implies \psif[*]{\q} = 0$, meaning that the constraint $\psif[*]{\q} = 0$ is redundant and can simply be omitted, leading to the simplified program
        \begin{gather} \label{eq:b_barq_def_tightened_3}
        \barbq \in \uargmin{\q\in\R^n} \ff{\q} \\
        \begin{aligned}
        \text{subject to} \myquad[6]
        -q_i &< 0 \myquad[2]\text{for all } i\in [n],\nonumber\\ 
        \text{and} \myquad[1] 1-\sumn \gf[-1]{q_i}  &= 0\nonumber.
        \end{aligned}
    \end{gather}
    The final step is to derive an auxiliary convex program
         \begin{gather} \label{eq:b_imp_bias_relax_eq_repeat}
        \q^\star \in \uargmin{\q\in\R^n}\usub{\ff{\q}}{\frac{1}{2}\q^\top\yXXy\q} \\
        \begin{aligned} \nonumber
            \text{subject to} \myquad[6] -q_i                                        &< 0 \myquad[2] \text{for all}\; i \in [n],\\
            \text{and} \myquad[1] 1-\sumn \gf[-1]{q_i}  &\leq 0.
        \end{aligned}
    \end{gather}
    Note that in the above, we have relaxed the equality constraint $1 - \sumn \gf[-1]{q_i} = 0$ to an inequality constraint, $1 - \sumn \gf[-1]{q_i} \leq 0$.
    To complete the proof, it remains to show that any optimal solution to~\eqref{eq:b_imp_bias_relax_eq_repeat} satisfies $\sum_{i=1}^n \gf[-1]{q_i} = 1$. From~\eqref{eq:b_barq_def_tightened_3}, this directly implies that the set of optima of~\eqref{eq:b_barq_def} and~\eqref{eq:b_imp_bias_relax_eq_repeat} are identical.
    We now show this final step. It is necessary and sufficient for any optimal solution $\q^\star$ to the auxiliary convex program~\eqref{eq:b_imp_bias_relax_eq_repeat} to satisfy its KKT conditions, listed below:
    \begin{subequations}\label{eq:b_kkt_relaxed}
    \begin{align}
        -q_i                                    &<      0 \myquad[2] \text{for all } i\in[n],\label{b_eigen_primal_1}\\
        1-\sumn \gf[-1]{q_i}                    &\leq   0,\label{b_eigen_primal_2}\\
        \lambda_i                               &\geq   0 \myquad[2] \text{for all } i\in[n],\label{b_eigen_dual_1}\\
        \mu                                     &\geq   0,\label{b_eigen_dual_2}\\
        -\lambda_i q_i                          &=      0 \myquad[2] \text{for all } i\in[n],\label{b_eigen_comp_1}\\
        \mu\biggl(1-\sumn \gf[-1]{q_i}\biggl)           &=      0,\label{b_eigen_comp_2}\\
        \yXXy\q-\blambda-\mu \gpf[-1]{\q}  &=      \zero,\label{b_eigen_station_1}
    \end{align}
    \end{subequations}
    where $\gpf[-1]{\q}:=\vct{\gpf[-1]{q_1}}{\gpf[-1]{q_n}}$.
    First, we claim that any optimal solution $\q^\star$ needs to satisfy $1-\sumn \gf[-1]{q^\star_i} = 0$.
    This follows because we need to set $\mu > 0$ for a valid solution; together with Eq.~\eqref{b_eigen_comp_2} this implies that we need $1-\sumn \gf[-1]{q_i^\star} = 0$. 
    To see why we need to set $\mu > 0$, consider the alternative choice $\mu = 0$.
    Note that Equations~\eqref{b_eigen_primal_1} and~\eqref{b_eigen_comp_1} together also require $\blambda = \zero$.
    Eq.~\eqref{b_eigen_station_1} would then become 
    \begin{align*}
        \XX \diag{\y} \q^\star = \zero \iff \diag{\y} \q^\star = \zero \iff \q^\star = \zero,
    \end{align*}
    where the first iff statement follows because we have assumed that $\XX \succ \zero$.
    However, this $\q^\star$ is not a valid solution as it violates Eq.~\eqref{b_eigen_primal_1}. Hence, we can conclude both $\barbq$ and $\q^\star$ satisfy $1-\sumn \gf[-1]{q_i} = 0$, and $\q^\star = \barbq$. This completes the proof of the lemma.
\end{proof}

\subsection{Proof of Proposition~\ref{thm:b_eqv_thm} (exact equivalence to MNI)}\label{sec:beqvthmproof}

\begin{proof}
The proof of Proposition~\ref{thm:b_eqv_thm} is divided into two parts.
\paragraph{Proof of Part 1}
By Lemma~\ref{lem:b_imp_bias_relax} it suffices to characterize an optimal solution to the relaxed convex program~\eqref{eq:b_imp_bias_relax_eq}, reproduced below.
    \begin{gather*}
        \barbq = \uargmin{\q\in\R^n}\usub{\ff{\q}}{\frac{1}{2}\q^\top\yXXy\q}\\
        \begin{aligned}
            \text{subject to} \myquad[6] -q_i                                        &< 0 \myquad[2] \text{for all}\; i\in[n],\\
            \text{and} \myquad[1] 1-\sumn \gf[-1]{q_i}  &\leq 0.
        \end{aligned}
    \end{gather*}

    Any optimal solution must satisfy the KKT conditions for this convex program, listed in Eq.~\eqref{eq:b_kkt_relaxed}.
    Let $k > 0$ be the positive eigenvalue corresponding to the exact eigenvector $\y$, i.e. we consider $\XX \y = k \y$.
    We choose the candidate solution $\barbq=\gf{\frac{1}{n}}\one \propto \one$, and verify that it satisfies all the KKT conditions below.
    \begin{itemize}
        \item The primal feasibility equations, Eq.~\eqref{b_eigen_primal_1} and Eq.~\eqref{b_eigen_primal_2} are satisfied because $\gf{\frac{1}{n}} > 0$ (as $\gf{\cdot}$ is non-negative), and $1 - \sum_{i=1}^n \gf[-1]{q_i} = 1 - \sum_{i=1}^n \frac{1}{n} = 0$.
        \item The dual feasibility equations, Eq.~\eqref{b_eigen_dual_1} and Eq.~\eqref{b_eigen_dual_2} are satisfied by setting $\blambda =\zero$ and $\mu=\frac{k\gf{\frac{1}{n}}}{\gpf[-1]{\gf{\frac{1}{n}}}}$.
        \item The complementary slackness equations, Eq.~\eqref{b_eigen_comp_1} and Eq.~\eqref{b_eigen_comp_2} are satisfied because $\blambda = \zero$ and $1 - \sumn \gpf[-1]{q_i} = 0$.
        \item The stationary condition, Eq.~\eqref{b_eigen_station_1} is satisfied because
        \begin{align*}
            &\diag{\y} \XX \diag{\y} \q - \blambda - \mu \gpf[-1]{\q} \\
            &= k\gf{\frac{1}{n}} \cdot \diag{\y} \y - \mu \left[g^{-1}\right]'\biggl(\gf{\frac{1}{n}}\biggl) \cdot \one \\
            &= k \gf{\frac{1}{n}} \one - k \gf{\frac{1}{n}} \one = \zero.
        \end{align*}
    \end{itemize}
    
    This shows that the candidate solution $\barbq=\gf{\frac{1}{n}}\one$ is indeed optimal.
    By Lemma~\ref{lem:b_imp_bias}, we have $\barw=\ulim{\tinf{t}}\normalize[2]{\w_t}=\normalize[2]{\X^\top\diag{\y}\barbq}=\normalize[2]{\gf{1/n}\X^\top\y}$.
    On the other hand, since $\y$ is an exact eigenvector of $\XX$, we have $\wmni=\XXXy=\frac{1}{k}\X^\top\y$ for some positive eigenvalue $k>0$.
    Therefore, $\wmni \propto \barw$.
    This completes the proof of Part 1 of the proposition.

\paragraph{Proof of Part 2}
As with the proof of Part 1 of the proposition, we start by analyzing the convex program~\eqref{eq:b_imp_bias_relax_eq}.
In the special case $\gf{d}=d$, the KKT conditions reduce to
\begin{subequations}
\begin{align}
        -q_i                                    &<      0 \myquad[2] \text{for all } i\in[n],\label{b_exp_primal_1}\\
        1-\sumn q_i                             &\leq   0,\label{b_exp_primal_2}\\
        \lambda_i                               &\geq   0 \myquad[2] \text{for all } i\in[n],\label{b_exp_dual_1}\\
        \mu                                     &\geq   0,\label{b_exp_dual_2}\\
        -\lambda_i q_i                          &=      0 \myquad[2] \text{for all } i\in[n],\label{b_exp_comp_1}\\
        \mu\biggl(1-\sumn q_i\biggl)                    &=      0,\label{b_exp_comp_2}\\
        \yXXy\q-\blambda-\mu\one  &=      \zero.\label{b_exp_station_1}
\end{align}
\end{subequations}

In this case, we pick the candidate solution $\barbq = \normalize[1]{\diag{\y}\XXiy}$ and verify that it satisfies all the KKT conditions below.

    \begin{itemize}
        \item The primal feasibility equations, Eq.~\eqref{b_exp_primal_1} and Eq.~\eqref{b_exp_primal_2} are satisfied because of our assumed condition $\diag{\y}\XXiy \succ \zero$ and because, by definition, $\sumn \barq_i = 1$.
        \item The dual feasibility equations, Eq.~\eqref{b_exp_dual_1} and Eq.~\eqref{b_exp_dual_2} are satisfied by setting $\blambda =\zero$ and $\mu=\frac{1}{\nnorm[1]{\diag{\y}\XXiy}}$.
        \item The complementary slackness equations, Eq.~\eqref{b_exp_comp_1} and Eq.~\eqref{b_exp_comp_2} are satisfied because we have set $\blambda = \zero$ and $1 - \sumn q_i = 0$.
        \item The stationary condition Eq.~\eqref{b_exp_station_1} is satisfied because
        \begin{align*}
            &\diag{\y} \XX \diag{\y} \q - \blambda - \mu \one \\
            &= \frac{1}{\nnorm[1]{\diag{\y}\XXiy}} \one -  \frac{1}{\nnorm[1]{\diag{\y}\XXiy}} \one = \zero,
        \end{align*}
        where we have used $\diag{\y}\diag{\y} = \I$ due to the labels being binary, i.e. $y_i = \pm 1$.
    \end{itemize}
    
    Therefore, the candidate solution $\barbq$ is optimal.
    By Lemma~\ref{lem:b_imp_bias}, we have $\barw=\ulim{\tinf{t}}\normalize[2]{\w_t}=\normalize[2]{\X^\top\diag{\y}\barbq}=\normalize[2]{\XXXy}$ . Therefore, $\wmni \propto \barw$.
    This completes the proof of Part 2 of the proposition.
\end{proof}

\subsection{Proof of Theorem~\ref{thm:b_sen_upperbound} (approximate equivalence to MNI upper bound)} \label{sec:proofofsensitivity}

In this section, we present the proof of Theorem~\ref{thm:b_sen_upperbound}. We divide the proof in four steps. 
\paragraph{Step 1.} Our proof starts with the relaxed convex program~\eqref{eq:b_imp_bias_relax_eq} and identifies a necessary set of characteristic equations that the optimal solution $\barbq$ needs to satisfy.
The KKT conditions for this convex program are given in Eq.~\eqref{eq:b_kkt_relaxed}.
Lemma~\ref{lem:b_imp_bias_relax} postulates that any optimal solution must satisfy $\blambda = \zero$ and $\mu > 0$; therefore, it is necessary for $\barbq$ to satisfy the following characteristic equations:
\begin{subequations}\label{b_sen_charc}
\begin{align}
    \yXXy\barbq                 &=      \mu \hf{\barbq},      \label{b_sen_charc_1}\\
    \mu                         &>      0,   \label{b_sen_charc_2}\\
    \sumn \gf[-1]{\barq_i}      &=      1,  \label{b_sen_charc_3}
\end{align}
\end{subequations}
where we have denoted $\hf{\q}:=\gpf[-1]{\q}$ as shorthand.
\paragraph{Step 2.} 
Next, we use the nonlinear characteristic equations in Eq.~\eqref{b_sen_charc} to determine a relationship between the directions of the vectors $\barbq$ and $h(\barbq)$.
We denote $\qy := \diag{\y} \barbq$ and $\epsilon_{\alpha}(\q) := \frac{\nnorm[2]{\XX \q - \alpha \q}}{\nnorm[2]{\q}}$ as shorthand.
From Eq.~\eqref{b_sen_charc_1}, we have the following sequence of implications for any value of $\alpha > 0$:
\begin{align}
    \yXXy\barbq                 &=          \mu \hf{\barbq}\nonumber\\
    \iff \XX\qy     &=          \mu \diag{\y}\hf{\barbq}\nonumber\\
    \iff \pts{\XXaI}\qy &= \mu\diag{\y}\hf{\barbq}-\alpha\qy\nonumber\\
    \Rightarrow \nnorm[2]{\pts{\XXaI}\qy} &= \nnorm[2]{\mu\diag{\y}\hf{\barbq}-\alpha\qy}\nonumber\\
    \iff \epsilon_{\alpha}(\qy) &= \nnorm[2]{\mu\diag{\y}\frac{\hf{\barbq}}{\nnorm[2]{\qy}}-\alpha\normalize[2]{\qy}} \nonumber\\
    \iff \epsilon_{\alpha}(\qy) &= \nnorm[2]{\mu\frac{\hf{\barbq}}{\nnorm[2]{\barbq}}-\alpha\normalize[2]{\barbq}}\nonumber\\
    \iff \epsilon_{\alpha}(\qy) &= \nnorm[2]{\frac{\nnorm[2]{\yXXy\barbq}}{\nnorm[2]{\hf{\barbq}}}\frac{\hf{\barbq}}{\nnorm[2]{\barbq}}-\alpha\frac{\barbq}{\nnorm[2]{\barbq}}}.\label{b_sen_inter_res1}
\end{align}

The last implication follows because Eq.~\eqref{b_sen_charc_1} implies $\nnorm[2]{\yXXy\barbq} = \mu\nnorm[2]{\hf{\barbq}}$. 
We proceed from Eq.~\eqref{b_sen_inter_res1}.
By the reverse triangle inequality, we have

\begin{align}
    \epsilon_{\alpha}(\qy) &= \nnorm[2]{\pts{\alpha\normalize[2]{\hf{\barbq}}-\alpha\normalize[2]{\barbq}}-\pts{\alpha\normalize[2]{\hf{\barbq}}-\frac{\nnorm[2]{\yXXy\barbq}}{\nnorm[2]{\barbq}}\normalize[2]{\hf{\barbq}}}}\nonumber\\
    \epsilon_{\alpha}(\qy) &\geq \normo{\nnorm[2]{\alpha\normalize[2]{\hf{\barbq}}-\alpha\normalize[2]{\barbq}}-\nnorm[2]{\alpha\normalize[2]{\hf{\barbq}}-\frac{\nnorm[2]{\yXXy\barbq}}{\nnorm[2]{\barbq}}\normalize[2]{\hf{\barbq}}}}\nonumber\\
    &=\normo{\nnorm[2]{\alpha\normalize[2]{\hf{\barbq}}-\alpha\normalize[2]{\barbq}}-\usub{C}{\nnorm[2]{\pts{\alpha-\frac{\nnorm[2]{\yXXy\barbq}}{\nnorm[2]{\barbq}}}\normalize[2]{\hf{\barbq}}}}}\nonumber\\
    &\Rightarrow \nnorm[2]{\alpha\normalize[2]{\hf{\barbq}}-\alpha\normalize[2]{\barbq}}\leq \epsilon_{\alpha}(\qy)+C. \label{b_sen_inter_res2}
\end{align}

Therefore, it suffices to upper-bound the term $C$.
We get
\begin{align*}
    C&=\nnorm[2]{\pts{\alpha-\frac{\nnorm[2]{\yXXy\barbq}}{\nnorm[2]{\barbq}}}\normalize[2]{\hf{\barbq}}}\\
    &=\normo{\alpha-\frac{\nnorm[2]{\yXXy\barbq}}{\nnorm[2]{\barbq}}}\nnorm[2]{\normalize[2]{\hf{\barbq}}}\\
    &=\normo{\alpha-\frac{\nnorm[2]{\XX\qy}}{\nnorm[2]{\qy}}}\\
    &\leq \frac{\nnorm[2]{\XX \qy - \alpha \qy}}{\nnorm[2]{\qy}} := \epsilon_{\alpha}(\qy),
\end{align*}
where the last inequality follows by again applying the reverse triangle inequality.
Hence, Eq.~\eqref{b_sen_inter_res2} together with the upper bound on $C$ gives us
\begin{align}
    \nnorm[2]{\normalize[2]{\hf{\barbq}}-\normalize[2]{\barbq}} \leq \frac{2 \epsilon_{\alpha}(\qy)}{\alpha}.\label{b_sen_inter_res3}
\end{align}
\paragraph{Step 3.} 
Next, we show that the angle between $\barbq$ and $\one$ is less than or equal to the angle between $\hf{\barbq}$ and $\barbq$.
We introduce the following key lemma, which critically utilizes the convexity of $g(\cdot)$.

\begin{lemma} \label{b_sen_chebyshv}
    For every non-negative, strictly increasing, convex function $g: [0,1] \rightarrow [0,1]$, where $\gf{0}=0$ and $\gf{1}=1$, we have $\hf{q}:=\gpf[-1]{q}$ is a decreasing function satisfying 
    \begin{align}
        \frac{\sqrt{n}\sumn \hf{q_i}q_i}{\pts{\sumn q_i}\nnorm[2]{\hf{\q}}}\leq1
    \end{align}
    for all $0\leq q_i\leq1$ and $\q=\vct{q_1}{q_n}$.
\end{lemma}

\begin{proof}
    Without loss of generality, we assume $\seq{q_i}$ is an increasing sequence, where $q_i\leq q_j$ if index $i\leq j$. Next, since $\gf{\cdot}$ is convex and strictly increasing, $\gf[-1]{\cdot}$ is a concave function, and then $\hf{\cdot}:=\gpf[-1]{\cdot}$ is a decreasing function. Hence, $\seq{\hf{q_i}}$ is a decreasing sequence, where $\hf{q_i}\geq \hf{q_j}$ for $i\leq j$. Then, according to Chebyshev's Sum Inequality~\citep{hardy1952inequalities}, we can have 
    \begin{align*}
         \frac{\sqrt{n}\sumn \hf{q_i}q_i}{\pts{\sumn q_i}\nnorm[2]{\hf{\q}}}\leq \frac{\frac{\sqrt{n}}{n}\pts{\sumn \hf{q_i}}\pts{\sumn q_i}}{\pts{\sumn q_i}\nnorm[2]{\hf{\q}}}=\frac{\nnorm[1]{\hf{\q}}}{\sqrt{n}\nnorm[2]{\hf{\q}}}\leq1,
    \end{align*}
    where the last inequality holds because $\nnorm[1]{\cdot}\leq\sqrt{n}\nnorm[2]{\cdot}$.
\end{proof}

Then, Eq.~\eqref{b_sen_inter_res3} yields
\begin{align}
    \frac{2 \epsilon_{\alpha}(\qy)}{\alpha} &\geq\nnorm[2]{\normalize[2]{\hf{\barbq}}-\normalize[2]{\barbq}} \nonumber\\
    &=\sqrt{\sumn \frac{1}{n}-\frac{2\sumn \hf{\barq_i}\barq_i}{\nnorm[2]{\hf{\barbq}}\nnorm[2]{\barbq}}+\sumn \frac{\barq_i^2}{\nnorm[2]{\barbq}^2}}\nonumber\\
    &=\sqrt{\sumn \frac{1}{n}-\frac{2\sumn \barq_i}{\nnorm[2]{\barbq}\sqrt{n}}\times\frac{\sqrt{n}\sumn \hf{\barq_i}\barq_i}{\pts{\sumn \barq_i}\nnorm[2]{\hf{\barbq}}} +\sumn \frac{\barq_i^2}{\nnorm[2]{\barbq}^2}}\nonumber\\
    &\geq \sqrt{\sumn \frac{1}{n}-\frac{2\sumn \barq_i}{\nnorm[2]{\barbq}\sqrt{n}}\times1 +\sumn \frac{\barq_i^2}{\nnorm[2]{\barbq}^2}}\nonumber\\
    &=\nnorm[2]{\normalize[2]{\barbq}- \frac{1}{\sqrt{n}}}\label{eq:dualconvergenceqy},
\end{align}
where the last inequality follows by applying Lemma~\ref{b_sen_chebyshv}.
To complete the proof of dual variable convergence, we relate $\epsilon_{\alpha}(\qy)$ to $\epsilon_{\alpha}(\y)$.
Denote the unit-normalization of a vector $\bu(\q) := \frac{\q}{\nnorm[2]{\q}}$ as shorthand.
Note that for any vector $\q$, we have
\begin{align*}
\epsilon_{\alpha}(\q) := \nnorm[2]{(\XX - \alpha \I) \bu(\q)}.
\end{align*}

Therefore, we get
\begin{align*}
\epsilon_{\alpha}(\qy) &= \nnorm[2]{(\XX - \alpha \I) \bu(\qy)} \\
&\leq \nnorm[2]{(\XX - \alpha \I) \bu(\y)} + \nnorm[2]{(\XX - \alpha \I) (\bu(\qy) - \bu(\y)} \\
&\leq \epsilon_{\alpha}(\y) + \nnorm[2]{\XX - \alpha \I} \cdot \frac{2 \epsilon_{\alpha}(\qy)}{\alpha},
\end{align*}
where the last inequality follows by noting that $\nnorm[2]{\bu(\qy) - \bu(\y)} = \nnorm[2]{\normalize[2]{\barbq}- \frac{\one}{\sqrt{n}}}$ (owing to the binary labels $y_i = \pm 1$) and substituting Eq.~\eqref{eq:dualconvergenceqy}.
Next, we utilize the assumption made in the statement of Theorem~\ref{thm:b_sen_upperbound} that $\alpha > 0$ is chosen such that $\frac{\nnorm[2]{\XX - \alpha \I}}{\alpha} \leq 1/3 < 1/2$. This assumption yields
\begin{align*}
    \epsilon_{\alpha}(\qy) &\leq \epsilon_{\alpha}(\y) + 2c \cdot  \epsilon_{\alpha}(\qy) \\
    \Rightarrow \epsilon_{\alpha}(\qy) &\leq \frac{\epsilon_{\alpha}(\y)}{(1 - 2c)} =: \frac{C\epsilon_{\alpha}(\y)}{2},
\end{align*}
where $C := \frac{2}{(1 - 2c)} \in (0,\infty)$ is a universal positive constant.
Thus, we ultimately get
\begin{align}\label{eq:dualconvergencefinal}
\nnorm[2]{\normalize[2]{\barbq}- \frac{1}{\sqrt{n}}} \leq \frac{C \epsilon_{\alpha}(\y)}{\alpha} := \frac{C \nnorm[2]{\XX \y - \alpha \y}}{\alpha \nnorm[2]{\y}},
\end{align}
which completes our dual convergence proof.
\paragraph{Step 4.} 
We complete the proof with the following lemma, which relates the primal variables to the dual variables.

\begin{lemma}\label{lem:b_dual-primal}
Under the assumptions of Theorem~\ref{thm:b_sen_upperbound}, we have
\begin{align}\label{eq:primal-convergence}
 \nnorm[2]{\normalize[2]{\barw} - \normalize[2]{\wmni}} \leq  4 \nnorm[2]{\normalize[2]{\barbq}- \frac{1}{\sqrt{n}}} + \frac{12\epsilon_{\alpha}(\y)}{\alpha}.    
\end{align}
\end{lemma}

Lemma~\ref{lem:b_dual-primal} essentially shows that the statement of dual convergence in Eq.~\eqref{eq:dualconvergencefinal} can be converted into a statement of primal closeness with only the loss of a multiplicative constant factor.
The proof of Lemma~\ref{lem:b_dual-primal} follows via a series of algebraic manipulations and is listed below. Putting Lemma~\ref{lem:b_dual-primal} together with Eq.~\eqref{eq:dualconvergencefinal} completes the proof of Theorem~\ref{thm:b_sen_upperbound}.

\begin{proof}
Recall from Lemma~\ref{lem:b_imp_bias} that the primal implicit bias is defined as $\barw \propto \X^\top \diag{\y} \barbq$.
Also, recall the definition of the primal MNI as $\wmni = \XXXy$.
We define $\bumni := \normalize[2]{\XXiy}$ and $\barbu := \normalize[2]{\diag{\y}\barbq}$. Then, a simple normalization shows that
\begin{align*}
    \nnorm[2]{\normalize[2]{\barw} - \normalize[2]{\wmni}} = \nnorm[2]{\normalize[2]{\X^\top \barbu} - \normalize[2]{\X^\top \bumni}}.
\end{align*}

Now, we denote $\bu_1 := \barbu$, $\bu_2 := \bumni$ and $\bDelta := \bu_1 - \bu_2$ as shorthand.
We have
\begin{align*}
     \nnorm[2]{\normalize[2]{\barw} - \normalize[2]{\wmni}} &= \nnorm[2]{\normalize[2]{\X^\top \bu_1 } - \normalize[2]{\X^\top \bu_2}} \\
     &\leq \underbrace{\nnorm[2]{\normalize[2]{\X^\top \bu_1 } - \frac{\X^\top \bu_2}{\nnorm[2]{\X^\top \bu_1}}}}_{T_1} + \underbrace{\nnorm[2]{\frac{\X^\top \bu_2}{\nnorm[2]{\X^\top \bu_1}} - \normalize[2]{\X^\top \bu_2}}}_{T_2}.
\end{align*}

We first show that both $T_1$ and $T_2$ are upper bounded by $2 \nnorm[2]{\bDelta}$.
We denote $\epsilon := \nnorm[2]{\XX - \alpha \I}$ as shorthand, and note that by the assumptions of Theorem~\ref{thm:b_sen_upperbound} we have $\epsilon \leq \frac{\alpha}{3}$.
Beginning with $T_1$, note that
\begin{align*}
    T_1 &= \frac{\nnorm[2]{\X^\top \pts{\bu_1 - \bu_2}}}{\nnorm[2]{\X^\top \bu_1}} \leq \frac{\sigma_{\max}\pts{\X^\top} \cdot \nnorm[2]{\bDelta}}{\sigma_{\min}\pts{\X^\top}}\leq \sqrt{\frac{\alpha + \epsilon}{\alpha - \epsilon}} \cdot \nnorm[2]{\bDelta}\leq 2 \nnorm[2]{\bDelta}.
\end{align*}

Above, the last inequality uses that $\sqrt{\frac{\alpha + \epsilon}{\alpha - \epsilon}} \leq \frac{\alpha + \epsilon}{\alpha - \epsilon} \leq 2$ as long as $\epsilon \leq \frac{\alpha}{3}$.
Proceeding to $T_2$, we have
\begin{align*}
    T_2 &= \pts{\frac{1}{\nnorm[2]{\X^\top \bu_1}} - \frac{1}{\nnorm[2]{\X^\top \bu_2}}} \nnorm[2]{\X^\top \bu_2} \\
    &\leq \sqrt{\alpha + \epsilon} \cdot \frac{\nnorm[2]{\X^\top \bu_2} - \nnorm[2]{\X^\top \bu_1}}{\nnorm[2]{\X^\top \bu_1} \cdot \nnorm[2]{\X^\top \bu_2}} \\
    &\leq \frac{\sqrt{\alpha + \epsilon} \cdot \nnorm[2]{\X^\top \pts{\bu_1 - \bu_2}}}{\nnorm[2]{\X^\top \bu_1} \cdot \nnorm[2]{\X^\top \bu_2}} \\
    &\leq \frac{\pts{\alpha + \epsilon} \nnorm[2]{\bDelta}}{\pts{\alpha - \epsilon}} \\
    &\leq 2 \nnorm[2]{\bDelta},
\end{align*}
where in the above we have repeatedly used the inequality $\sqrt{\alpha - \epsilon} \leq \sigma_{\min}\pts{\X^\top} \leq \sigma_{\max}\pts{\X^\top} \leq \sqrt{\alpha + \epsilon}$.
The second inequality above uses the reverse triangle inequality, and the last inequality again uses $\frac{\alpha + \epsilon}{\alpha - \epsilon} \leq 2$ as long as $\epsilon \leq \frac{\alpha}{3}$.
Combining the upper bounds on $T_1$ and $T_2$ thus yields
\begin{align}\label{eq:intermediateprimal}
    \nnorm[2]{\normalize[2]{\barw} - \normalize[2]{\wmni}} \leq 4 \nnorm[2]{\bDelta}.
\end{align}

It remains to show that $\nnorm[2]{\bDelta} \leq \frac{C \epsilon_{\alpha}(\y)}{\alpha}$ for some universal constant $C$.
We will use the statement of Eq.~\eqref{eq:dualconvergencefinal} as a starting point to upper-bound $\nnorm[2]{\bDelta}$.
Recall that $\bDelta = \barbu - \bumni$ and $\barbu := \normalize[2]{\diag{\y}\barbq}$.
Then, applying the triangle inequality gives us
\begin{align*}
    \nnorm[2]{\bDelta} &\leq \nnorm[2]{\barbu - \frac{\y}{\sqrt{n}}} + \nnorm[2]{\bumni - \frac{\y}{\sqrt{n}}}.
\end{align*}

Consequently, it suffices to show that $\bumni$ is sufficiently close to the label vector $\y$; in other words, to upper bound $\nnorm[2]{\bumni - \frac{\y}{\sqrt{n}}}$.
We use a similar algebraic technique as in the preceding steps.
First, we write
\begin{align*}
    \nnorm[2]{\bumni - \frac{\y}{\sqrt{n}}} &= \nnorm[2]{\normalize[2]{\XXiy} - \normalize[2]{\alpha^{-1} \y}} \\
    &\leq \underbrace{\nnorm[2]{\frac{\XXiy }{\nnorm[2]{\alpha^{-1} \y}} - \frac{\alpha^{-1} \y}{\nnorm[2]{\alpha^{-1} \y}}}}_{T_1} + \underbrace{\nnorm[2]{\frac{\XXiy}{\nnorm[2]{\alpha^{-1} \y}} -  \frac{\XXiy }{\nnorm[2]{\XXiy}}}}_{T_2}.
\end{align*}

It remains to upper bound $T_1$ and $T_2$.
Beginning with $T_1$, we have
\begin{align*}
    T_1 &= \frac{\nnorm[2]{\pts{\XXi - \alpha^{-1} \I} \y}}{\nnorm[2]{\alpha^{-1} \y }} \\
    &\leq \nnorm[2]{\XXi} \cdot \frac{\nnorm[2]{\XX \y - \alpha \y}}{\nnorm[2]{\y}}\\
    &= \frac{\nnorm[2]{\XX \y - \alpha \y}}{\lambda_{\min}\bigl(\XX) \cdot \nnorm[2]{\y}}.
\end{align*}

Now, we note that $\lambda_{\min}\bigl(\XX\bigl) \geq \alpha - \epsilon \geq \frac{2\alpha}{3}$ because $\epsilon \leq \frac{\alpha}{3}$.
Consequently, we get
\begin{align*}
    T_1 \leq \frac{3\nnorm[2]{\XX \y - \alpha \y}}{2\alpha\cdot \nnorm[2]{\y}} =: \frac{3 \epsilon_{\alpha}\pts{\y}}{2 \alpha}.
\end{align*}

Proceeding to $T_2$, an identical series of arguments to the previous term $T_2$ yields
\begin{align*}
    T_2 &= \frac{\nnorm[2]{\XXiy}}{\nnorm[2]{\alpha^{-1} \y} \cdot \nnorm[2]{\XXiy}} \cdot \pts{\nnorm[2]{\XXiy} - \nnorm[2]{\alpha^{-1} \y}} \\
    &\leq \frac{\nnorm[2]{\XXiy}}{\nnorm[2]{\alpha^{-1} \y} \cdot \nnorm[2]{\XXiy}} \cdot \nnorm[2]{\pts{\XXi - \alpha^{-1} \I} \y} \\
    &= \frac{\nnorm[2]{\pts{\XXi - \alpha^{-1} \I} \y }}{\nnorm[2]{\alpha^{-1} \y}} =: T_1 \leq \frac{3 \epsilon_{\alpha}\pts{\y}}{2 \alpha}.
\end{align*}

Consequently, we have $\nnorm[2]{\bumni - \frac{\y}{\sqrt{n}}} \leq T_1 + T_2 \leq \frac{3\epsilon_{\alpha}\pts{\y}}{\alpha}$, and so we ultimately get $\nnorm[2]{\bDelta} \leq \nnorm[2]{\normalize[2]{\barbq}- \frac{1}{\sqrt{n}}} + \frac{3\epsilon_{\alpha}(\y)}{\alpha}$.
Combining this with Eq.~\eqref{eq:intermediateprimal} yields
\begin{align*}
 \nnorm[2]{\normalize[2]{\barw} - \normalize[2]{\wmni}} \leq 4 \nnorm[2]{\bDelta} \leq  4 \nnorm[2]{\normalize[2]{\barbq}- \frac{1}{\sqrt{n}}} + \frac{12\epsilon_{\alpha}(\y)}{\alpha},    
\end{align*}
completing the desired proof of our theorem.

\end{proof}

\subsection{Proof of Corollary~\ref{cor:convergencehighdims} (upper bound in effective dimensions)}\label{sec:corhighdimsproof}

In this section, we prove Corollary~\ref{cor:convergencehighdims}.

\begin{proof}
    We consider the setting of independent sub-Gaussian covariates described in Corollary~\ref{cor:convergencehighdims} and set $\alpha := \nnorm[1]{\blambda}$.
    It suffices to show the following with high probability:
    \begin{enumerate}
        \item $\frac{\nnorm[2]{\XX-\alpha\I}}{\alpha} \leq \frac{1}{3}$, and
        \item $\frac{C \epsilon_{\alpha}(\y)}{\alpha} \leq \max \left\{\sqrt{\frac{n}{d_2}}, \frac{n}{d_{\infty}}\right\}$.
    \end{enumerate}
    
    To prove both statements, we will use~\cite[Lemma 8]{hsu2021proliferation}, restated below.

    \begin{lemma}[\cite{hsu2021proliferation}]
        For any $\tau > 0$ and a universal constant $c > 0$, we have
        \begin{align*}
            \Prob\left[\nnorm[2]{\XX - \nnorm[1]{\blambda} \I} \geq \tau \right] \leq 2 \cdot 9^n \cdot \exp\left(-c \cdot \min \left\{ \frac{\tau^2}{v^2 \nnorm[2]{\blambda}^2}, \frac{\tau}{v \nnorm[\infty]{\blambda}}  \right\}\right).
        \end{align*}
    \end{lemma}
    
    To prove the first statement, we select $\tau = \frac{1}{3} \nnorm[1]{\blambda}$, so that the upper bound on the probability becomes $2 \cdot 9^n \cdot \exp\left(-\frac{c}{9} \cdot \min \left\{ \frac{d_2}{v^2},\frac{d_{\infty}}{v} \right\}\right)$.
    Because $d_2 \gg v^2 n$ and $d_{\infty} \gg v n$, there exists a large enough constant $C > 0$ such that $\min\left\{\frac{d_2}{v^2}, \frac{d_{\infty}}{v} \right\} \geq Cn$ and $C \cdot c > 9 \ln 9$.
    Therefore, we get
    $\frac{\nnorm[2]{\XX-\nnorm[1]{\blambda} \I}}{\nnorm[1]{\blambda}} \leq \frac{1}{3}$ with probability at least $1 - 2 \cdot \exp(- n (C \cdot c/9 - \ln 9))$.
    To prove the second statement, we instead select $\tau = C \cdot v \cdot \max (\nnorm[2]{\blambda}\sqrt{n}, \nnorm[\infty]{\blambda} n)$, where $C > 1$ is picked to be large enough so that $C \cdot c > \ln 9$.
    This ensures that $\min \left\{ \frac{\tau^2}{v^2 \nnorm[2]{\blambda}^2}, \frac{\tau}{v \nnorm[\infty]{\blambda}}\right\} \geq Cn$, and in turn that
    \begin{align*}
        \Prob\left[\nnorm[2]{\XX - \nnorm[1]{\blambda} \I} \geq C \cdot v \cdot \max (\nnorm[2]{\blambda}\sqrt{n}, \nnorm[\infty]{\blambda} n) \right] \leq 2 \cdot \exp(- n (C \cdot c - \ln 9)).
    \end{align*}
    Thus, for the choice $\alpha := \nnorm[1]{\blambda}$, we have, with probability at least $1 - 2e^{-c'n}$,
    \begin{align*}
        \frac{C \epsilon_{\alpha}(\y)}{\alpha} &= \frac{\nnorm[2]{\XX \y - \nnorm[1]{\blambda} \y}}{\nnorm[1]{\blambda} \nnorm[2]{\y}} \\
        &\leq \frac{\nnorm[2]{\XX-\nnorm[1]{\blambda}\I}}{\nnorm[1]{\blambda}} \\
        &\leq C \cdot v \cdot \max \left\{\sqrt{\frac{n}{d_2}}, \frac{n}{d_{\infty}}\right\}
    \end{align*} 
    which completes the proof of the second statement.
\end{proof}

\subsection{Proof of Proposition~\ref{prop:b_loss_func} (popular loss functions)}\label{sec:b_loss_proof}




In this section, we prove Proposition~\ref{prop:b_loss_func}.
\begin{proof}
    In \cite[Theorem 11]{ji2021characterizing}, it has already shown that $\ells{\expt}$, $\ells{\logt}$ and $\ells{\polyt}$ satisfy both Assumption~\ref{asm:b_loss_assump} and Lemma~\ref{lem:g_func}. In the following, we demonstrate their respective $\gf{\cdot}$ functions.
    

    \paragraph{Exponential loss:} Here $\ellf{z}=\expf{z}$, $\ellpf{z}=\expf{z}$, $\ellppf{z}=\expf{z}$, $\ellf[-1]{z}=\lnf{z}$, and $\ellpellnofz{z}=z$. Therefore, we have $\ell'(\ell^{-1}(z)) = z$, which directly gives $\frac{\ell'(\ell^{-1}(a\cdot z))}{\ell'(\ell^{-1}(a))} = z$ and yields the function $\gsf[\expt]{d} = d$.

    \paragraph{Logistic loss:} Here $\ellf{z}=\lnf{1+\expf{z}}$, $\ellpf{z}=\frac{\expf{z}}{1+\expf{z}}$, $\ellppf{z}=\frac{\expf{z}}{\pts{1+\expf{z}}^2}$, $\ellf[-1]{z}=\lnf{\expf{z}-1}$, and $\ellpellnofz{z}=1-\expf{-z}$. 
    Consequently, we have $\ell'(\ell^{-1}(z)) = 1 - \exp(-z)$ and so $\frac{\ell'(\ell^{-1}(a\cdot z))}{\ell'(\ell^{-1}(a))} = \frac{1 - \exp(-a\cdot z)}{1 - \exp(-a)}$.
    Applying l'Hospital's rule yields 
    \begin{align*}
    \lim_{a \to 0} \frac{\ell'(\ell^{-1}(a\cdot z))}{\ell'(\ell^{-1}(a))} &= \lim_{a \to 0} \frac{1 - \exp(-a\cdot z)}{1 - \exp(-a)} = \lim_{a \to 0} \frac{\exp(-a\cdot z) \cdot z}{\exp(-a)} = z.
    \end{align*}
    As a result, we get $\gsf[\logt]{d} = d$.

    \paragraph{Polynomial loss (degree $m>0$):} Here we use the continuation of the polynomial loss to $z > 0$ used in~\cite{ji2021characterizing,wang2021importance} to ensure convexity.
    \begin{align*}
        \ellf{z}=
        \begin{cases}
            \frac{1}{\pts{1-z}^m} &\; z\leq0\\
            2mz+\frac{1}{\pts{1+z}^m} &\; z>0\\
        \end{cases}
        \myquad[2]\ellpf{z}=
        \begin{cases}
            \frac{m}{\pts{1-z}^{m+1}} &\; z\leq0\\
            2m-\frac{m}{\pts{1+z}^{m+1}} &\; z>0\\
        \end{cases}
    \end{align*}
    \begin{align*}
        \ellppf{z}=
        \begin{cases}
            \frac{m\pts{m+1}}{\pts{1-z}^{m+2}} &\; z\leq0\\
            \frac{m\pts{m+1}}{\pts{1+z}^{m+2}} &\; z>0\\
        \end{cases}
    \end{align*}
    For $z\leq \ell(0)$, we have $\ellf[-1]{z}=1-z^{-1/m}$, and $\ellpellnofz{z}=mz^{\frac{m+1}{m}}$. 
    Hence, we have
    \begin{align*}
        \lim_{a \to 0} \frac{\ell'(\ell^{-1}(a\cdot z))}{\ell'(\ell^{-1}(a))} = \left(\frac{a\cdot z}{a}\right)^{\frac{m+1}{m}} = z^{\frac{m+1}{m}},
    \end{align*}
    and so we get $\gsf[\polyt]{d} = d^{\frac{m+1}{m}}$.

\end{proof}

\newpage

\section{Derivations for Lemma~\ref{lem:m_imp_bias_res} (generalizing primal-dual analysis to the multiclass setting)} \label{app_m_imp_bias}

In this section, we provide the derivations for Lemma~\ref{lem:m_imp_bias_res}, which generalizes the primal-dual analysis of~\cite{ji2021characterizing} to the multiclass setting.
In particular, we state several lemmas for the multiclass that are analogous to the lemmas in~\cite{ji2021characterizing} for the binary case.
We first introduce these analogous lemmas, and then we show Lemma~\ref{lem:m_imp_bias_res} follows as a direct result of them.
Note that we require these analogous lemmas because the generalized sum $\psi(\cdot)$ is slightly different in the multiclass setting for loss functions satisfying Assumption~\ref{asm:m_ova_loss_assump}, and completely different for the cross-entropy loss under Assumption~\ref{asm:m_ce_loss_assump}.
Some of these lemmas are direct extensions of those in~\cite{ji2021characterizing}, and so we do not provide proofs for these particular lemmas.

\begin{lemma} \label{lem:m_primal_dual_conv}
    Under Assumption~\ref{asm:m_ova_loss_assump}, (or~\ref{asm:m_ce_loss_assump}), for all $\Q\in\dom\psi^*$, if $\hetat\leq 1/\beta$, then the following results hold:
    \begin{enumerate}
        \item Dual convergence: for all $t\geq0$,
            \begin{align*}
                \Ff{\Qto}\leq \Ff{\Qt}\text{, and }\; \hetat\pts{\Ff{\Qto}-\Ff{\Q}}\leq \breg{\Q}{\Qt}-\breg{\Q}{\Qto}.
            \end{align*}
        As a result, for all $t>0$,
            \begin{align*}
                \Ff{\Qt}-\Ff{\Q}\leq\frac{\breg{\Q}{\Qo}-\breg{\Q}{\Qt}}{\sumjt\hetaj}\leq \frac{\breg{\Q}{\Qo}}{\sumjt\hetaj}
            \end{align*}
        \item Primal convergence: for all $t\geq0$,
            \begin{align*}
                \psif{\Pt}-\psif{\Pto}\geq \hetat\pts{\Ff{\Qt}+\Ff{\Qto}}=\frac{\hetat}{2}\pts{\nnorm{\tX^\top\C\Qt}^2+\nnorm{\tX^\top\C\Qto}^2}, 
            \end{align*}
        and thus if $\hetat$ is nonincreasing, we have
            \begin{align*}
                \psif{\Po}-\psif{\Pt} \geq \sumjt \hetaj\nnorm{\tX^\top\C\Qs{j}}^2-\frac{\heta_0}{2}\nnorm{\tX^\top\C\Qo}^2 +\frac{\hetat}{2}\nnorm{\tX^\top\C\Qt}^2
            \end{align*}
    \end{enumerate}
\end{lemma}

This lemma is analogous to and a direct application of~\cite[Theorem 1]{ji2021characterizing} in our notation, since the primal and dual setup is identical to the binary case and the generalized sum $\psi(\cdot)$ was verified to be $\beta$-smooth for loss functions satisfying either Assumption~\ref{asm:m_ova_loss_assump}, or~\ref{asm:m_ce_loss_assump}.

\begin{lemma} \label{lem:m_dual_imp_bias}
    Under Assumption \ref{asm:m_ova_loss_assump}, (or~\ref{asm:m_ce_loss_assump}) and Assumption~\ref{asm:b_lin_sep}, suppose $\hetat\leq 1/\beta$ is nonincreasing, and $\sumt\hetat=\infty$.
    \begin{enumerate}
        \item The set $\seq{\Q|\psicf{\Q}\leq0}$ is nonempty, compact and convex. Moreover, $\smin{\psicf{\Q}\leq 0} \Ff{\Q}>0$, and $\tX^\top\C\barQ$ is the same for all $\barQ\in\sargmin{\psicf{\Q}\leq 0} \Ff{\Q}$.
        \item For $\barQ\in\sargmin{\psicf{\Q}\leq 0} \Ff{\Q}$, and all $t$ with $\psif{\C\tZ\Wt}\leq0$ (which holds for all large enough $t$), we have
            \begin{align*}
                \nnorm{\tX^\top\C\Qt-\tX^\top\C\barQ}^2\leq\frac{2 \breg{\barQ}{\Qo}}{\sumjt\hetaj} \text{, and } \ip{\normalize{\Wt}}{\normalize{\tX^\top\C\barQ}} \geq 1-\frac{\deltaf{\Qo}{\barQ}}{\sumjt\heta_j},
            \end{align*}
        where 
        \begin{align*}
        \deltaf{\W_0}{\barQ}:=\frac{\psif{\Qo}+\heta_0 \Ff{\Qo}+\nnorm{\W_0}\nnorm{\tX^\top\C\barQ}}{2\Ff{\barQ}}
        \end{align*}
         is a constant, depending only on $\W_0$ and $\barQ$. In particular, it holds that the implicit bias is
        \begin{align}
            \barW:=\ulim{\tinf{t}}\normalize{\Wt}=\normalize{\tX^\top\C\barQ},
        \end{align}
        where $\barQ=\vct{\barbq_1^\top}{\barbq_K^\top}$ for $\barbq_k\in\R^n$ for all $k\in[K]$.
    \end{enumerate}
\end{lemma}

This lemma is analogous to \cite[Theorem 5]{ji2021characterizing}, and almost all the steps in its proof are a direct extension of their proof.
We only reproduce the parts of the proof that need to be done from scratch.
We begin with the following lemma, which shows the feasibility of the convex conjugate constraint $\psi^*(\Q) \leq 0$.
This admits a different proof from the binary case due to the differing formulations of the generalized sum $\psi(\cdot)$ in the multiclass case.

\begin{lemma} \label{lem:m_ccsz_cond}
    Under Assumption~\ref{asm:m_ova_loss_assump}, (or~\ref{asm:m_ce_loss_assump}), for $\bXi \in\R^{Kn}$ such that $\psif{\bXi}\leq0$, it holds that $\psicf{\gpsif{\bXi}}\leq0$. This lemma is analogous to \cite[Lemma 6]{ji2021characterizing}.
\end{lemma}

Before we prove Lemma~\ref{lem:m_ccsz_cond}, we introduce Lemma~\ref{lem:g_sup_add} as an auxiliary lemma. 

\begin{lemma} \label{lem:g_sup_add}
    For loss functions under Assumption~\ref{asm:b_loss_assump}, we have $\sigmaf{s}:=\ellpellnofz{s}\ellf[-1]{s}$ which is a super-additive function on $\pts{0, \ellf{0}}$.
\end{lemma}

\begin{proof}
    The proof follows the proof of~\cite[Lemma 6]{ji2021characterizing}. Note that by Assumption~\ref{asm:b_loss_assump}, we have
    \begin{align*}
        \ulim{\tzero{s}}\sigmaf{s}=\ellpellnofz{s}\ellf[-1]{s}=0 \quad \text{and} \quad \sigmaf{s}/s \text{ is increasing on } \pts{0, \ellf{0}},
    \end{align*}
    by letting $s=\ellf{z}$ and $z = \ellf[-1]{s}$. Next, for some $t\in\R$ and $0< t\leq1$, we assume $s_1=tx$ and $s_2=x$ for some $x\in\R$ and $0<x\leq\ell(0)$. Then we have
    \begin{align*}
        \frac{\sigmaf{s_1}}{s_1}\leq\frac{\sigmaf{s_2}}{s_2}\, \iff \frac{\sigmaf{tx}}{tx}\leq\frac{\sigmaf{x}}{x}\, \iff \sigmaf{tx}\leq t\sigmaf{x}.
    \end{align*}
    
    For $a,b>0$ and $a+b<\ell(0)$,

    \begin{align*}
        \sigmaf{a}+\sigmaf{b}&=\sigmaf{\frac{a}{a+b}\times\pts{a+b}}+\sigmaf{\frac{b}{a+b}}\times\pts{a+b}\\
        &\leq\frac{a}{a+b}\sigmaf{a+b} + \frac{b}{a+b}\sigmaf{a+b}\\
        &=\sigmaf{a+b}.
    \end{align*}

\end{proof}

Armed with Lemma~\ref{lem:g_sup_add} we can prove Lemma~\ref{lem:m_ccsz_cond}.

\begin{proof} (Proof of Lemma~\ref{lem:m_ccsz_cond})
    Recalling the definition of $\psi$ and its convex conjugate $\psi^*$, we have
    \begin{align*}
        \psicf{\gpsif{\bXi}}=\ip{\bXi}{\gpsif{\bXi}}-\psif{\bXi} =\sumn\sumk \frac{\xiki\frac{\partial\multilossf{\seqr{\xiki}{k=1}{K}}}{\partial\xiki}}{\ellpf{\psif{\bXi}}}-\psif{\bXi}.
    \end{align*}
    Multiplying $\ellpf{\psif{\bXi}}$ on both sides, we get
    \begin{align} \label{eq:m_cc_eq}
        \ellpf{\psif{\bXi}}\psicf{\gpsif{\bXi}} &= \sumn\sumk \xiki\frac{\partial\multilossf{\seqr{\xiki}{k=1}{K}}}{\partial\xiki}-\ellpf{\psif{\bXi}}\psif{\bXi}\nonumber\\
        &=\sumn\sumk \xiki\frac{\partial\multilossf{\seqr{\xiki}{k=1}{K}}}{\partial\xiki}\nonumber\\
        &\myquad[2]-\ellpellnoft{\sumn\multilossf{\seqr{\xiki}{k=1}{K}}}\ell^{-1}\Biggl(\sumn\multilossf{\seqr{\xiki}{k=1}{K}}\Biggl)\nonumber\\
        &=\sumn\sumk \xiki\frac{\partial\multilossf{\seqr{\xiki}{k=1}{K}}}{\partial\xiki} - \sigma\Biggl(\sumn\multilossf{\seqr{\xiki}{k=1}{K}}\Biggl),
    \end{align}
    since we substitute $\sigmaf{s}=\ellpellnofz{s}\ellf[-1]{s}$ by Lemma~\ref{lem:g_sup_add}. 
    Hereafter, we handle the situations under Assumption~\ref{asm:m_ova_loss_assump} and~\ref{asm:m_ce_loss_assump} separately (as the generalized sum $\psi$ is distinct in each case).
    %
    Under Assumption~\ref{asm:m_ova_loss_assump}, we have $\multilossf{\seqr{\xiki}{k=1}{K}}=\sumk\ellf{\xiki}$, and Eq.~\eqref{eq:m_cc_eq} becomes
    \begin{align*}
        \ellpf{\psif{\bXi}}\psicf{\gpsif{\bXi}}&=\sumn\sumk \xiki\ellpf{\xiki}-\sigma\Biggl(\sumn\sumk\ellf{\xiki}\Biggl)\\
        &=\sumn\sumk \sigmaf{\ellf{\xiki}}- \sigma\Biggl(\sumn\sumk\ellf{\xiki}\Biggl)\leq0,\\
    \end{align*}
    where the last inequality uses the super-additivity property of $\sigmaf{s}$ on $\pts{0, \ellf{0}}$ (Lemma~\ref{lem:g_sup_add}, together with the assumption $\psif{\bXi}\leq0$ or equivalently $\sumn\sumk\ellf{\xiki}\leq\ellf{0}$). 
    Finally, since $\ell' > 0$, we have $\psicf{\gpsif{\bXi}} \leq0$.

    Under Assumption~\ref{asm:m_ce_loss_assump}, we have 
    \begin{align*}
        \multilossf{\seqr{\xiki}{k=1}{K}}=\ln\Biggl(1+\sumkny \expf{\cyi\xiyi-\cki\xiki}\Biggl)=\lnf{1+\deltai},
    \end{align*}
    where we denote $\deltai:=\sumkny \expf{\cyi\xiyi-\cki\xiki}$, $\deltaki:=\expf{\cyi\xiyi-\cki\xiki}$. Direct calculations verify that
    \begin{align*}
        \frac{\partial\multilossf{\seqr{\xiki}{k=1}{K}}}{\partial\xiki}=
        \begin{cases}
            \frac{\cyi\deltai}{1+\deltai} &\; k=y_i\\
            \frac{-\cki\deltaki}{1+\deltai} &\; k\neq y_i.\\
        \end{cases}
    \end{align*}
    Hence, Eq.~\eqref{eq:m_cc_eq} becomes
    \begin{align}
         \ellpf{\psif{\bXi}}\psicf{\gpsif{\bXi}} &= \underbrace{\sumn \frac{\cyi\xiyi\deltai + \sumkny -\cki\xiki\deltaki}{1+\deltai}}_{T}-\sigma\Biggl(\sumn\multilossf{\seqr{\xiki}{k=1}{K}}\Biggl).\label{eq:m_cc_ce_eq}
    \end{align}
    
    Next, we show $T$ is upper bounded by $\sumn\sigma\Bigl(\multilossf{\seqr{\xiki}{k=1}{K}}\Bigl)$ in a series of calculations below:
    \begin{align*}
        T &= \sumn \frac{\cyi\xiyi\deltai + \sumkny -\cki\xiki\deltaki}{1+\deltai}\\
        &= \sumn \frac{\sumkny (\cyi\xiyi - \cki\xiki)\deltaki}{1+\deltai}\\
        &= \sumn \frac{\sumkny \lnf{\deltaki}\deltaki}{1+\deltai}\\
        &\leq \sumn \frac{\sumkny\lnf{\deltai}\deltaki}{1+\deltai}\\
        &= \sumn \frac{\deltai}{1+\deltai}\lnf{\deltai}\\
        &= \sumn \frac{\expf{\multilossf{\seqr{\xiki}{k=1}{K}}}-1}{\expf{\multilossf{\seqr{\xiki}{k=1}{K}}}} \Bigl(\ln\Bigl(\exp\Bigl(\multilossf{\seqr{\xiki}{k=1}{K}}\Bigl)-1\Bigl)\Bigl)\\
        &= \sumn \ellpellnofzo{\multilossf{\seqr{\xiki}{k=1}{K}}}\ell^{-1}\Bigl(\multilossf{\seqr{\xiki}{k=1}{K}}\Bigl)=\sumn\sigma\Bigl(\multilossf{\seqr{\xiki}{k=1}{K}}\Bigl).
    \end{align*}
    Above, the inequality holds because $\ln(z)$ is an increasing function. We also use the property that $\ellpellnofz{z} = \frac{\expf{z}-1}{\expf{z}}$ and $\ellf[-1]{z}=\lnf{\expf{z}-1}$ for logistic loss in the last equality. Proceeding from Eq.~\eqref{eq:m_cc_ce_eq}, we then get
    \begin{align*}
        \ellpf{\psif{\bXi}}\psicf{\gpsif{\bXi}} \leq \sumn\sigma\Bigl(\multilossf{\seqr{\xiki}{k=1}{K}}\Bigl) - \sigma\Bigl(\sumn\multilossf{\seqr{\xiki}{k=1}{K}}\Bigl)\leq0,
    \end{align*}
 where the last inequality uses the super-additivity property of $\sigmaf{s}$ on $\pts{0, \ellf{0}}$ (Lemma~\ref{lem:g_sup_add}, together with the assumption $\psif{\bXi}\leq0$ or equivalently $\sumn\multilossf{\seqr{\xiki}{k=1}{K}}\leq\ellf{0}$). Since $\ell' > 0$, we have $\psicf{\gpsif{\bXi}} \leq0$. This completes the proof for both types of losses.
\end{proof}

Finally, one key step that is utilized in the proofs of the binary analogs Lemma~\ref{lem:m_primal_dual_conv} and~\ref{lem:m_dual_imp_bias} (specifically, the proof of~\cite[Lemma 4]{ji2021characterizing} and~\cite[Theorem 5 part 1]{ji2021characterizing}) is the statement that $\Q \in \dom \psi^* \implies \Q = \nabla \psi(\bP)$ for some $\bP$; or, equivalently, $\bP \in \partial \psicf{\Q} \implies \Q = \nabla \psi(\bP)$.
This fact appears from~\cite[Theorems 23.5]{rockafellar1970convex} along with the reverse implication and implicitly assumes the joint convexity of $\psi$, but we show below that the forward implication continues to hold under individual convexity.
(Note that the reverse implications no longer hold under individual convexity, but are not required for these proofs.)

\begin{lemma} \label{lem:m_psi_ind_conv}
For any individually convex and differentiable function $\psi(\cdot)$ and its convex conjugate $\psicf{\cdot}$, we have that $\Q \in \dom \psi^* \implies \Q = \nabla \psi(\bP^*)$ for some $\bP^*$ that achieves $\sup_{\bP\in\R^{Kn}} \ip{\bP}{\Q}-\psif{\bP}$.
Equivalently, $\bP^* \in \partial \psicf{\Q} \implies \Q = \nabla \psi(\bP^*)$.
\end{lemma}

\begin{proof}
Recall the definition of the convex conjugate $\psicf{\Q}=\usup{\bP\in\R^{Kn}}\ip{\bP}{\Q}-\psif{\bP}$.
Because $\psi(\bP)$ is individually convex, the function $\ip{\bP}{\Q}-\psif{\bP}$ is individually concave.
Consider any $\bP^*$ that achieves the supremum of this function over $\bP$.
Because of the property of individual concavity, it is \emph{necessary} (but not sufficient) for $\bP^*$ to satisfy the first-order condition $\Q = \nabla \psi(\bP^*)$.
Moreover, for any $\bP^*$ that achieves the supremum of this function over $\bP$, we have $\psi^*(\Q) = \ip{\bP^*}{\Q}-\psif{\bP^*}$.
Taking the subdifferential with respect to $\Q$ directly gives $\bP^* \in \partial \psicf{\Q}$.
This completes the proof.
\end{proof}

Now that we have established Lemmas~\ref{lem:m_primal_dual_conv},~\ref{lem:m_dual_imp_bias},~\ref{lem:m_ccsz_cond} and~\ref{lem:m_psi_ind_conv}, Lemma~\ref{lem:m_imp_bias_res} now directly follows as a result.

\newpage

\section{Proofs of multiclass results} \label{app_m_proof}

In this section we present the proofs of all of our results for the multiclass case. In order to prove Theorem~\ref{thm:m_sen_upperbound} and Theorem~\ref{thm:m_exp_ce_imp_bias}, we introduce their respective auxiliary convex programs in multiclass which is analogous to the binary case.

\subsection{Proof of Theorem~\ref{thm:m_sen_upperbound} (approximate equivalence to one-vs-all MNI upper bound)} \label{sec:m_sen_upperbound_proof}

Similar to the strategy in binary case, before we prove Theorem~\ref{thm:m_sen_upperbound}, we introduce an auxiliary convex program that will ultimately provide a simpler characterization of the solution in~\eqref{eq:m_barQ_def}.

\begin{lemma} \label{lem:m_ova_imp_bias_relax}
    Under Assumptions~\ref{asm:m_ova_loss_assump} and~\ref{asm:b_lin_sep}, any solution to the auxiliary convex program
    \begin{gather} \label{eq:m_ova_imp_bias_relax_eq}
        \barQ \in \uargmin{\Q\in\R^{Kn}}\usub{\Ff{\Q}}{\frac{1}{2}\Q^\top\CXXC\Q}\\
        \begin{aligned} \nonumber
            \myquad[5]\text{subject to} \myquad[7] -\qki  &< 0 \myquad[2] \text{for all}\; i\in[n] \;\text{and}\; k\in[K] ,\\
            \text{and} \myquad[2] 1-\sumn\sumk \gf[-1]{\qki} &\leq 0.
        \end{aligned}
    \end{gather}
    is also an optimal solution to the original convex program~\eqref{eq:m_barQ_def}. Above, $\gf{\cdot}$ is a convex function depending on the loss function, defined in Lemma~\ref{lem:g_func}.
\end{lemma}

In order to prove Lemma~\ref{lem:m_ova_imp_bias_relax}, we first show that the new equality constraints are sufficient to imply the original equality constraint. Second, we show that any solution to the original program also satisfies the new equality constraint; therefore, the solution sets of both programs coincide. The following lemma demonstrates the first part.

\begin{lemma}\label{lem:m_ova_constraint_equal}
    Under Assumption~\ref{asm:m_ova_loss_assump}, for any $\Q\in\R^{Kn}$ such that $q_{k,i} = \gf{z_{k,i}} = \ulim{a \to 0} \frac{\ellpellnofz{a\cdot z_{k,i}}}{\ellpellnofz{a}} > 0$, for all $i \in [n]$ and $k\in[K]$ with $z_{k,i} \in (0, 1]$ and $\sum_{i=1}^n\sumk z_{k,i} = 1$, it implies $\psif[*]{\Q} = 0$.
\end{lemma}
\begin{proof}
    Analogous to Lemma~\ref{lem:b_constraint_equal}, we once again apply the convex analysis in the astral space introduced in~\cite{dudik2022convex}. We show that $\Q$ is a subdifferential of the astral extension of $\psi$ for some $\overline{\bP}\in\overline{\R^{Kn}}$, and then the astral convex conjugate (which is equivalent to the original convex conjugate) is equal to zero. We start with writing $\Q$ in the astral format. Since $q_{k,i}$ is finite and continuous in $a$, we can take the limit inside the function, obtaining
    \begin{align}
        q_{k,i} = \ulim{a \to 0} \frac{\ellpellnofz{a\cdot z_{k,i}}}{\ellpellnofz{a}} =  \frac{\ellpellnofzo{\ulim{a \to 0}a\cdot z_{k,i}}}{\ellpellnofzo{\ulim{a \to 0} a}} = \frac{\overline{\ell}'\Bigl(\overline{\ell}^{-1}\Bigl(\ulim{a \to 0}a\cdot z_{k,i}\Bigl)\Bigl)}{\overline{\ell}'\Bigl(\overline{\ell}^{-1}\Bigl(\ulim{a \to 0} a\Bigl)\Bigl)},\label{eq:m_ova_q_astral}
    \end{align}
    where in the last equality, we replace the original functions with their astral extensions. Next, we define an astral point $\overline{\bP}\in\overline{\R^{Kn}}$ such that $\overline{p}_{k,i} = \ulim{a\to 0}\; \overline{\ell}^{-1}\pts{a\cdot z_{k,i}} =  \overline{\ell}^{-1}\Bigl(\ulim{a\to 0}\;a\cdot z_{k,i}\Bigl)$ for all $i\in[n]$ and $k\in[K]$. This also implies $\sumn\sumk \overline{\ell}(\overline{p}_{k,i}) = \ulim{a\to 0} a$. Substituting these values in Eq.~\eqref{eq:m_ova_q_astral}, we can write $q_{k,i}$ as
    \begin{align*}
        q_{k,i} = \frac{\overline{\ell}'(\overline{p}_{k,i})}{\overline{\ell}'\Bigl(\overline{\ell}^{-1}\Bigl(\sumn \overline{\ell}\pts{\overline{p}_{k,i}}\Bigl)\Bigl)}
    \end{align*}
    for all $i\in[n]$ and $k\in[K]$.
    
    On the other hand, according to Lemma~\ref{lem:g_func}, we can also define $q_{k,i}$ in the limit of a different co-convergent sequence as
    \begin{align}\label{eq:m_ova_q_astral2}
        q_{k,i} = \ulim{a \to 0} \;\frac{\ell^{-1}\pts{a}\cdot z_{k,i}}{\ell^{-1}\pts{a\cdot z_{k,i}}} = \frac{\ell^{-1}\Bigl( \ulim{a \to 0} a\Bigl)\cdot z_{k,i}}{\ell^{-1}\Bigl( \ulim{a \to 0} a \cdot z_{k,i}\Bigl)} = \frac{\overline{\ell}^{-1}\Bigl( \ulim{a \to 0} a\Bigl)\cdot z_{k,i}}{\overline{\ell}^{-1}\Bigl(\ulim{a \to 0} a \cdot z_{k,i}\Bigl)} = \frac{\overline{\ell}^{-1}\pts{ \sumn\overline{\ell}(\overline{p}_{k,i})}\cdot z_{k,i}}{\overline{p}_{k,i}}
    \end{align}
    for all $i \in[n]$ and $k\in[K]$.
    Next, we show that $\Q$ is a subdifferential of $\overline{\psi}\pts{\overline{\bP}}$. By the definition of $\overline{\psi}:\overline{\R^{Kn}}\to\overline{\R}$ and for any $\bP\in\overline{\R^{Kn}}$, we have
    \begin{align*}
        \overline{\psi}\pts{\bP} = \overline{\ell}^{-1}\Biggl(\sumn\sumk\overline{\ell}\pts{p_{k,i}}\Biggl) \text{ and } \frac{\partial}{\partial p_{k,i}} \overline{\psi}\pts{\bP} = \frac{\overline{\ell}'\pts{p_{k,i}}}{\overline{\ell}'\Bigl(\overline{\ell}^{-1}\Bigl(\sumn\sumk\overline{\ell}\pts{p_{k,i}}\Bigl)\Bigl)},
    \end{align*}
    for all $i\in[n]$ and $k\in[K]$. Therefore, according to Eq.~\eqref{eq:m_ova_q_astral}, it implies that $\Q$ is in the subdifferential of $\overline{\psi}\pts{\overline{\bP}}$ such that $\Q \in \partial \overline{\psi}\pts{\overline{\bP}}$. Hence, we can further apply the property of Fenchel–Young inequality in the convex conjugate, obtaining
    \begin{align*}
        \overline{\psi}^*\pts{\Q} = \usup{\bP\in\overline{\R^{Kn}}}\ip{\bP}{\Q}-\overline{\psi}\pts{\bP} = \ip{\overline{\bP}}{\Q}-\overline{\psi}\pts{\overline{\bP}}.
    \end{align*}
    As a result, by Eq.~\eqref{eq:m_ova_q_astral2} and the definition of $\overline{\psi}\pts{\cdot}$, we can write
    \begin{align*}
        \overline{\psi}^*\pts{\Q} &= \ip{\overline{\bP}}{\Q}-\overline{\psi}\pts{\overline{\bP}}\\
        &= \sumn \sumk\overline{p}_{k,i} \cdot \frac{\overline{\ell}^{-1}\pts{ \sumn\sumk\overline{\ell}(\overline{p}_{k,i})}\cdot z_{k,i}}{\overline{p}_{k,i}} - \overline{\ell}^{-1}\Biggl(\sumn\sumk\overline{\ell}(\overline{p}_{k,i})\Biggl) = 0.
    \end{align*}
    Finally, by~\cite[Proposition 8.5]{dudik2022convex}, we have $\psi^*\pts{\Q} = \overline{\psi}^*\pts{\Q} = 0$. This completes the proof of the lemma.
    
\end{proof}

Next, we show in the following lemma that $\barQ=\lim_{t \to \infty} \Q_t$ under Assumption~\ref{asm:b_lin_sep}.

\begin{lemma} \label{lem:m_ova_qt_eq_bq}
    Under Assumption~\ref{asm:m_ova_loss_assump} and~\ref{asm:b_lin_sep}, the gradient descent dual variable $\Q_t$ converges to $\barQ$ when $t\rightarrow\infty$. It gives $\barq_{k,i} = \lim_{t \to \infty} q_{t,k,i} = g(z_{k,i}) > 0$ for all $i \in [n]$ and $k\in[K]$ with some $z_{k,i} \in (0, 1]$ and $\sumn\sumk z_{k,i} = 1$.
\end{lemma}
\begin{proof}
    Analogous to Lemma~\ref{lem:b_qt_eq_bq}, according to Lemma~\ref{lem:m_dual_imp_bias}, we have $\limt \tX^\top\C\Q_t=\tX^\top\C\barQ$, and $\tX^\top\C\barQ$ is the same for all $\barQ\in\uargmin{\psif[*]{\Q}\leq 0}\Ff{\Q}$. Based on the definition of $q_{t,k,i}$ such that
    \begin{align}\label{eq:m_ova_q_def}
        q_{t,k,i} = \frac{\partial}{\partial p_{t,k,i}}\psif{\bP_t} = \frac{\ell'\pts{p_{t,k,i}}}{\ell'\Bigl(\ell^{-1}\Bigl(\sumn\sumk \ell\pts{p_{t,k,i}}\Bigl)\Bigl)} = \frac{\ell'\pts{p_{t,k,i}}}{\ell'\pts{\psif{\bP_t}}},
    \end{align}
    and considering $\ellpf{\cdot}$ is an increasing function with $p_{t,k,i} \leq \psif{\bP_t}$ for all $i\in[n]$ and $k\in[K]$ and $t\geq0$, it follows that $0<q_{t,k,i}\leq1$; hence, we can conclude that $\limt \tX^\top\C\Q_t=\tX^\top\C\limt \Q_t=\tX^\top\C\barQ$. Next, by Assumption~\ref{asm:b_lin_sep}, $\X^\top$ has full column rank, and $\X\X^\top\succ \zero$. Therefore, we can multiply the pseudo-inverse of $\tX^\top\C$ on both sides of $\tX^\top\C\limt \Q_t=\tX^\top\C\barQ$, which implies that $\limt \Q_t=\barQ$. Finally, by the  definition of $q_{t,k,i}$ in Eq.~\eqref{eq:m_ova_q_def} and the primal convergence in Lemma~\ref{lem:m_primal_dual_conv} such that $\limt \sumn\sumk \ellf{p_{t,k,i}} = 0$, we have
    \begin{align*}
        \limt q_{t,k,i} = \limt \frac{\ellpellnofzo{\ellf{p_{t,k,i}}}}{\ellpellnofzo{\sumn\sumk \ellf{p_{t,k,i}}}} = \ulim{a \to 0} \frac{\ellpellnofz{a\cdot z_{k,i}}}{\ellpellnofz{a}} = g(z_{k,i}),
    \end{align*}
    where we let $a = \sumn\sumk \ellf{p_{t,k,i}}$, $\ellf{p_{t,k,i}} = a \cdot z_{k,i}$, and $z_{k,i} = \limt \frac{\ellf{p_{t,k,i}}}{\sumn\sumk \ellf{p_{t,k,i}}}$. Finally, Assumption~\ref{asm:b_loss_assump} guarantees that $\barq_{k,i} = g(z_{k,i}) > 0$. This completes the proof of the lemma.
\end{proof}

Armed with Lemma~\ref{lem:m_ova_constraint_equal} and~\ref{lem:m_ova_qt_eq_bq}, we can prove Lemma~\ref{lem:m_ova_imp_bias_relax}.

\begin{proof} (of Lemma~\ref{lem:m_ova_imp_bias_relax})
    We start with the original convex program in~\eqref{eq:m_barQ_def}
    \begin{align*}
        \barQ \in \uargmin{\psif[*]{\Q}\leq 0}\Ff{\Q}.
    \end{align*}
    By complementary slackness in KKT conditions, if the constraint is inactive such that $\psif[*]{\Q}< 0$, we get an invalid solution $\barQ = \zero$. Therefore, the constraint is active and $\barQ$ satisfies $\psif[*]{\Q} = 0$. 
    In other words, we can write
    \begin{align}\label{eq:m_ova_barq_def_tightened_1}
        \barQ \in \uargmin{\psif[*]{\Q} = 0}\Ff{\Q}.
    \end{align}
    Now, Lemma~\ref{lem:m_ova_qt_eq_bq} directly implies that $\barQ$ must satisfy $\sumn\sumk g^{-1}(\barq_{k,i}) = 1$ and $\barq_{k,i} > 0$ for all $i\in[n]$ and $k\in[K]$. This means that we can further tighten~\eqref{eq:m_ova_barq_def_tightened_1} to obtain
        \begin{gather} \label{eq:m_ova_barq_def_tightened_2}
        \barQ \in \uargmin{\Q\in\R^{Kn}} \Ff{\Q} \\
        \begin{aligned}
        \myquad[4]\text{subject to} \myquad[7] \psif[*]{\Q} &= 0,\nonumber \\
        -q_{k,i} &< 0 \myquad[2]\text{for all } i\in [n] \text{ and } k\in[K],\nonumber\\ 
        \text{and} \myquad[1] 1-\sumn\sumk \gf[-1]{q_{k,i}}  &= 0\nonumber.
        \end{aligned}
    \end{gather}
    Next, Lemma~\ref{lem:m_ova_constraint_equal} tells us that $1-\sumn\sumk \gf[-1]{q_i}  = 0 \implies \psif[*]{\Q} = 0$, meaning that the constraint $\psif[*]{\Q} = 0$ is redundant and can simply be omitted, leading to the simplified program
        \begin{gather} \label{eq:m_ova_barq_def_tightened_3}
        \barQ \in \uargmin{\Q\in\R^{Kn}} \Ff{\Q} \\
        \begin{aligned}
        \myquad[4]\text{subject to} \myquad[8]
        -q_{k,i} &< 0 \myquad[2]\text{for all } i\in [n] \text{ and } k\in[K],\nonumber\\ 
        \text{and} \myquad[1] 1-\sumn\sumk \gf[-1]{q_{k,i}}  &= 0\nonumber.
        \end{aligned}
    \end{gather}
    The final step is to derive an auxiliary convex program
         \begin{gather} \label{eq:m_ova_imp_bias_relax_eq_repeat}
        \Q^\star \in \uargmin{\Q\in\R^{Kn}}\usub{\Ff{\Q}}{\frac{1}{2}\Q^\top\CXXC\Q} \\
        \begin{aligned} \nonumber
            \myquad[6]\text{subject to} \myquad[7] -q_{k,i}                                        &< 0 \myquad[2] \text{for all}\; i \in [n] \text{ and } k\in[K],\\
            \text{and} \myquad[1] 1-\sumn\sumk \gf[-1]{q_{k,i}}  &\leq 0.
        \end{aligned}
    \end{gather}
    Note that in the above, we have relaxed the equality constraint $1 - \sumn\sumk \gf[-1]{q_{k,i}} = 0$ to an inequality constraint, $1 - \sumn\sumk \gf[-1]{q_{k,i}} \leq 0$.

    To complete the proof, we need to show that any optimal solution to~\eqref{eq:m_ova_imp_bias_relax_eq_repeat} satisfies $\sumn\sumk \gf[-1]{q_{k,i}} = 1$. From~\eqref{eq:m_ova_barq_def_tightened_3}, this directly implies that the set of optima of~\eqref{eq:m_barQ_def} and~\eqref{eq:m_ova_imp_bias_relax_eq_repeat} are identical. We now show this final step. It is necessary and sufficient for any optimal solution $\Q^\star$ to the auxiliary convex program~\eqref{eq:m_ova_imp_bias_relax_eq_repeat} to satisfy its KKT conditions, listed below:
    \begin{subequations}\label{eq:m_ova_kkt_relaxed}
    \begin{align}
        -\qki                                    &<      0 \myquad[2] \text{for all } i\in[n] \text{ and }k\in[K],\label{m_ova_relaxed_primal_1}\\
        1-\sumn\sumk \gf[-1]{\qki}                    &\leq   0,\label{m_ova_relaxed_primal_2}\\
        \lambda_{k,i}                               &\geq   0 \myquad[2] \text{for all } i\in[n] \text{ and }k\in[K],\label{m_ova_relaxed_dual_1}\\
        \mu                                     &\geq   0,\label{m_ova_relaxed_dual_2}\\
        -\lambda_{k,i} \qki                          &=      0 \myquad[2] \text{for all } i\in[n]  \text{ and }k\in[K],\label{m_ova_relaxed_comp_1}\\
        \mu\Biggl(1-\sumn\sumk \gf[-1]{\qki}\Biggl)           &=      0,\label{m_ova_relaxed_comp_2}\\
        \cXXc\q_k-\blambda_k-\mu \gpf[-1]{\q_k}  &=      \zero \myquad[2] \text{for all } k\in[K].\label{m_ova_relaxed_station_1}
    \end{align}
    \end{subequations}
    First, we claim that any optimal solution $\Q^\star$ needs to satisfy $1-\sumn\sumk g^{-1}\bigl(q_{k,i}^\star\bigl) = 0$.
    This follows because we need to set $\mu > 0$ for a valid solution; together with Eq.~\eqref{m_ova_relaxed_comp_2} this implies that we need $1-\sumn\sumk g^{-1}\bigl(q_{k,i}^\star\bigl) = 0$. To see why we need to set $\mu > 0$, consider the alternative choice $\mu = 0$ for all $k\in[K]$.
    Note that Equations~\eqref{m_ova_relaxed_primal_1} and~\eqref{m_ova_relaxed_comp_1} together also require $\blambda_k = \zero$. Eq.~\eqref{m_ova_relaxed_station_1} would then become 
    \begin{align*}
        \XX \diagt{\ick} \q_k^\star = \zero \iff \diagt{\ick}  \q_k^\star = \zero \iff  \q_k^\star = \zero,
    \end{align*}
    where the first iff statement follows because we have assumed that $\XX \succ \zero$. However, this $\q_k^\star$ is not a valid solution as it violates Eq.~\eqref{m_ova_relaxed_primal_1}. Hence, we can conclude both $\barQ$ and $\Q^\star$ satisfy $1-\sumn\sumk \gf[-1]{q_{k,i}} = 0$, and $\Q^\star = \barQ$. This completes the proof of the lemma.
\end{proof}

With the auxiliary convex program, we can now prove Theorem~\ref{thm:m_sen_upperbound}. In the proof, we show that we have the exact characteristic equations in $\barbq_k$ for each $k$. Therefore, the primal and dual rate for each $k\in[K]$ is the same as Theorem~\ref{thm:b_sen_upperbound}.
\begin{proof} (Proof of Theorem~\ref{thm:m_sen_upperbound}.)
    Our proof starts with the auxiliary convex program~\eqref{eq:m_ova_imp_bias_relax_eq} and identifies a necessary set of characteristic equations that the optimal solution $\barQ$ needs to satisfy. The KKT conditions for this convex program are given in Eq.~\eqref{eq:m_ova_kkt_relaxed}.
    Lemma~\ref{lem:m_ova_imp_bias_relax} postulates that any optimal solution must satisfy $\blambda_k = \zero$ and $\mu > 0$ for all $k\in[K]$; therefore, it is necessary for $\barQ$ to satisfy the following characteristic equations for each $k$:
    
    \begin{subequations}\label{eq:m_ova_sen_charc}
    \begin{align}
        \cXXc\barbq_k                 &=      \mu \gpf[-1]{\barbq_k},      \label{m_sen_charc_1}\\
        \mu                         &>      0,   \label{m_sen_charc_2}\\
        \sumn\sumk \gf[-1]{\barqki}      &=      1.  \label{m_sen_charc_3}
    \end{align}
    \end{subequations}
    
    It is easy to see that for each value of $k$, the characteristic equations in Equation~\eqref{eq:m_ova_sen_charc} are identical to the characteristic equations for the binary case~\eqref{b_sen_charc}.
    Therefore, the rates of convergence of the dual and primal solutions are identical to the binary case for every value of $k$.
    This completes the proof.
\end{proof}

\subsection{Proof of Theorem~\ref{thm:m_exp_ce_imp_bias} (exact equivalence to simplex MNI for cross-entropy loss under Assumption~\ref{asm:m_ce_loss_assump})} \label{sec:m_ce_imp_proof}
Before we prove Theorem~\ref{thm:m_exp_ce_imp_bias} for cross-entropy loss under Assumption~\ref{asm:m_ce_loss_assump}, we state and prove two lemmas that we need to analyze the constraint $\psicf{\Q}\leq0$. Note that since the $\psi$ function for cross-entropy loss is different from other multiclass losses, we apply a different proof technique for the proof of this part. First, we utilize the following lemma to analyze the domain of $\psicf{\Q}$ under Assumption~\ref{asm:m_ce_loss_assump}.

\begin{lemma} \label{lem:m_ce_q_dom}
    Under Assumption~\ref{asm:m_ce_loss_assump}, for any $\Q=\vctt{\q_1^\top}{\q_K^\top}\in\dom{\psi^{*}}$, where $\q_k\in\R^n$ for all $k\in[K]$, we have $\Q=\gpsif{\bP^*}$ satisfying $0<\qki < 1$, $\icyi\qyi=-\sumkny\icki\qki$ for all $i\in[n]$ and $k\in[K]$, and $\sumk \one^\top\q_k\geq1$, for some $ \bP^*=\vctt{\p_1^{*\top}}{\p_K^{*\top}} \in \R^{Kn}$, where $\p_k^*\in\R^n$ for all $k\in[K]$.
\end{lemma}

\begin{proof}
    Under Assumption~\ref{asm:m_ce_loss_assump}, we have the definition of $\psi$ such that
    \begin{align*}
        \psif{\bP} = \ell^{-1}\Biggl(\sumn\multilossf{\seqr{\pki}{k=1}{K}}\Biggl) = \ell^{-1}\Biggl(\sumn\ln\Biggl(1+\sumkny \expf{\cyi\pyi-\cki\pki}\Biggl)\Biggl).
    \end{align*}
    For simplicity, we denote
    \begin{align} \label{eq:m_delta_setup}
        \deltai:=\sumkny \expf{\cyi\pyi-\cki\pki}, \text{ and }\deltaki:=\expf{\cyi\pyi-\cki\pki},
    \end{align}
    for all $i\in[n]$ and $k\in[K]$. By the definition of $\multiloss$, we also have $\deltai=\exp\Bigl(\multilossf{\seqr{\pki}{k=1}{K}}\Bigl)-1$ for all $i\in[n]$.
    For any $\Q=\vctt{\q_1^\top}{\q_K^\top}\in\dom{\psi^{*}}$, where $\q_k\in\R^n$ for all $k\in[K]$, we have $\psicf{\Q}=\usup{\bP\in\R^{Kn}}\ip{\bP}{\Q}-\psif{\bP}=\ip{\bP^*}{\Q}-\psif{\bP^*}$, where $\Q=\gpsif{\bP^*}$ for some $\bP^*=\vctt{\p_1^{*\top}}{\p_K^{*\top}} \in \R^{Kn}$, where $\p_k^*\in\R^n$. Therefore, we can conclude that
    \begin{align} \label{eq:m_qki_ori}
        \qki=\frac{\frac{\partial \multilossf{\seqr{\pki^*}{k=1}{K}}}{\partial \pki^*}}{\ellpellnof{\sumn\multilossf{\seqr{\pki^*}{k=1}{K}}}}=
        \begin{cases}
        \frac{\cyi \deltai}{\pts{1+\deltai} \ellpellnofzo{\sumn\multilossf{\seqr{\pki^*}{k=1}{K}}} } &\; k=y_i\\
        \frac{-\cki\deltaki}{\pts{1+\deltai} \ellpellnofzo{\sumn\multilossf{\seqr{\pki^*}{k=1}{K}}}} &\; k\neq y_i,\\
        \end{cases}
    \end{align}
    for all $i\in[n]$ and $k\in[K]$. Since we have $\ell'>0$ and the simplex labeling that $\cki=\cond{\frac{K-1}{K}}{k=y_i}{-\frac{1}{K}}{k\neq y_i}$ in Assumption~\ref{asm:m_ce_loss_assump}, we can conclude that $\qyi\geq\qki$, $\qki>0$, $\icyi\qyi=-\sumkny\icki\qki$ and $\sumk\qki=\frac{\deltai}{\pts{1+\deltai} \ellpellnofzo{\sumn\multilossf{\seqr{\pki^*}{k=1}{K}}} }$ for all $i\in[n]$ and $k\in[K]$. 
    
    Next, since logistic loss is used for $\ell$ in Assumption~\ref{asm:m_ce_loss_assump}, we get $\ellpellnofz{z}=\frac{\expf{z}-1}{\expf{z}}$ which is an increasing sub-additive function~\citep[Proof of Lemma 14]{ji2021characterizing}. Hence, we can have implication from Eq.~\eqref{eq:m_qki_ori} that
     \begin{align} \label{eq:m_qyi}
        \qyi=\frac{\cyi\ellpellnofzo{\multilossf{\seqr{\pki}{k=1}{K}}}}{{\ellpellnofzo{\sumn\multilossf{\seqr{\pki}{k=1}{K}}}}} \text{ for all } i\in[n].
    \end{align}
    
    Moreover, since we know $\multilossf{\seqr{\pki}{k=1}{K}}\leq\sumn\multilossf{\seqr{\pki}{k=1}{K}}$, $\cyi\leq1$, and $\qyi\geq\qki$ for all $i\in[n]$ and $k\in[K]$, these conditions imply $\qki\leq\qyi\leq\cyi<1$ for all $i\in[n]$ and $k\in[K]$. Next, by Eq.~\eqref{eq:m_qyi}, we let $\A:=\sumn\icyi \qyi=\sumn \sumkny \bigl(-\icki \qki\bigl)=\frac{\sumn\ellpellnofzo{\multilossf{\seqr{\pki}{k=1}{K}}}}{{\ellpellnofzo{\sumn\multilossf{\seqr{\pki}{k=1}{K}}}}}\geq1$, where we reuse the property that $\ellpellnofz{z}$ is an increasing sub-additive function. Followed by the operation in \cite[Eq. 31]{wang2021benign}, we have
\begin{align*}
    \A=\frac{K-1}{K}\A+\frac{1}{K}\A&=\frac{K-1}{K}\sumn\icyi \qyi + \frac{1}{K}\sumn \sumkny \bigl(-\icki \qki\bigl)\\
    &=\sumn \qyi + \sumn\sumkny \qki\\
    &=\sumn\sumk \qki=\sumk\one^T\q_k\geq1.
\end{align*}
    
\end{proof}

Next, we introduce Lemma~\ref{lem:m_ce_ccsz_cond_rev} that shows that $\sumk \one^\top\q_k\leq1$ implies the feasibility of the convex conjugacy feasibility constraint, i.e.~$\psicf{\Q}\leq0$.

\begin{lemma} \label{lem:m_ce_ccsz_cond_rev}
    Under Assumption~\ref{asm:m_ce_loss_assump}, for any $\Q\in\R^{Kn} \in\dom{\psi^{*}}$ that satisfies $\sumk \one^\top\q_k \leq 1$ also satisfies the convex conjugacy feasibility constraint $\psicf{\Q}\leq0$.
\end{lemma}

\begin{proof}
    Under Assumption~\ref{asm:m_ce_loss_assump}, we have
    \begin{align*}
        \psif{\bP}&=\ell^{-1}\Biggl(\sumn\multilossf{\seqr{\pki}{k=1}{K}}\Biggl)\\
        &= \ell^{-1}\Biggl(\sumn\ln\Biggl(1+\sumkny \expf{\cyi\pyi-\cki\pki}\Biggl)\Biggl).
    \end{align*}
    
    By Lemma~\ref{lem:m_ce_q_dom}, for any $\Q=\vctt{\q_1^\top}{\q_K^\top}\in\dom{\psi^{*}}$, where $\q_k\in\R^n$ for all $k\in[K]$, we have $\Q=\gpsif{\bP^*}$ for some $\bP^*=\vctt{\p_1^{*\top}}{\p_K^{*\top}} \in \R^{Kn}$. Also, we reuse the setup in Eq.~\eqref{eq:m_delta_setup} and~\eqref{eq:m_qki_ori} and have
    \begin{align} \label{eq:m_qki_ori2}
        \qki=
        \begin{cases}
        \frac{\cyi \deltai}{\pts{1+\deltai} \ellpf{\psif{\bP^*}}} &\; k=y_i\\
        \frac{-\cki\deltaki}{\pts{1+\deltai} \ellpf{\psif{\bP^*}}} &\; k\neq y_i,\\
        \end{cases}
    \end{align}
    for all $i\in[n]$ and $k\in[K]$. Next, in Lemma~\ref{lem:m_ccsz_cond}, we already show that $\psicf{\Q}\leq0$ if $\psif{\bP^*}\leq0$. Therefore, we only need to check the case when $\psif{\bP^*}>0$, and we have

    \begin{align}
        \psicf{\Q} = \sumn \sumk \pki^* \qki - \psif{\bP} &= \sumn \frac{\cyi\pyi^*\deltai + \sumkny \bigl(-\cki\pki^*\bigl)\deltaki}{\pts{1+\deltai} \ellpf{\psif{\bP^*}}} - \psif{\bP^*}\nonumber\\
        &= \sumn \frac{\sumkny \bigl(\cyi\pyi^*-\cki\pki^*\bigl)\deltaki}{\pts{1+\deltai} \ellpf{\psif{\bP^*}}} - \psif{\bP^*}\label{eq:m_ce_cc_inter}
    \end{align}
    On the other hand, we have 
    \begin{align}
        \qyi &= \frac{\cyi \deltai}{\pts{1+\deltai} \ellpf{\psif{\bP^*}}}\nonumber\\
        \iff \qyi\ellpf{\psif{\bP^*}} &= \cyi\ellpellnof{\multilossf{\seqr{\pki^*}{k=1}{K}}}\nonumber\\
        \iff \icyi\qyi\ellpf{\psif{\bP^*}} &= \ell'\Biggl(\ln\Biggl(\sumkny \deltaki\Biggl)\Biggl)\nonumber\\
        \iff \ellpinvf{\icyi\qyi\ellpf{\psif{\bP^*}}} &= \ln\Biggl(\sumkny \deltaki\Biggl)\nonumber\\
        \Rightarrow \ellpinvf{\ellpf{\psif{\bP^*}}} &\geq \ln\Biggl(\sumkny \deltaki\Biggl)\nonumber\\
        \Rightarrow \psif{\bP^*} &\geq \lnf{\deltaki}=\cyi\pyi^*-\cki\pki^* \myquad[1]\text{ for all } i\in[n] \text{, } k\neq y_i, \label{eq:m_ce_cc_ineq}
    \end{align}
    where the second equality comes from $\frac{\deltai}{1+\deltai}=\ellpellnof{\multilossf{\seqr{\pki^*}{k=1}{K}}}$, the first inequality derives from $\icyi\qyi< 1$ in Lemma~\ref{lem:m_ce_q_dom}, and the last inequality holds because $\lnf{\cdot}$ is an increasing function and $\deltaki>0$. Next, by introducing Eq.~\eqref{eq:m_ce_cc_ineq} into Eq.~\eqref{eq:m_ce_cc_inter}, we get
    \begin{align*}
        \psicf{\Q} &= \sumn \bigl(\sumkny \pts{\cyi\pyi^*-\cki\pki^*\bigl)\deltaki}{\pts{1+\deltai} \ellpf{\psif{\bP^*}}} - \psif{\bP^*}\\
        &\leq \sumn \frac{\sumkny \psif{\bP^*}\deltaki}{\pts{1+\deltai} \ellpf{\psif{\bP^*}}} - \psif{\bP^*}\\
        &= \psif{\bP^*} \Biggl(\sumn \frac{\sumkny \deltaki}{\pts{1+\deltai} \ellpf{\psif{\bP^*}}}-1\Biggl)\\
        &= \psif{\bP^*} \Biggl(\sumn \sumk \qki-1\Biggl) = \psif{\bP^*} \Biggl(\sumk \one^\top \q_k-1\Biggl)\leq0,
    \end{align*}
    where we reuse $\frac{\delta_i}{\pts{1+\delta_i}\ellpf{\psif{\bP^*}}}=\sumk\qki$ in the second to the last equality. The last inequality derives from the assumption in the lemma statement, $\sumk \one^\top\q_k \leq 1$, and because we are in the case where $\psif{\bP^*}>0$.
    This completes the proof of the lemma.
\end{proof}

Armed with Lemma~\ref{lem:m_ce_q_dom} and~\ref{lem:m_ce_ccsz_cond_rev}, we can prove the Part 2 of Theorem~\ref{thm:m_exp_ce_imp_bias}.

\begin{proof} (of Theorem~\ref{thm:m_exp_ce_imp_bias} [Part 2])
    Note that the constraint in Eq.~\eqref{eq:m_barQ_def} ($\psicf{\barQ}\leq0$) implicitly implies that $\barQ \in \dom{\psi^*}$.  
    Therefore, by Lemma~\ref{lem:m_ce_q_dom} the following constraints are implied:
    \begin{align*}
        \myquad[14]\barqki &> 0 \myquad[7]\text{ for all } i \in [n] \text{ and } k\in[K],  \\
        \icyi\qyi &=-\sumkny\icki\qki\myquad[2]  \text{ for all}\; i\in[n], \text{ and }\\
        \sumk \one^\top\q_k &\geq 1.
    \end{align*}

    We now show that a particular solution from the following convex program is also a solution in the original convex program~\eqref{eq:m_barQ_def}. We define a reformulated convex program:
        \begin{gather} \label{eq:m_ce_imp_bias_relax_eq_repeat}
        \tQ \in \uargmin{\Q\in\R^{Kn}}\usub{\Ff{\Q}}{\frac{1}{2}\Q^\top\CXXC\Q} \\
        \begin{aligned} \nonumber
            \myquad[7]\text{subject to} \myquad[4] -\qki                            &< 0 \myquad[7] \text{for all}\; i  \in [n] \text{ and } k\in[K],\\
            \icyi\qyi &=-\sumkny\icki\qki \myquad[2]\text{ for all}\; i\in[n], \\
            \text{and} \myquad[1] 1-\sumk \one^\top\q_k  &\leq 0.\\
        \end{aligned}
    \end{gather}
    
    It is necessary and sufficient for any optimal solution $\tQ$ to this reformulated convex program~\eqref{eq:m_ce_imp_bias_relax_eq_repeat} to satisfy its KKT conditions, listed below:
    \begin{subequations}\label{eq:m_ce_kkt_relaxed}
    \begin{align}
        -\qki                                    &<      0 \myquad[2] \text{for all } i\in[n] \text{ and }k\in[K],\label{m_ce_relaxed_primal_1}\\
        \icyi\qyi+\sumkny\icki\qki              &= 0\myquad[2] \text{for all } i\in[n] \label{m_ce_relaxed_primal_2},\\
        1-\sumk \one^\top\q_k                    &\leq   0,\label{m_ce_relaxed_primal_3}\\
        \lambda_{k,i}                               &\geq   0 \myquad[2] \text{for all } i\in[n] \text{ and }k\in[K],\label{m_ce_relaxed_dual_1}\\
        \delta_i                                    &\in\R\myquad[2] \text{for all } i\in[n],\label{m_ce_relaxed_dual_2}\\
        \mu                                    &\geq   0,\label{m_ce_relaxed_dual_3}\\
        -\lambda_{k,i} \qki                          &=      0 \myquad[2] \text{for all } i\in[n]  \text{ and }k\in[K],\label{m_ce_relaxed_comp_1}\\
        \mu\Biggl(1-\sumk \one^\top\q_k\Biggl)           &=      0,\label{m_ce_relaxed_comp_3}\\
        \cXXc\q_k-\blambda_k+\diag{\ick}\bdelta -\mu\one  &=      \zero \myquad[2] \text{for all } k\in[K].\label{m_ce_relaxed_station_1}
    \end{align}
    \end{subequations}
    
    Then we can pick a candidate solution $\tbq_k=\frac{\cXXic}{\sumk \cXXicv}$ satisfying all KKT conditions such that

    \begin{itemize}
        \item The primal feasibility equations, Eq.~\eqref{m_ce_relaxed_primal_1} is satisfied by theorem statement that $\cki\beta_{k,i}>0$, and ~\eqref{m_ce_relaxed_primal_2} is satisfied because of the following: Followed by \cite[Theorem 1 Step 2]{wang2021benign}, we let $\g_i\in\R^n$ denote the $i$th row of $(\X\X^T)^{-1}$ for all $i\in[n]$. Then for $i$th element of $\tq_k$, we have $\tqki=\frac{\cki\g_i^T\ck}{\sumk \ck^T\XXi\ck}$. Thus, for all $i\in[n]$, we have
        \begin{align*}
            \icyi\tqyi + \sumkny\icki \tqki=\frac{\g_i^T(\cy+\sumkny\ck)}{\sumk \ck^T\XXi\ck}=\frac{\g_i^T\pts{\sumk\ck}}{\sumk \ck^T\XXi\ck}=0,
        \end{align*}
        where the last equality followed by the simplex definition of $\ck$. Eq.~\eqref{m_ce_relaxed_primal_3} is satisfied such that $1-\sumk \one^\top\tbq_k=0$.
        \item The dual feasibility equations, Eq.~\eqref{m_ce_relaxed_dual_1} and Eq.~\eqref{m_ce_relaxed_dual_3} are satisfied by setting $\blambda_k =\zero$, and $\mu=\frac{1}{\sumk \cXXicv}$.
        \item The complementary slackness equations, Eq.~\eqref{m_ce_relaxed_comp_1} and Eq.~\eqref{m_ce_relaxed_comp_3} are satisfied because $\blambda_k =\zero$ and $1-\sumk \one^\top\tbq_k = 0$.
        \item Stationary condition is satisfied because we choose $\bdelta=\zero$, and then
        \begin{align*}
            &\cXXc\tbq_k-\blambda_k+\diagt{\ick}\boldsymbol{\delta} -\mu\one\\
            &=\frac{1}{\sumk \cXXicv}\one -\zero + \zero-\frac{1}{\sumk \cXXicv}\one=0,
        \end{align*}
        for all $k\in[K]$.
    \end{itemize}
    
    Lastly, by Lemma~\ref{lem:m_ce_ccsz_cond_rev}, we have $\psicf{\tQ}\leq0$ since $\sumk \one^\top\tbq_k=1$. Therefore, $\tQ$ is also a solution in the original convex program~\eqref{eq:m_barQ_def} that satisfies $\psicf{\tQ}\leq0$, and we end up with $\barQ=\tQ$. As a result, by Lemma~\ref{lem:m_imp_bias_res}, we have $\barw_k=\ulim{\tinf{t}}\normalize[2]{\w_{k,t}}=\normalize[2]{\X^\top\diag{\ick}\barbq_k}=\normalize[2]{\XXXc}$ for each $k$. Therefore, $\barw_k$ parallel to simplex version MNI $\wsimk=\XXXc$ for each $k$. The proof is complete.
\end{proof}

\subsection{Convexity and smoothness proof} \label{app_m_conv_beta}
We can directly apply loss functions defined in Assumption~\ref{asm:b_loss_assump} in the binary case with the following lemmas that ensure the properties of convexity and $\beta$-smoothness with respect to the $\nm{\infty}$ norm of $\psi$ in the multiclass case.\\

\begin{lemma} [From Lemma 12 in~\cite{ji2021characterizing}]
    If $\ell'^2/\pts{\ell\ell''}$ is increasing on $\pts{-\infty,\infty}$, then $\psi$ is jointly convex under Assumption~\ref{asm:m_ova_loss_assump}, and is individually convex toward each $\xik$ under Assumption~\ref{asm:m_ce_loss_assump}.
\end{lemma}
\begin{proof}
We discuss the situation under Assumption~\ref{asm:m_ova_loss_assump} and Assumption~\ref{asm:m_ce_loss_assump} separately. Under Assumption~\ref{asm:m_ova_loss_assump}, we have the definition of $\psi$ such that
\begin{align*}
    \psif{\bXi}=\ell^{-1}\Biggl(\sumn\sumk\ellf{\xiki}\Biggl),
\end{align*}
and the gradient $\gpsif{\bXi}$ is defined by
\begin{align*}
    \gpsif{\bXi}_{k,i}=\frac{\partial \psif{\bXi}}{\partial\xiki}=\frac{\ellpf{\xiki}}{\ellpellnof{\sumn\ellf{\sumk\xiki}}}=\frac{\ellpf{\xiki}}{\ellpf{\psif{\bXi}}},
\end{align*}
for all $i\in[n]$ and $k\in[K]$. Next, the Hessian $\nabla^2\psif{\bXi}\in\R^{Kn\times Kn}$ is
\begin{align}
    \nabla^2\psif{\bXi} = \diag{\frac{\ellppf{\xi_{1,1}}}{\ellpf{\psif{\bXi}}},\cdots,\frac{\ellppf{\xi_{K,n}}}{\ellpf{\psif{\bXi}}}} - \frac{\ellppf{\psif{\bXi}}}{\ellpf{\psif{\bXi}}}\gpsif{\bXi}\gpsif{\bXi}^\top.\label{eq:m_ova_hessian}
\end{align}
Note that the Hessian is identical to \cite[Lemma 12, Eq.~(24)]{ji2021characterizing} with additional $K$ dimensions; therefore, the convexity for $\psi$ holds for loss functions under Assumption~\ref{asm:m_ova_loss_assump}.

Next, we show the convexity proof for cross-entropy loss under Assumption~\ref{asm:m_ce_loss_assump}. According to the definition of $\psi$ function under Assumption~\ref{asm:m_ce_loss_assump}, we have
\begin{align*}
    \psif{\bXi}&=\ell^{-1}\Biggl(\sumn\multilossf{\seqr{\xiki}{k=1}{K}}\Biggl),
\end{align*}
and the gradient $\gpsif[\xik]{\bXi}$ is defined by
\begin{align*}
    \frac{\partial \psif{\bXi}}{\partial\xiki}=\frac{\frac{\partial\multilossf{\seqr{\xiki}{k=1}{K}}}{\partial \xiki}}{\ellpellnof{\sumn\multilossf{\seqr{\xiki}{k=1}{K}}}}=\pts{\frac{1}{\ellpf{\psif{\bXi}}}}\frac{\partial\multilossf{\seqr{\xiki}{k=1}{K}}}{\partial \xiki}.
\end{align*}
Next, the second order of the partial derivatives are:
\begin{align*}
    \frac{\partial^2\psif{\bXi}}{\partial \xiki^2} &= \pts{\frac{1}{\ellpf{\psif{\bXi}}}}\frac{\partial^2\multilossf{\seqr{\xiki}{k=1}{K}}}{\partial\xiki^2}\\
    &-\frac{\ellppf{\psif{\bXi}}}{\ellpf{\psif{\bXi}}}\pts{\frac{1}{\ellpf{\psif{\bXi}}}}\frac{\partial\multilossf{\seqr{\xiki}{k=1}{K}}}{\partial \xiki}\pts{\frac{1}{\ellpf{\psif{\bXi}}}}\frac{\partial\multilossf{\seqr{\xiki}{k=1}{K}}}{\partial \xiki},\\
    \frac{\partial^2\psif{\bXi}}{\partial \xiki\partial \xikj} &= -\frac{\ellppf{\psif{\bXi}}}{\ellpf{\psif{\bXi}}}\pts{\frac{1}{\ellpf{\psif{\bXi}}}}\frac{\partial\multilossf{\seqr{\xiki}{k=1}{K}}}{\partial \xiki}\pts{\frac{1}{\ellpf{\psif{\bXi}}}}\frac{\partial\multilossf{\seqr{\xikj}{k=1}{K}}}{\partial \xikj}
\end{align*}
for all $i\neq j$ and $k\in[K]$. Hence, we can write the Hessian $\nabla_{\xik}^2\psif{\bXi}\in\R^{n\times n}$ as
\begin{align*}
    \nabla_{\xik}^2\psif{\bXi}&=\diag{\frac{\frac{\partial^2\multilossf{\seqr{\xiko}{k=1}{K}}}{\partial \xiko^2}}{\ellpf{\psif{\bXi}}},\cdots,\frac{\frac{\partial^2\multilossf{\seqr{\xikn}{k=1}{K}}}{\partial \xikn^2}}{\ellpf{\psif{\bXi}}}}-\frac{\ellppf{\psif{\bXi}}}{\ellpf{\psif{\bXi}}}\nabla_{\xik}\psif{\bXi}\nabla_{\xik}\psif{\bXi}^\top,
\end{align*}
for all $k\in[K]$. Therefore, it remains to show that for any $v\in\R^{n}$
\begin{align} \label{eq:m_ce_target}
    \sumn \frac{\frac{\partial^2\multilossf{\seqr{\xiki}{k=1}{K}}}{\partial \xiki^2}}{\ellpf{\psif{\bXi}}} \vi \geq \frac{\ellppf{\psif{\bXi}}}{\ellpf{\psif{\bXi}}}\pts{\sumn\frac{\frac{\partial\multilossf{\seqr{\xiki}{k=1}{K}}}{\partial \xiki}}{\ellpf{\psif{\bXi}}}\vi}^2.
\end{align}
By the Cauchy-Schwarz inequality, we can write
\begin{align}
    \pts{\sumn \frac{\frac{\partial^2\multilossf{\seqr{\xiki}{k=1}{K}}}{\partial \xiki^2}}{\ellpf{\psif{\bXi}}}v_i^2}\pts{\sumn\frac{\mts{\frac{\partial\multilossf{\seqr{\xiki}{k=1}{K}}}{\partial \xiki}}^2}{\frac{\partial^2\multilossf{\seqr{\xiki}{k=1}{K}}}{\partial \xiki^2}\ellpf{\psif{\bXi}}}}&\geq\pts{\sumn\frac{\frac{\partial\multilossf{\seqr{\xiki}{k=1}{K}}}{\partial \xiki}}{\ellpf{\psif{\bXi}}} \vi}^2 \label{eq:m_ce_ineq}
\end{align}
Next, we can show that Eq.~\eqref{eq:m_ce_target} is satisfied for each $k\in[K]$ by showing
\begin{align} \label{eq:m_ce_cond}
    \frac{\ellpf{\psif{\bXi}}^2}{\ellppf{\psif{\bXi}}}=\frac{\ellpellnof{\sumn \multilossf{\seqr{\xiki}{k=1}{K}}}^2}{\ellppellnof{\sumn \multilossf{\seqr{\xiki}{k=1}{K}}}}\geq\pts{\sumn\frac{\mts{\frac{\partial\multilossf{\seqr{\xiki}{k=1}{K}}}{\partial \xiki}}^2}{\frac{\partial^2\multilossf{\seqr{\xiki}{k=1}{K}}}{\partial \xiki^2}}}.
\end{align}
For cross-entropy loss under Assumption~\ref{asm:m_ce_loss_assump}, we have
\begin{align*}
    \psif{\bXi}&=\ell^{-1}\Biggl(\sumn\multilossf{\seqr{\xiki}{k=1}{K}}\Biggl) = \ln\Biggl(\exp\Biggl(\sumn\ln\Biggl(1+\sumkny \expf{\cyi\xiyi-\cki\xiki}\Biggl)\Biggl)-1\Biggl).
\end{align*}
We start from the LHS of Eq.~\eqref{eq:m_ce_cond}, we have
\begin{align} \label{eq:m_ce_cond_lhs}
    \frac{\ellpf{\psif{\bXi}}^2}{\ellppf{\psif{\bXi}}}=\frac{\ellpellnof{\sumn \multilossf{\seqr{\xiki}{k=1}{K}}}^2}{\ellppellnof{\sumn \multilossf{\seqr{\xiki}{k=1}{K}}}} = \exp\Biggl(\sumn \multilossf{\seqr{\xiki}{k=1}{K}}\Biggl) - 1,
\end{align}
by direct expansion with $\ellf{z}=\lnf{1+\expf{z}}$. Next, we work on RHS of Eq.~\eqref{eq:m_ce_cond}. For simplicity, we denote 
\begin{align*}
    \deltai:=\sumkny \expf{\cyi\xiyi-\cki\xiki} \text{ and }\deltaki:=\expf{\cyi\xiyi-\cki\xiki}.
\end{align*}
Since $\multilossf{\seqr{\xiki}{k=1}{K}}=\lnf{1+\sumkny \expf{\cyi\xiyi-\cki\xiki}}=\lnf{1+\deltai}$, we have the first derivative as
\begin{align} \label{eq:m_ce_fd}
    \frac{\partial\multilossf{\seqr{\xiki}{k=1}{K}}}{\partial \xiki}=
    \begin{cases}
    \frac{\cyi \deltai}{1+\deltai} &\; k=y_i\\
    \frac{-\cki\deltaki}{1+\deltai} &\; k\neq y_i\\
    \end{cases},
\end{align}
and the second derivative as 
\begin{align} \label{eq:m_ce_sd}
    \frac{\partial^2\multilossf{\seqr{\xiki}{k=1}{K}}}{\partial \xiki^2}=
    \begin{cases}
    \frac{\cyi^2\deltai}{\mts{1+\deltai}^2} &\; k=y_i\\
    \frac{\cki^2\deltaki\pts{1+\sumknynk \deltaki}}{\mts{1+\deltai}^2} &\; k\neq y_i\\
    \end{cases}.
\end{align}
Next, by substituting Eq.~\eqref{eq:m_ce_fd} and Eq.~\eqref{eq:m_ce_sd} into the RHS of Eq.~\eqref{eq:m_ce_cond}, we get
\begin{align*}
    \frac{\mts{\frac{\partial\multilossf{\seqr{\xiki}{k=1}{K}}}{\partial \xiki}}^2}{\frac{\partial^2\multilossf{\seqr{\xiki}{k=1}{K}}}{\partial \xiki^2}}=
    \begin{cases}
    \deltai &\; k=y_i\\
    \frac{\deltaki}{1+\sumknynk \deltaki} &\; k\neq y_i\\
    \end{cases}.
\end{align*}
Based on this, we can also derive
\begin{align}\label{eq:m_ce_rhs_ineq}
    \frac{\mts{\frac{\partial\multilossf{\seqr{\xiki}{k=1}{K}}}{\partial \xiyi}}^2}{\frac{\partial^2\multilossf{\seqr{\xiki}{k=1}{K}}}{\partial \xiyi^2}}\geq \frac{\mts{\frac{\partial\multilossf{\seqr{\xiki}{k=1}{K}}}{\partial \xiki}}^2}{\frac{\partial^2\multilossf{\seqr{\xiki}{k=1}{K}}}{\partial \xiki^2}},
\end{align} for all $k\neq y_i$ by a direct comparison of the two conditions. Moreover, we can also write
\begin{align} \label{eq:m_ce_y_eq}
    \frac{\mts{\frac{\partial\multilossf{\seqr{\xiki}{k=1}{K}}}{\partial \xiyi}}^2}{\frac{\partial^2\multilossf{\seqr{\xiki}{k=1}{K}}}{\partial \xiyi^2}}=\deltai=\exp\Bigl(\multilossf{\seqr{\xiki}{k=1}{K}}\Bigl)-1.
\end{align}
Therefore, starting from Eq.~\eqref{eq:m_ce_cond_lhs}, we can show
\begin{align*}
    \frac{\ellpf{\psif{\bXi}}^2}{\ellppf{\psif{\bXi}}}&= \exp\Biggl(\sumn \multilossf{\seqr{\xiki}{k=1}{K}}\Biggl) - 1\\
    &\geq \sumn\pts{\exp\Bigl(\multilossf{\seqr{\xiki}{k=1}{K}}\Bigl)-1}\\
    &=\pts{\sumn\frac{\mts{\frac{\partial\multilossf{\seqr{\xiki}{k=1}{K}}}{\partial \xiyi}}^2}{\frac{\partial^2\multilossf{\seqr{\xiki}{k=1}{K}}}{\partial \xiyi^2}}}\geq\pts{\sumn\frac{\mts{\frac{\partial\multilossf{\seqr{\xiki}{k=1}{K}}}{\partial \xiki}}^2}{\frac{\partial^2\multilossf{\seqr{\xiki}{k=1}{K}}}{\partial \xiki^2}}},
\end{align*}
where the first inequality holds because $\ff{z}=\expf{z}-1$ is a super-additive function, the second equality comes from Eq.~\eqref{eq:m_ce_y_eq}, and the last inequality derives from Eq.~\eqref{eq:m_ce_rhs_ineq}. Therefore, Eq.~\eqref{eq:m_ce_cond} holds for all $k\in[K]$. This completes the proof of this lemma.
\end{proof}

\begin{lemma}[From Lemma 13 in~\cite{ji2021characterizing}]
    Under Assumption~\ref{asm:m_ova_loss_assump}, if $\ell''\leq c\ell'$ for some constant $c>0$, then the smoothness constant $\beta\leq cnK$ for $\psi$. Particularly, for exponentially-tailed loss, the smoothness constant is $\beta = 1$ for $\psi$. Under Assumption~\ref{asm:m_ce_loss_assump}, the cross-entropy loss has the smoothness constant $\beta=2K^2$ for $\psi$.
\end{lemma}

\begin{proof}
We follow the proof strategy in \cite[Lemma 13]{ji2021characterizing}~\cite[Lemma 14]{shalev2007online}, to check the $\beta$ smoothness of $\psi$ with respect to $\nm{\infty}$ norm.
Note that it is sufficient to show for any $\bXi$, $\bv\in\R^{Kn}$, it holds that $\bv^\top\nabla^2\psif{\bXi}\bv\leq\beta\nnorm[\infty]{\bv}^2$.
We discuss the situations under Assumption~\ref{asm:m_ova_loss_assump} and Assumption~\ref{asm:m_ce_loss_assump} separately.

Under Assumption~\ref{asm:m_ova_loss_assump}, by the definition of $\psi$, we have
\begin{align*}
    \psif{\bXi}=\ell^{-1}\Biggl(\sumn\sumk\ellf{\xiki}\Biggl).
\end{align*}
According to the Hessian we derived in Eq.~\eqref{eq:m_ova_hessian}, it is enough to show that for any $\bv=\vct{v_1}{v_K}\in\R^{Kn}$, where $v_k\in\R^n$ for all $k\in[K]$ and $\bXi\in\R^{Kn}$, we have
\begin{align*}
    \sumn\sumk \frac{\ellppf{ \xiki}}{\ellpf{\psif{\bXi}}} \vki^2 \leq \beta \umax{1\leq k\leq K}\umax{1\leq i\leq n} \vki^2.
\end{align*}
Note that the condition is identical to \cite[Lemma 13, Eq. (27)]{ji2021characterizing} with additional summation in $K$; therefore, the smoothness constant conclusion for $\psi$ is the same as in the binary case for loss functions under Assumption~\ref{asm:m_ova_loss_assump}.

Next, for cross-entropy loss under Assumption~\ref{asm:m_ce_loss_assump}, by the definition of $\psi$, we have 
\begin{align*}
    \psif{\bXi}&=\ell^{-1}\Biggl(\sumn \multilossf{\seqr{\xiki}{k=1}{K}}\Biggl)=\ln\Biggl(\exp\Biggl(\sumn\ln\Biggl(1+\sumkny \expf{\cyi\xiyi-\cki\xiki}\Biggl)\Biggl)-1\Biggl).
\end{align*}
For simplicity, we again denote
\begin{align*}
    \deltai:=\sumkny \expf{\cyi\xiyi-\cki\xiki}, \text{ and }\deltaki:=\expf{\cyi\xiyi-\cki\xiki},
\end{align*}
and since $\multilossf{\seqr{\xiki}{k=1}{K}}=\lnf{1+\sumkny \expf{\cyi\xiyi-\cki\xiki}}=\lnf{1+\deltai}$, and we have the first derivative as
\begin{align*} 
    \frac{\partial\multilossf{\seqr{\xiki}{k=1}{K}}}{\partial \xiki}=
    \begin{cases}
        \frac{\cyi\deltai}{1+\deltai} &\; k=y_i\\
        \frac{-\cki\deltaki}{1+\deltai} &\; k\neq y_i\\
    \end{cases},
\end{align*}
and the second derivative as
\begin{align*}
    \frac{\partial^2\multilossf{\seqr{\xiki}{k=1}{K}}}{\partial \xiki\partial\xihi}=
    \begin{cases}
    \frac{\cyi^2\deltai}{\mts{1+\deltai}^2} &\; k=h=y_i\\
    \frac{-\cki\chhi\deltahi}{\mts{1+\deltai}^2} &\; k=y_i\text{ and } h \neq y_i\\
    \frac{\cki^2\deltaki\pts{1+\sumknynk \deltaki}}{\mts{1+\deltai}^2} &\; k=h\neq y_i\\
    \frac{-\cki\chhi\deltaki\deltahi}{\mts{1+\deltai}^2} &\; k\neq y_i \text{ and }h\neq y_i\\
    \end{cases}.
\end{align*}
By direct comparison in value, we can conclude that the first derivative w.r.t $\xiyi$ upper-bounds all the second derivatives such that $\frac{\partial\multilossf{\seqr{\xiki}{k=1}{K}}}{\partial \xiyi} \geq \frac{\partial^2\multilossf{\seqr{\xiki}{k=1}{K}}}{\partial \xiki\partial\xihi}$ for all $k,h\in [K]$. Also, we have $\frac{\partial\multilossf{\seqr{\xiki}{k=1}{K}}}{\partial \xiyi} = \frac{\cyi\Bigl(\expf{\multilossf{\seqr{\xiki}{k=1}{K}}}-1\Bigl)}{\exp\Bigl(\multilossf{\seqr{\xiki}{k=1}{K}}\Bigl)}$ and $$\ellpf{\psif{\bXi}}=\ellpellnoft{\sumn \multilossf{\seqr{\xiki}{k=1}{K}}}=\frac{\exp\Bigl(\sumn\multilossf{\seqr{\xiki}{k=1}{K}}\Bigl)-1}{\exp\Bigl(\sumn\multilossf{\seqr{\xiki}{k=1}{K}}\Bigl)},$$ since $\ellpellnofz{z} = \frac{\expf{z}-1}{\expf{z}}$. Hence, we can write
\begin{align*}
     \sumk\sumh\sumn \frac{\frac{\partial^2\multilossf{\seqr{\xiki}{k=1}{K}}}{\partial \xiki \partial \xihi}}{\ellpf{\psif{\bXi}}}\vki\vhi
     &\leq \sumk\sumh\sumn \frac{\frac{\partial\multilossf{\seqr{\xiki}{k=1}{K}}}{\partial \xiyi}}{\ellpf{\psif{\bXi}}}\normo{\vki\vhi} \\
     &=\sumk\sumh\sumn\frac{\frac{\cyi\Bigl(\expf{\multilossf{\seqr{\xiki}{k=1}{K}}}-1\Bigl)}{\expf{\multilossf{\seqr{\xiki}{k=1}{K}}}}}{\frac{\expf{\sumn\multilossf{\seqr{\xiki}{k=1}{K}}}-1}{\expf{\sumn\multilossf{\seqr{\xiki}{k=1}{K}}}}}\normo{\vki\vhi}\\
     &\leq \sumk\sumh2\umax{1\leq i\leq n} \normo{\vki\vhi}\leq 2K^2 \nnorm[\infty]{\bv}^2,
\end{align*}
where the second inequality holds because $\frac{\sumn\frac{\pts{\expf{\multilossf{\seqr{\xiki}{k=1}{K}}}-1}}{\expf{\multilossf{\seqr{\xiki}{k=1}{K}}}}}{\frac{\expf{\sumn\multilossf{\seqr{\xiki}{k=1}{K}}}-1}{\expf{\sumn\multilossf{\seqr{\xiki}{k=1}{K}}}}}\leq2$ by \cite[Proof of Lemma 14]{ji2021characterizing}. This completes the proof of the lemma.
\end{proof}

\newpage

\section{Proofs of converse results}\label{sec:converseproofs}

In this section, we collect the proofs of the converse results in Section~\ref{sec:converse}.

\subsection{Proof of Proposition~\ref{prop:converseexact}}\label{sec:converseexactproof}
\begin{proof}
The proof is divided into two parts.
\paragraph{Proof of Part 1}
For the proof of Part 1, 
we work from Part 1 of the proof of Theorem~\ref{thm:b_sen_upperbound}.
There, we showed that we require $\barbq$ to satisfy a series of characteristic equations; in particular, $\barbq$ needs to satisfy the equation
\begin{align}\label{eq:kktcharacteristic1}
    \diag{\y} \XX \diag{\y} \barbq = \mu h(\barbq) \text{ for some } \mu > 0,
\end{align}
where we recall that we defined $h(\cdot) := \gpf[-1]{\cdot}$ as shorthand.
Our goal is to show that if $\y$ is not an exact eigenvector of $\XX$, then all candidate solutions in the family $\barbq \propto \diag{\y} (\XX)^{-1} \y$ cannot satisfy Eq.~\eqref{eq:kktcharacteristic1} for any value of $\mu > 0$.
We consider the candidate solution $\barbq = \beta \diag{\y} (\XX)^{-1} \y$ for some $\beta > 0$.
Because we have assumed that $\XX$ is full-rank, this is the unique direction of the dual solution that would correspond to a primal solution that is proportional to the MNI.
Then, Eq.~\eqref{eq:kktcharacteristic1} being satisfied for some $\mu > 0$ implies
\begin{align}
\beta \cdot \diag{\y} \XX \diag{\y} \cdot \diag{\y} (\XX)^{-1} \y &= \mu \cdot h(\beta \cdot \diag{\y} (\XX)^{-1} \y)  \nonumber \\
\implies \one &= \frac{\mu}{\beta} \cdot h(\beta \diag{\y} (\XX)^{-1} \y).\label{eq:hequation}
\end{align}

Now, we recall the properties of $h(\cdot)$ that arise when $g(d) \neq d$, i.e. when the mapping $g(\cdot)$ is not the identity.
Because $g(\cdot)$ is strictly convex and increasing, we have that $g^{-1}(\cdot)$ is strictly concave and $h(\cdot) = \gpf[-1]{\cdot}$ is therefore \emph{strictly decreasing}.
This means that for any $d \neq e$, we have $h(d) \neq h(e)$.
Consequently, for Eq.~\eqref{eq:hequation} to be true for any value of $\mu > 0$, we require all the entries of $\diag{\y} (\XX)^{-1} \y$ to be equal. In other words, we need
\begin{align*}
    \diag{\y} (\XX)^{-1} \y &\propto \one \\
    \implies \diag{\y} (\XX)^{-1} \y &= \gamma \one \text{ for some } \gamma \neq 0 \\
    \implies (\XX)^{-1} \y &= \gamma \y \text{ for some } \gamma \neq 0 \\
    \implies \XX \y &= \frac{1}{\gamma} \y \text{ for some } \gamma \neq 0,
\end{align*}
implying that $\y$ needs to be an exact non-zero eigenvector of $\XX$.
This completes the proof of the first part of the proposition.


\paragraph{Proof of Part 2}
For the proof of Part 2, recall the KKT conditions in Eq.~\eqref{eq:b_kkt_relaxed} in the proof of Lemma~\ref{lem:b_imp_bias_relax} for the auxiliary convex program~\eqref{eq:b_imp_bias_relax_eq}, reproduced below for completeness.

\begin{subequations}\label{b_kktprimalall}
    \begin{align}
        -q_i                                    &<      0 \myquad[2] \text{for all } i\in[n],\label{b_conv_primal_2}\\
        1-\sumn \gf[-1]{q_i}                    &\leq   0,\label{b_conv_primal_3}\\
        \lambda_i                               &\geq   0 \myquad[2] \text{for all } i\in[n],\label{b_conv_dual_2}\\
        \mu                                     &\geq   0,\label{b_conv_dual_3}\\
        -\lambda_i q_i                          &=      0 \myquad[2] \text{for all } i\in[n],\label{b_conv_comp_2}\\
        \mu\biggl(1-\sumn \gf[-1]{q_i}\biggl)           &=      0,\label{b_conv_comp_3}\\
        \yXXy\q-\blambda-\mu \hf{\q}  &=      \zero,\label{b_conv_station_1}
    \end{align}
    \end{subequations}
    where we denote $\hf{\q}:=\gpf[-1]{\q}$ as shorthand.

    To write our candidate solution $\barbq$ in the case where $\XX = \D = \diag{\dbold} \neq \I$, we define some additional notation.
    We define $f(d) := \frac{h(d)}{d}$ on the domain $(0,1]$.
    We note that, because $h(d)$ is strictly decreasing in $d$ and $\frac{1}{d}$ is strictly decreasing in $d$, $f(d)$ is strictly decreasing in $d$ as well, and is therefore invertible. Also note that $f(d) \in (0,\infty)$.
    
    Then we can pick a candidate $\barbq$ such that for every $i \in [n]$, we have 
    \begin{align*}
    \barq_i = f^{-1} \left(\frac{d_i}{\mu}\right),
    \end{align*}
    where $\mu > 0$ satisfies
    \begin{align}\label{eq:muregularity}
    \sum_{i=1}^n g^{-1}\biggl(f^{-1} \left(\frac{d_i}{\mu}\right)\biggl) &= 1.
    \end{align}
    
    Before verifying the KKT conditions for this candidate solution, let us confirm that it is possible to select a value of $\mu > 0$ satisfying Eq.~\eqref{eq:muregularity}.
    Hiding the dependence on $\dbold$, we define $H(\mu) := \sumn g^{-1}\biggl(f^{-1} \left(\frac{d_i}{\mu}\right)\biggl)$.
    Note that $H(\mu)$ is a continuous function in $\mu > 0$.
    Moreover, it is easy to verify that $H(0) = \sum_{i=1}^n g^{-1}(0) = 0 < 1$.
    Assuming $h(1) \neq 0$, we can set $\mu = \frac{d_1}{f(1)}$ and get $H(\mu) > g^{-1}\left(f^{-1}(f(1))\right) = g^{-1}(1) = 1$.
    In the alternative case where $h(1) = 0$, we would still get $\lim_{\mu \to \infty} H(\mu) > 1$ by the same logic; meaning that there exists a value of $\mu > 0$ such that $H(\mu) > 1$ as well.
    In either case, the mean-value-theorem implies that there exists a $\mu \in \left[0, \frac{d_1}{f(1)}\right)$ such that Eq.~\eqref{eq:muregularity} is satisfied.
    
    We verify that all the KKT conditions are satisfied by this candidate solution:
    \begin{itemize}
        \item \emph{Primal feasibility:} ~\eqref{b_conv_primal_2} is satisfied because the domain of $f(\cdot)$, and therefore the range of $f^{-1}(\cdot)$, is $(0,1]$. ~\eqref{b_conv_primal_3} is satisfied because Eq.~\eqref{eq:muregularity} implies that $\sum_{i=1}^n g^{-1}(q_i) = 1$.
        \item \emph{Dual feasibility:} ~\eqref{b_conv_dual_2} is satisfied by setting $\blambda = \zero$, and~\eqref{b_conv_dual_3} is satisfied by the choice of $\mu \in \left[0, \frac{d_1}{f(1)}\right)$.
        \item \emph{Complementary slackness:}~\eqref{b_conv_comp_2}~\eqref{b_conv_comp_3} are satisfied because of the choices of $\blambda = \zero$ and Eq.~\eqref{eq:muregularity} respectively being satisfied.
        \item \emph{Stationary condition:} We require $\diag{\y} \XX \diag{\y} \barbq = \mu h(\barbq)$.
        This is equivalent to
        \begin{align*}
            d_i \barq_i &= \mu h(\barq_i) \\
            \iff \mu &= \frac{d_i}{f(\barq_i)} \text{ for all } i \in [n]. 
        \end{align*}
        Substituting our choice of $\barq_i$ into the RHS above gives
        \begin{align*}
            \frac{d_i}{f(\barq_i)} = \frac{d_i}{\frac{d_i}{\mu}} = \mu.
        \end{align*}
    \end{itemize}
    
    Thus, we have verified all the KKT conditions for this candidate solution.
    Ultimately, we get $\barw:=\ulim{\tinf{t}}\normalize[2]{\w_t}=\normalize[2]{\X^\top\diag{\y}\barbq}=\normalize[2]{\X^\top(\XX)^{-1} \D \diag{\y} \barbq}$.
    Therefore, the primal solution interpolates the adjusted binary levels given by
    \begin{align*}
        \tilde{y}_i = d_i y_i \barq_i = d_i y_i f^{-1} \left(\frac{d_i}{\mu}\right).
    \end{align*}
    This completes the proof.
\end{proof}

\subsection{Proof of Corollary~\ref{cor:ood_interpolation}}\label{sec:oodinterpolationproof}

In this section, we prove Corollary~\ref{cor:ood_interpolation}.

\begin{proof}
It is easy to verify that minimizing the importance-weighted empirical risk with a polynomial loss function of degree $m > 0$ becomes equivalent to minimizing the unweighted empirical risk on the following per-example loss function:
\begin{align}\label{eq:per-example-loss}
    \ell_i(\tilde{z}_i;Q) := \frac{1}{(Q^{-\frac{\Ind[i \in S]}{m}} - \tilde{z}_i)^m},
\end{align}
where $\tilde{z}_i := - y_i \langle \tilde{\x}_i, \w \rangle$ and $\tilde{\x}_i := Q^{-\frac{1}{m} \Ind[i \in S]}\x_i$.
Clearly, the per-example loss function in Eq.~\eqref{eq:per-example-loss} continues to verify Assumption~\ref{asm:b_loss_assump} for any fixed value of $Q > 0$.
Specifically, it continues to satisfy $\ell_i'(\ell_i^{-1}(z)) = m z^{\frac{m+1}{m}}$ for $z \geq 0$ and so we get $g_i(d) = g(d) = d^{\frac{m+1}{m}}$ for each $i \in [n]$.
Consequently, the convex program underlying the dual implicit bias is identical to~\eqref{eq:b_imp_bias_relax_eq} and the setting of Proposition~\ref{prop:converseexact}, with adjusted diagonal matrix $\D = \tilde{\X} \tilde{\X}^\top$, where we denote $\tX:=\diag{\Q^{-\frac{1}{m} \Ind[i \in S]}}\X$.
To apply Proposition~\ref{prop:converseexact}, we first calculate the functions $g(\cdot),g^{-1}(\cdot),h(\cdot),f(\cdot)$ and $f^{-1}(\cdot)$.
Direct calculations yield $g(z) = z^{\frac{m+1}{m}},g^{-1}(z) = z^{\frac{m}{m+1}},h(z) = \frac{m}{m+1} \cdot z^{-\frac{1}{m+1}},f(z) = \frac{m}{m+1} \cdot z^{-\frac{m+2}{m+1}}$, and $f^{-1}(z) = \left(\frac{(m+1)z}{m}\right)^{-\frac{m+1}{m+2}}$.
Applying Proposition~\ref{prop:converseexact} then gives us $\barq_i = f^{-1} \left(\frac{d_i}{\mu}\right) \propto  d_i^{-\frac{m+1}{m+2}}$.
Next, we have
\begin{align*}
    \barw:=\ulim{\tinf{t}}\normalize[2]{\w_t}=\normalize[2]{\tX^\top\diag{\y}\barbq}&=\normalize[2]{\tX^\top(\tXtX)^{-1} \D \diag{\y} \barbq}\\
    &=\normalize[2]{\X^\top(\XX)^{-1} \diag{\Q^{\frac{1}{m} \Ind[i \in S]}} \D \diag{\y} \barbq}.
\end{align*}

Therefore, the adjusted labels that are interpolated are proportional to $Q^{\frac{1}{m} \Ind[i \in S]}y_i d_i \cdot d_i^{-\frac{m+1}{m+2}} = Q^{\frac{1}{m} \Ind[i \in S]}d_i^{\frac{1}{m+2}}$.
It remains to calculate the value of $d_i$.
Note that we have assumed $\XX = \alpha \I$, and so $\ip{\x_i}{\x_j} = \delta_{ij}$ where $\delta_{ij}$ denotes the Kronecker delta function.
Because we have defined $\tilde{\x}_i := Q^{-\frac{1}{m} \Ind[i \in S]}\x_i$, we automatically get $\langle \tilde{\x}_i, \tilde{\x}_j \rangle = Q^{-\frac{1}{m} \cdot \pts{\Ind[i \in S]+\Ind[j \in S]}} \delta_{ij}$, meaning that $d_i = Q^{-\frac{2}{m} \Ind[i \in S]}$.
Putting all of this together results in interpolation of the per-example-adjusted labels $\tilde{y}_i \propto Q^{\frac{1}{m+2} \Ind[i \in S]} y_i$.
This completes the proof.
\end{proof}

\subsection{Proof of Proposition~\ref{prop:b_sen_lowerbound}}\label{sec:converseapproxproof}
In this section, we prove Proposition~\ref{prop:b_sen_lowerbound}.
\begin{proof}
Our starting point lies in the proof of Theorem~\ref{thm:b_sen_upperbound}; the necessity for the dual implicit bias $\barbq$ to satisfy the following characteristic equations, restated below.
\begin{subequations}
\begin{align}
    \yXXy\barbq                 &=      \mu \hf{\barbq},      \label{b_con_sen_charc_1}\\
    \mu                         &>      0,  \myquad[2]\text{ and } \label{b_con_sen_charc_2}\\
    \sumn \gf[-1]{\barq_i}      &=      1.  \label{b_con_sen_charc_3}
\end{align}
\end{subequations}

We consider in particular Equations~\eqref{b_con_sen_charc_1} and~\eqref{b_con_sen_charc_2}. Recalling that we defined $\qy := \diag{\y} \barbq$ as shorthand, our equivalent goal is to lower-bound $\nnorm[2]{\normalize[2]{\qy}-\normalize[2]{\y}}$. Moreover, pre-multiplying both sides by $\diag{\y}$ means that the first and second characteristic equations imply
\begin{align}\label{eq:b_sen_charc_for_converse}
    \XX \qy = \mu \diag{\y} \hf{\barbq} \text{ for some } \mu > 0.
\end{align}

Next, we show that without loss of generality we can set $\mu = 1$ or any positive value, which greatly simplifies the proof exposition.
The reason for this is as follows: consider a solution $\qy$ that satisfies Equation~\eqref{eq:b_sen_charc_for_converse} for some $\mu \neq 1$.
Then, since $\hf{q}$ is a homogeneous function where $\hf{ab}=a^{\gamma}\hf{b}$ for $a,b\geq0$ and $\gamma\in\R$, it is easy to verify that the modified solution $\mu^{\frac{1}{\gamma-1}} \qy$ will satisfy Equation~\eqref{eq:b_sen_charc_for_converse} for $\mu = 1$. Moreover, because the new solution is a scalar multiple of $\qy$, it is identical in direction. Hence, for simplicity, we choose to solve the characteristic equations with a $\bar{\mu}$ where $\bar{\mu}\hf{1}=1$, and we also define $\bhf{z} := \bar{\mu}\hf{z}$, where $\bhf{z}$ is still a homogeneous function. Equation~\eqref{eq:b_sen_charc_for_converse} becomes
\begin{align}\label{eq:b_sen_charc_for_converse_bar_h}
    \XX \qy = \diag{\y} \bhf{\barbq}.
\end{align}

Then we denote that $\barbq = \beta \pts{\one + \bDelta}$ for some $\beta > 0$ and some vector $\bDelta$ such that $\nnorm[2]{1 + \bDelta} = \sqrt{n}$, and this ensures that $\nnorm[2]{\qy} = \beta \sqrt{n}$.
Note that
\begin{align*}
   \frac{1}{\sqrt{n}} \nnorm[2]{\bDelta} = \nnorm[2]{\frac{\barbq}{\sqrt{n}\beta} - \frac{\one}{\sqrt{n}}} = \nnorm[2]{\normalize[2]{\qy} - \normalize[2]{\y}},
\end{align*}
and so to obtain our desired lower bound on the set of candidate solutions $\qy$ satisfying $\nnorm[2]{\qy} = \beta\sqrt{n}$, it suffices to obtain a lower bound on $\nnorm[2]{\bDelta}$.
By considering Equation~\eqref{eq:b_sen_charc_for_converse_bar_h},
we write as shorthand $\bDelta_{\y} := \diag{\y} \cdot \bDelta$; therefore, we have
\begin{align}\label{eq:intermediate}
    \beta \cdot \XX \y + \beta \cdot \XX \bDelta_{\y} &= \beta^{\gamma} \cdot \diag{\y} \bhf{\one + \bDelta} \nonumber \\
    \iff \beta \cdot \XX \bDelta_{\y} - \beta^{\gamma} \cdot \diag{\y} \sigma(\bDelta) &= \pts{\beta^{\gamma} \I - \beta \XX} \y,
\end{align}
where we define $\sigma(\Delta_i) := \bhf{1 + \Delta_i}-1$ for any $\Delta_i > -1$.
This function is well-defined for our choice of $\bDelta$, because the constraint $\barbq \succ 0$ necessitates $\bDelta \succ - \one$.

We now upper bound the norm of the LHS of Equation~\eqref{eq:intermediate} above. Note that we assume $\nnorm[\infty]{\bDelta}\leq\delta$ for some $\delta \in (0,1)$, and $\bhf{z}$ is a decreasing function with $\bhf{1}=1$. Hence, since we have $\sigma(0)=0$, it is straightforward to upper bound $|\sigma(\Delta_i)|$ using an absolute linear function $\normo{-k\Delta_i}$ with $k\geq0$. If $\bhf{z}$ is a convex function, we can determine $k$ using $\Delta_i =-\delta$; otherwise, if $\bhf{z}$ is a concave function, we can determine $k$ using $\Delta_i =\delta$. Therefore, we have $|\sigma(\Delta_i)| \leq \cond{\frac{\bhf{1-\delta}-1}{\delta}\normo{\Delta_i}}{\bhf['']{z}\geq0}{\frac{1-\bhf{1+\delta}}{\delta}\normo{\Delta_i}}{\bhf['']{z}<0}$. As a result, we can choose $k=\max\pts{{\frac{\bhf{1-\delta}-1}{\delta}, \frac{1-\bhf{1+\delta}}{\delta}}}$ such that $|\sigma(\Delta_i)|\leq k \normo{\Delta_i}$ and $\nnorm[2]{\sigma(\bDelta)}\leq k \nnorm[2]{\bDelta}$. This leads to the upper bound
\begin{align*}
\nnorm[2]{\beta \cdot \XX \bDelta_{\y} - \beta^{\gamma} \cdot \diag{\y} \sigma(\bDelta)}
&\leq \beta\nnorm[2]{\XX}\nnorm[2]{\bDelta_{\y}} + \beta^{\gamma}\nnorm[2]{\sigma(\bDelta)} \\
&\leq \left(\beta \nnorm[2]{\XX} + k \beta^{\gamma} \right) \nnorm[2]{\bDelta} \\
&\leq 2 \max(\beta \nnorm[2]{\XX}, k \beta^{\gamma}) \nnorm[2]{\bDelta}.
\end{align*}

Plugging this upper bound into Equation~\eqref{eq:intermediate} above and dividing numerator and denominator by $\beta > 0$ yields
\begin{align*}
    \nnorm[2]{\bDelta} \geq \min\left\{ \frac{\nnorm[2]{(\XX - \alpha \I)\y}}{2k \alpha}, \frac{\nnorm[2]{(\XX - \alpha \I)\y}}{2\nnorm[2]{\XX}}\right\},
\end{align*}
where we defined $\alpha := \beta^{\gamma-1}$ as shorthand.
Consequently, we have
\begin{align*}
   \nnorm[2]{\normalize[2]{\qy} - \normalize[2]{\y}} =\frac{1}{\sqrt{n}} \nnorm[2]{\bDelta} \geq \frac{1}{2\sqrt{n}}\min\left\{ \frac{\nnorm[2]{(\XX - \alpha \I)\y}}{k \alpha}, \frac{\nnorm[2]{(\XX - \alpha \I)\y}}{\nnorm[2]{\XX}}\right\},
\end{align*}
Further minimizing over all $\alpha > 0$ then yields the desired result.
\end{proof}

\newpage

\section{Additional simulations for importance weighting under random data} \label{app:simulation}
\begin{figure}[h]
\noindent\begin{subfigure}[b]{.45\textwidth}
 \includegraphics[width=68mm]{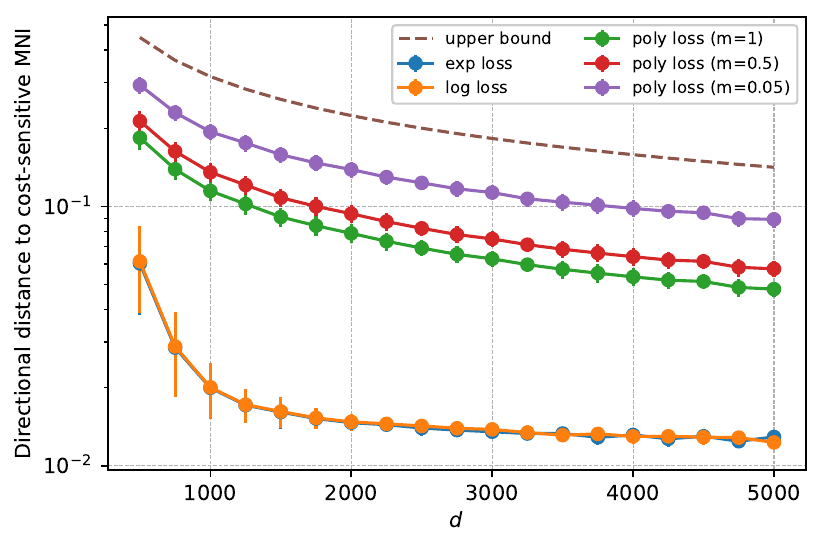}
 \caption{Directional convergence of various loss functions with importance weighting.}\label{fig:iw_conv}
\end{subfigure}
\hfill
\noindent\begin{subfigure}[b]{.45\textwidth}
 \includegraphics[width=69mm]{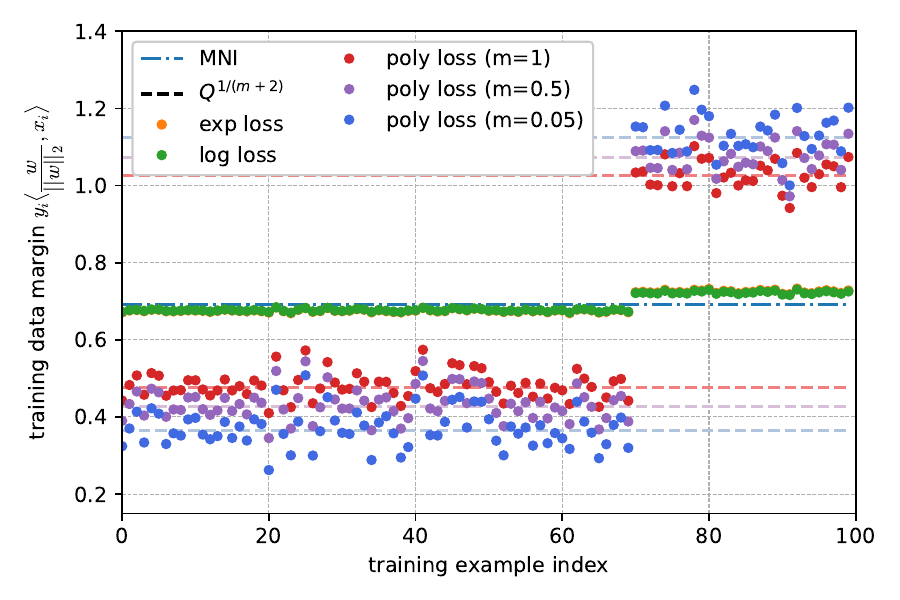}
 \caption{Importance weighting on different loss functions for random data.}
 \label{fig:margin_5k}
\end{subfigure}

\caption{Panel (a) compares the implicit bias of gradient descent to the \emph{cost-sensitive MNI} (defined in Corollary~\ref{cor:ood_interpolation}), which is obtained by fitting the adjusted label $\diag{ Q^{\frac{1}{m+2} \cdot \Ind[i \in S]}}\y$ with $Q=2.0$ on $y_i=-1$. The results demonstrate that the directional distance to the cost-sensitive MNI follows a similar upper bound as in Theorem~\ref{thm:b_sen_upperbound}. The simulation setup is the same as Figure~\ref{fig:binary_exp}. Panel (b) present the outcomes of importance weighting on different loss functions under random data for data dimensions $d = 5000$. The covariates $\seqr{\x_i}{i=1}{n}$ are independently and identically distributed (IID) isotropic Gaussian with a fixed sample size $n = 100$. The first $70$ examples are \emph{majority examples} and labeled as $y_i=+1$, and the rest of the $30$ examples are \emph{minority examples} labeled as $y_i=-1$. Note that we apply the importance weighting factor $Q=10.0$ only to the minority examples. We run gradient descent on all loss functions for the minimum of $10^4$ iterations, or when the empirical risk falls below $10^{-12}$. We also observe that heavier-tailed polynomial losses (i.e. smaller values of $m$) lead to a stronger importance weighting effect.}
\end{figure}
In this section, we provide additional simulations on random data in order to evaluate how different loss functions influence the training data margins under importance weighting.
These simulations are a more realistic complement to Figure~\ref{fig:margin}, which considered the idealized scenario where $\XX = \alpha \I$.

\end{document}